%% file: main.tex
\pdfoutput=1 
\documentclass[final,textrefs,noinfo]{nddiss2e}

\usepackage[utf8]{inputenc}
\usepackage[LGR,T1]{fontenc}

\input{packages} 
\input{config} 
\input{commands} 

\begin{document}

\frontmatter

\title{Nondeterministic Stacks in Neural Networks}
\author{Brian DuSell}
\work{Dissertation}
\degaward{Doctor of Philosophy}
\advisor{David Chiang}
\department{Computer Science and Engineering}
\degdate{April 2023}

\maketitle

\copyrightholder{Brian DuSell}
\copyrightyear{2023}
\copyrightlicense{All Rights Reserved}
\makecopyright

\begin{abstract}
\input{frontmatter/01-abstract}
\end{abstract}

\begin{dedication}
\input{frontmatter/02-dedication}
\end{dedication}

\tableofcontents
\listoffigures
\listoftables

\begin{acknowledge}
\input{frontmatter/04-acknowledgments}
\end{acknowledge}

\input{frontmatter/05-symbols}

\mainmatter
\include{chapters/01-introduction}
\include{chapters/02-background}
\include{chapters/03-stack-rnns}
\include{chapters/04-ns-rnn}
\include{chapters/05-learning-cfls}
\include{chapters/06-rns-rnn}
\include{chapters/07-incremental-execution}
\include{chapters/08-learning-non-cfls}
\include{chapters/09-vrns-rnn}
\include{chapters/10-stack-attention}
\include{chapters/11-conclusion}

\backmatter
\bibliographystyle{acl_natbib}
\bibliography{main}

\end{document}

%% file: packages.tex
\usepackage{amsmath}
\usepackage[capitalize]{cleveref} 

\usepackage[noend]{algpseudocode} 
\usepackage{algorithm} 
\usepackage{amsfonts}
\usepackage{amsthm} 

\usepackage{bm} 
\usepackage{booktabs} 
\usepackage{enumitem} 
\usepackage{forest} 
\usepackage{mathtools}
\usepackage{pgfplots} 
\usepackage{tikz} 
\usepackage{tikz-qtree} 
\usepackage{xcolor}

\usepackage{times}
\usepackage{newtxmath}

%% file: config.tex
\hypersetup{hidelinks}

\usepgfplotslibrary{groupplots,dateplot}
\pgfplotsset{compat=1.16}
\DeclareUnicodeCharacter{2212}{\ensuremath{-}} 

\usetikzlibrary{arrows}
\usetikzlibrary{automata}
\usetikzlibrary{calc}
\usetikzlibrary{fit}
\usetikzlibrary{patterns}
\usetikzlibrary{positioning}
\usetikzlibrary{shapes}
\usetikzlibrary{decorations.pathreplacing}

\tikzset{every edge/.style={draw,->,>=stealth',shorten >=1pt,auto,semithick}}
\tikzset{initial text={},double distance=2pt}
\tikzset{state/.append style={minimum size=18pt,inner sep=1pt}}

\pgfplotsset{cedline/.style={line width=1.5pt}}

\newcommand{\doi}[1]{\href{https://doi.org/#1}{{doi: \urlstyle{rm}\nolinkurl{#1}}}}

\crefformat{enumi}{#2#1#3}
\Crefformat{enumi}{#2#1#3}

\forestset{
    default preamble={
        sn edges/.style={for tree={parent anchor=south, child anchor=north, s sep=0in}},
        sn edges,
        where n children=0{tier=word}{}
    }
}

%% file: commands.tex
\newcommand{\generalnote}[1]{\vspace*{\abovefigureskip}#1}

\newcommand{\term}[1]{\textit{#1}}
\newcommand{\defterm}[1]{\textbf{\term{#1}}}
\newcommand{\dash}{---}
\newcommand{\greektext}[1]{\begingroup\fontencoding{LGR}\selectfont#1\endgroup}

\newcommand{\indicator}[1]{\mathbb{I}[{#1}]}
\newcommand{\funcname}[1]{\textsc{#1}}
\newcommand{\vecvar}[1]{\mathbf{#1}}
\newcommand{\vecvargreek}[1]{\bm{#1}}
\newcommand{\veczero}{\vecvar{0}}
\newcommand{\matvar}[1]{\mathbf{#1}}
\newcommand{\tensorvar}[1]{\mathbf{#1}}
\newcommand{\compositionsoflengthof}[2]{\funcname{Compositions}(#1, #2)}
\newcommand{\logisticletter}{\sigma}
\newcommand{\logistic}[1]{\logisticletter(#1)}
\newcommand{\realset}{\mathbb{R}}
\newcommand{\elementwisemulop}{\odot}
\newcommand{\powerset}[1]{\mathcal{P}(#1)}

\newcommand{\probletter}{p}
\newcommand{\crossentropyletter}{H}
\newcommand{\crossentropyofdistonset}[2]{\crossentropyletter(#2, #1)}
\newcommand{\sampleprobname}[1]{\probletter_{#1}}
\newcommand{\sampleprob}[2]{\sampleprobname{#1}(#2)}
\newcommand{\modelprobname}[1]{\probletter_{#1}}

\newcommand{\bigo}[1]{O(#1)}

\newcommand{\emptystring}{\varepsilon}
\newcommand{\sym}[1]{\texttt{#1}}
\newcommand{\reverse}[1]{#1^R}
\DeclareMathOperator*{\joinwithcommas}{{\ooalign{$\bigcirc$\cr\hidewidth$\raisebox{2pt}{\sym{,}}$\hidewidth\cr}}}
\newcommand{\nonterminalsinstring}[1]{\funcname{Nonterminals}(#1)}
\newcommand{\multisym}[1]{\fbox{\raisebox{0pt}[1.5ex][0.25ex]{\ensuremath{#1}}}}
\newcommand{\concatop}{\cdot}
\newcommand{\stringsoflength}[2]{#1_{#2}}
\newcommand{\pcfgrule}[3]{#1 & \rightarrow #2 && \mathrel/ #3 \\}
\newcommand{\langof}[1]{L(#1)}

\newcommand{\weightoftrans}[2]{\psi(#1, #2)}
\newcommand{\stringsumruns}[2]{\Pi(#1, #2)}
\newcommand{\allsumruns}[1]{\Pi(#1)}
\newcommand{\runweight}[2]{\psi(#1, #2)}
\newcommand{\stringsum}[2]{\psi(#1, #2)}
\newcommand{\allsum}[1]{\psi(#1)}

\newcommand{\natvar}[1]{#1}
\newcommand{\natterm}[1]{#1}

\newcommand{\pdatrans}[5]{#1, #3 \xrightarrow{#2} #4, #5}
\newcommand{\pdatransnoinput}[4]{#1, #2 \rightarrow #3, #4}
\newcommand{\pdatransletter}{\tau}
\newcommand{\pdatranst}[2]{#1_{#2}}
\newcommand{\pdarunletter}{\pi}
\newcommand{\wpdaweight}[5]{\delta(#1, #2, #3, #4, #5)}
\newcommand{\pdaconfig}[3]{(#1, #2, #3)}
\newcommand{\wpdaweightletter}{\psi}
\newcommand{\wpdarunweight}[1]{\wpdaweightletter(#1)}
\newcommand{\pdarunendsin}[4]{#1 \leadsto #2, #3, #4}
\newcommand{\pdarunendsinnot}[3]{#1 \leadsto #2, #3}
\newcommand{\pdanumrunsname}{\textit{ct}}
\newcommand{\pdanumruns}[1]{\pdanumrunsname(#1)}
\newcommand{\pdanumrunsstateset}[2]{\pdanumrunsname(#1, #2)}

\newcommand{\softmax}{\mathrm{softmax}}
\newcommand{\softmaxover}[1]{\underset{#1}{\softmax}}
\newcommand{\affine}[2]{\weightparam{#1} #2 + \biasparam{#1}}
\newcommand{\weightparamletter}{\matvar{W}}
\newcommand{\weightparam}[1]{\weightparamletter_{\mathrm{#1}}}
\newcommand{\biasparamletter}{\vecvar{b}}
\newcommand{\biasparam}[1]{\biasparamletter_{\mathrm{#1}}}
\newcommand{\vectorparam}[2]{\vecvar{#1}_{\mathrm{#2}}}
\newcommand{\unittimesteptag}[3]{#1_{#2}^{\mathrm{#3}}}
\newcommand{\unittimesteptagi}[4]{#1_{#2}^{\mathrm{#3}, #4}}
\newcommand{\probdistmt}[2]{\probletter_{#2}}
\newcommand{\probdistm}[1]{\probletter}

\newcommand{\inputstring}{w}
\newcommand{\inputsymbolt}[1]{w_{#1}}
\newcommand{\rnninputletter}{\vecvar{x}}
\newcommand{\rnninputt}[1]{\rnninputletter_{#1}}
\newcommand{\rnnoutputletter}{\vecvar{y}}
\newcommand{\rnnoutputt}[1]{\rnnoutputletter_{#1}}
\newcommand{\rnnhiddenletter}{\vecvar{h}}
\newcommand{\rnnhiddent}[1]{\rnnhiddenletter_{#1}}
\newcommand{\rnnhiddenunit}[2]{\unittimesteptag{h}{#1}{#2}}
\newcommand{\rnnhiddenuniti}[3]{\unittimesteptagi{h}{#1}{#2}{#3}}
\newcommand{\lstmmemorycellletter}{\vecvar{c}}
\newcommand{\lstmmemorycellt}[1]{\lstmmemorycellletter_{#1}}
\newcommand{\bos}{\textsc{bos}}
\newcommand{\eos}{\textsc{eos}}
\newcommand{\rnnrecognizeroutput}{y}
\newcommand{\rnnrecognizermlpunit}[1]{y_{#1}}

\newcommand{\lstminputgatet}[1]{\vecvar{i}_{#1}}
\newcommand{\lstmforgetgatet}[1]{\vecvar{f}_{#1}}
\newcommand{\lstmcandidatet}[1]{\vecvar{g}_{#1}}
\newcommand{\lstmoutputgatet}[1]{\vecvar{o}_{#1}}
\newcommand{\lstmcandidateunit}[2]{\unittimesteptag{g}{#1}{#2}}
\newcommand{\lstmcandidateuniti}[3]{\unittimesteptagi{g}{#1}{#2}{#3}}
\newcommand{\lstmmemorycellunit}[2]{\unittimesteptag{c}{#1}{#2}}
\newcommand{\lstmmemorycelluniti}[3]{\unittimesteptagi{c}{#1}{#2}{#3}}

\newcommand{\stackactionsletter}{a}
\newcommand{\stackactionst}[1]{\stackactionsletter_{#1}}
\newcommand{\stackreadingletter}{\vecvar{r}}
\newcommand{\stackreadingt}[1]{\stackreadingletter_{#1}}
\newcommand{\stackobjectletter}{s}
\newcommand{\stackobjectt}[1]{\stackobjectletter_{#1}}
\newcommand{\stackactionsfuncname}{\funcname{Actions}}
\newcommand{\stackactionsfunc}[1]{\stackactionsfuncname(#1)}
\newcommand{\stackobjectfuncname}{\funcname{Stack}}
\newcommand{\stackobjectfunc}[2]{\stackobjectfuncname(#1, #2)}
\newcommand{\stackreadingfuncname}{\funcname{Reading}}
\newcommand{\stackreadingfunc}[1]{\stackreadingfuncname(#1)}
\newcommand{\stackvectorsizeletter}{m}
\newcommand{\pushedstackvectorletter}{\vecvar{v}}
\newcommand{\pushedstackvectort}[1]{\pushedstackvectorletter_{#1}}
\newcommand{\stackopweightt}[2]{a_{#2}^{\mathrm{#1}}}
\newcommand{\stackpushweightt}[1]{\stackopweightt{push}{#1}}
\newcommand{\stackpopweightt}[1]{\stackopweightt{pop}{#1}}
\newcommand{\stacknoopweightt}[1]{\stackopweightt{noop}{#1}}

\newcommand{\stratpoporigt}[1]{u_{#1}}
\newcommand{\stratpushorigt}[1]{d_{#1}}
\newcommand{\stratpopt}[1]{\stackpopweightt{#1}}
\newcommand{\stratpusht}[1]{\stackpushweightt{#1}}
\newcommand{\stratstacktensort}[1]{\matvar{V}_{#1}}
\newcommand{\stratstrengthtensort}[1]{\vecvar{s}_{#1}}
\newcommand{\stratstackvector}[2]{\stratstacktensort{#1}[#2]}
\newcommand{\stratstrength}[2]{\stratstrengthtensort{#1}[#2]}

\newcommand{\supprobst}[1]{\vecvar{a}_{#1}}
\newcommand{\suppusht}[1]{\stackpushweightt{#1}}
\newcommand{\suppopt}[1]{\stackpopweightt{#1}}
\newcommand{\supnoopt}[1]{\stacknoopweightt{#1}}
\newcommand{\supstacktensort}[1]{\matvar{V}_{#1}}
\newcommand{\supstackvector}[2]{\supstacktensort{#1}[#2]}

\newcommand{\stackwfastatecontent}[3]{#1, #2, #3}
\newcommand{\stackwfastate}[3]{(\stackwfastatecontent{#1}{#2}{#3})}

\newcommand{\nsactionset}[1]{\funcname{Act}(#1)}
\newcommand{\nstranstensorletter}{\bm{\Delta}}
\newcommand{\nstranst}[1]{\nstranstensorletter_{#1}}
\newcommand{\nstransttype}[2]{\nstranst{#1}[#2]}
\newcommand{\nstransweight}[5]{\nstranst{#2}[#1, #3 \rightarrow #4, #5]}
\newcommand{\nspushweight}[5]{\nstransweight{#1}{#2}{#3}{#4}{#3 #5}}
\newcommand{\nsreplweight}[5]{\nstransweight{#1}{#2}{#3}{#4}{#5}}
\newcommand{\nspopweight}[4]{\nstransweight{#1}{#2}{#3}{#4}{\emptystring}}
\newcommand{\nsinnerweightletter}{\gamma}
\newcommand{\nsinnerweightit}[2]{\nsinnerweightletter[#1 \rightarrow #2]}
\newcommand{\nsinnerweight}[6]{\nsinnerweightit{#1}{#4}[#2, #3 \rightarrow #5, #6]}
\newcommand{\nsinnerweightauxletter}{\nsinnerweightletter'}
\newcommand{\nsinnerweightaux}[5]{\nsinnerweightauxletter[#1 \rightarrow #4][#2, #3 \rightarrow #5]}
\newcommand{\nsforwardweightletter}{\alpha}
\newcommand{\nsforwardweightt}[1]{\nsforwardweightletter[#1]}
\newcommand{\nsforwardweight}[3]{\nsforwardweightt{#1}[#2, #3]}
\newcommand{\nsstackreadingnostate}[2]{\stackreadingt{#1}[#2]}
\newcommand{\nsstackreading}[3]{\stackreadingt{#1}[#2, #3]}

\newcommand{\vrnsbottomvector}{\pushedstackvectort{0}}
\newcommand{\topvectorofrun}[1]{\pushedstackvectorletter(#1)}

\newcommand{\vrnsstackreadingvec}[3]{\stackreadingt{#1}[#2, #3]}
\newcommand{\vrnsinnervectorletter}{\vecvargreek{\zeta}}
\newcommand{\vrnsinnervector}[6]{\vrnsinnervectorletter[#1 \rightarrow #4][#2, #3 \rightarrow #5, #6]}
\newcommand{\vrnsstackreadingnumerletter}{\vecvargreek{\eta}}
\newcommand{\vrnsstackreadingnumer}[3]{\vrnsstackreadingnumerletter_{#1}[#2, #3]}

\newcommand{\sublayerinputletter}{\vecvar{x}}
\newcommand{\sublayerinputt}[1]{\sublayerinputletter_{#1}}
\newcommand{\sublayeroutputletter}{\vecvar{y}}
\newcommand{\sublayeroutputt}[1]{\sublayeroutputletter_{#1}}
\newcommand{\layerinputletter}{\hat{\sublayerinputletter}}
\newcommand{\layerinputt}[1]{\layerinputletter_{#1}}
\newcommand{\layeroutputletter}{\hat{\sublayeroutputletter}}
\newcommand{\layeroutputt}[1]{\layeroutputletter_{#1}}
\newcommand{\stackactionvecletter}{\vecvar{a}}
\newcommand{\stackactionvect}[1]{\stackactionvecletter_{#1}}
\newcommand{\dmodel}{d_{\mathrm{model}}}

\newcommand{\MarkedReversal}{\ensuremath{w \sym{\#} \reverse{w}}}
\newcommand{\UnmarkedReversal}{\ensuremath{w \reverse{w}}}
\newcommand{\PaddedReversal}{\ensuremath{w a^p \reverse{w}}}
\newcommand{\CountThree}{\ensuremath{\sym{a}^n \sym{b}^n \sym{c}^n}}
\newcommand{\MarkedReverseAndCopy}{\ensuremath{w \sym{\#} \reverse{w} \sym{\#} w}}
\newcommand{\CountAndCopy}{\ensuremath{w \sym{\#}^n w}}
\newcommand{\MarkedCopy}{\ensuremath{w \sym{\#} w}}
\newcommand{\UnmarkedCopyDifferentAlphabets}{\ensuremath{w w'}}
\newcommand{\UnmarkedReverseAndCopy}{\ensuremath{w \reverse{w} w}}
\newcommand{\UnmarkedCopy}{\ensuremath{ww}}

\newtheorem{theorem}{Theorem}
\newtheorem{proposition}[theorem]{Proposition}
\newtheorem{lemma}[theorem]{Lemma}

\theoremstyle{definition}
\newtheorem{definition}{Definition}

%% file: frontmatter/01-abstract.tex
Human language is full of compositional syntactic structures, and although neural networks have contributed to groundbreaking improvements in computer systems that process language, widely-used neural network architectures still exhibit limitations in their ability to process syntax. To address this issue, prior work has proposed adding stack data structures to neural networks, drawing inspiration from theoretical connections between syntax and stacks. However, these methods employ deterministic stacks that are designed to track one parse at a time, whereas syntactic ambiguity, which requires a nondeterministic stack to parse, is extremely common in language. In this dissertation, we remedy this discrepancy by proposing a method of incorporating nondeterministic stacks into neural networks. We develop a differentiable data structure that efficiently simulates a nondeterministic pushdown automaton, representing an exponential number of computations with a dynamic programming algorithm. Since it is differentiable end-to-end, it can be trained jointly with other neural network components using standard backpropagation and gradient descent. We incorporate this module into two predominant architectures: recurrent neural networks (RNNs) and transformers. We show that this raises their formal recognition power to arbitrary context-free languages, and also aids training, even on deterministic context-free languages. Empirically, neural networks with nondeterministic stacks learn context-free languages much more effectively than prior stack-augmented models, including a language with theoretically maximal parsing difficulty. We also show that an RNN augmented with a nondeterministic stack is capable of surprisingly powerful behavior, such as learning cross-serial dependencies, a well-known non-context-free pattern. We demonstrate improvements on natural language modeling and provide analysis on a syntactic generalization benchmark. This work represents an important step toward building systems that learn to use syntax in more human-like fashion.

%% file: frontmatter/02-dedication.tex
For my grandfather, D. Lee DuSell.

%% file: frontmatter/04-acknowledgments.tex
I would like to thank my advisor, David Chiang, for his mentorship, support, and unyielding patience throughout this project. He is not only a brilliant thinker, but a kind person and a role model of ethics. I have been delighted to work with him. I would also like to thank my colleagues in the Natural Language Processing Group at Notre Dame (Tomer Levinboim, Kenton Murray, Arturo Argueta, Antonis Anastasopoulos, Justin DeBenedetto, Toan Nguyen, Xing Jie Zhong, Darcey Riley, Stephen Bothwell, Aarohi Srivastava, Ken Sible, Chihiro Taguchi) for their mentorship, fellowship, and stimulating discussions over the years. I express my thanks to the other members of my thesis committee, Walter Scheirer, Taeho Jung, and Bob Frank, for their invaluable feedback and encouragement.

I especially thank Richard and Peggy Notebaert, who generously funded my research during my time as a graduate student at Notre Dame. This research was also supported in part by a Google Faculty Research Award to David Chiang.

I thank the Center for Research Computing at Notre Dame, without which none of the experimental work in this thesis would have been possible. Their staff has always responded promptly and helpfully to my questions. I would also like to thank Mimi Beck and Joyce Yeats, who, through their administrative duties, often made my time as a graduate student more comfortable and enjoyable. I thank my students in CSE 30151 Theory of Computing, for entrusting me with their education in a subject I love to teach.

I thank the anonymous reviewers who refereed the papers that eventually became chapters in this thesis. Their constructive comments undoubtedly improved the quality of the work you are about to read. I would also like to thank Justin, Darcey, and Stephen for their comments on earlier drafts of some chapters. I must also thank ChatGPT for invaluable assistance with the \LaTeX{} and TikZ code used to create this document.

As in everything, I thank my parents, whose love, support, and encouragement were indispensable throughout my Ph.D. journey. I also thank my brother, Severin, for his love and friendship. I thank my grandfather, whose great faith and dedication to his craft have always served as an inspiration for my own. I also thank my cat, Shadow, who comforted and amused me for many years, especially during grad school. I must also thank the various dining establishments of the South Bend-Mishawaka area, who fueled this research in a material way. If I am what I eat, then I am a Portillo's chicken sandwich.

Finally, I thank God, from whom all blessings flow, and without whom nothing is possible.

%% file: frontmatter/05-symbols.tex
\begin{symbols}[rp{4.5in}]

\symbolline{\Sigma}{An alphabet of input symbols.}
\symbolline{\Gamma}{An alphabet of stack symbols.}
\symbolline{\sym{a}, \sym{b}, \sym{0}, \sym{1}, \sym{\#}, \ldots}{Symbols, which are written in typewriter font.}
\symbolline{a, b}{Variables for symbols, which are written in italic font.}
\symbolline{w}{A string, especially of input symbols.}
\symbolline{n}{Length of an input string.}
\symbolline{u, v}{Strings.}
\symbolline{\ell}{Length of a string.}
\symbolline{\alpha, \beta, \gamma}{Strings, possibly mixing different alphabets.}
\symbolline{\emptystring}{The empty string.}
\symbolline{L}{A formal language.}
\symbolline{S}{A finite set of strings.}
\symbolline{\langof{\cdot}}{Language of a recognizer or grammar.}

\symbolline{\reverse{w}}{The reverse of string $w$.}
\symbolline{\phi}{A string homomorphism.}

\symbolline{\powerset{A}}{The power set of set $A$, i.e. the set of all subsets of $A$.}

\symbolline{G}{A context-free grammar.}
\symbolline{V}{Set of variables in a context-free grammar.}
\symbolline{R}{Set of production rules in a context-free grammar.}
\symbolline{S}{Start variable in a context-free grammar.}
\symbolline{A, B}{Variables in a context-free grammar.}
\symbolline{X}{Variable or terminal in a context-free grammar.}

\symbolline{P}{A pushdown automaton.}

\symbolline{Q}{A finite set of states in an automaton.}
\symbolline{\delta}{The transition function of an automaton.}
\symbolline{F}{Set of accept states in an automaton.}
\symbolline{q, r, s, u}{States of an automaton.}
\symbolline{\pdatransletter}{Variable for an automaton transition.}
\symbolline{\pdarunletter}{Variable for a run (sequence of transitions) of an automaton.}
\symbolline{\Pi}{A set of runs.}
\symbolline{\wpdaweightletter}{A non-negative weight.}

\symbolline{x, y, z}{Variables for stack symbols of a pushdown automaton.}
\symbolline{\bot}{Special stack symbol signifying the bottom of a stack.}

\symbolline{\realset}{Set of real numbers.}
\symbolline{\realset_{\geq 0}}{Set of non-negative real numbers.}

\symbolline{\probletter}{A probability or probability distribution.}
\symbolline{\crossentropyletter}{Cross-entropy.}

\symbolline{\vecvar{x}[j]}{The $j$th element of vector $\vecvar{x}$.}
\symbolline{\matvar{W}[i, j]}{The element at the $i$th row and $j$th column of matrix $\vecvar{W}$.}
\symbolline{\tensorvar{A}[i, j, k, \ldots]}{The element or sub-tensor at index $i, j, k, \ldots$ of tensor $\tensorvar{A}$.}
\symbolline{\sigma}{The logistic function, $\logistic{x} = \frac{1}{1+e^{-x}}$.}
\symbolline{\indicator{\phi}}{The indicator function, which is 1 if the proposition $\phi$ is true and 0 otherwise.}
\symbolline{t, i, k}{Variables for timesteps, i.e. positions in a sequence of inputs.}
\symbolline{M}{A neural network.}
\symbolline{\weightparamletter}{A matrix of weights in a neural network layer.}
\symbolline{\biasparamletter}{A vector of bias terms in a neural network layer.}
\symbolline{\rnninputletter}{Input vector to a neural network.}
\symbolline{\rnnhiddenletter}{Hidden state vector of an RNN.}
\symbolline{\lstmmemorycellletter}{Memory cell of an LSTM.}
\symbolline{\rnnoutputletter}{Output vector from a neural network.}
\symbolline{\eos}{End of sequence symbol.}

\symbolline{\stackobjectletter}{A differentiable stack.}
\symbolline{\stackactionsletter}{Actions on a differentiable stack.}
\symbolline{\stackreadingletter}{Stack reading from a differentiable stack.}
\symbolline{\pushedstackvectorletter}{A vector pushed to a differentiable stack.}
\symbolline{m}{Stack embedding size, i.e. the size of each vector in a stack of vectors.}

\symbolline{\nstranstensorletter}{Tensor containing transition weights of a weighted pushdown automaton.}
\symbolline{\nsinnerweightletter}{Tensor of inner weights.}
\symbolline{\nsforwardweightletter}{Tensor of forward weights.}
\symbolline{\vrnsinnervectorletter}{Tensor of vector inner weights.}
\symbolline{D}{Maximum window size in a tensor of inner weights.}

\end{symbols}

%% file: chapters/01-introduction.tex
\chapter{Introduction}
\label{chap:introduction}

Just about every introductory programming book begins with a version of the following remark: plain English is a poor choice for describing algorithms, because language contains ambiguity. What makes human language ambiguous? One source is \term{syntactic ambiguity}. Consider the following sentence:
\begin{center}
    Mary saw a man in the park with a telescope.
\end{center}
This sentence can mean at least three different things, depending on how its syntactic structure is interpreted:
\begin{enumerate}[label=(\alph*)]
    \item \label{item:ambiguity-example-man} There was a man in the park who had a telescope, and Mary saw him;
    \item \label{item:ambiguity-example-saw} Mary used a telescope to see a man who was in the park; or
    \item \label{item:ambiguity-example-park} Mary saw a man in the park, and this park is known for having a telescope in it.
\end{enumerate}
This sentence is an example of ambiguous prepositional phrase attachment. According to \cref{item:ambiguity-example-man}, the prepositional phrase ``with a telescope'' is attached to the noun phrase ``a man in the park''; according to \cref{item:ambiguity-example-saw}, it is attached to the verb ``saw''; and according to \cref{item:ambiguity-example-park}, it is attached to the noun ``park.'' \Cref{fig:intro-example-parse-trees} shows \textit{parse trees} that illustrate the syntactic structure of each of these three possibilities.

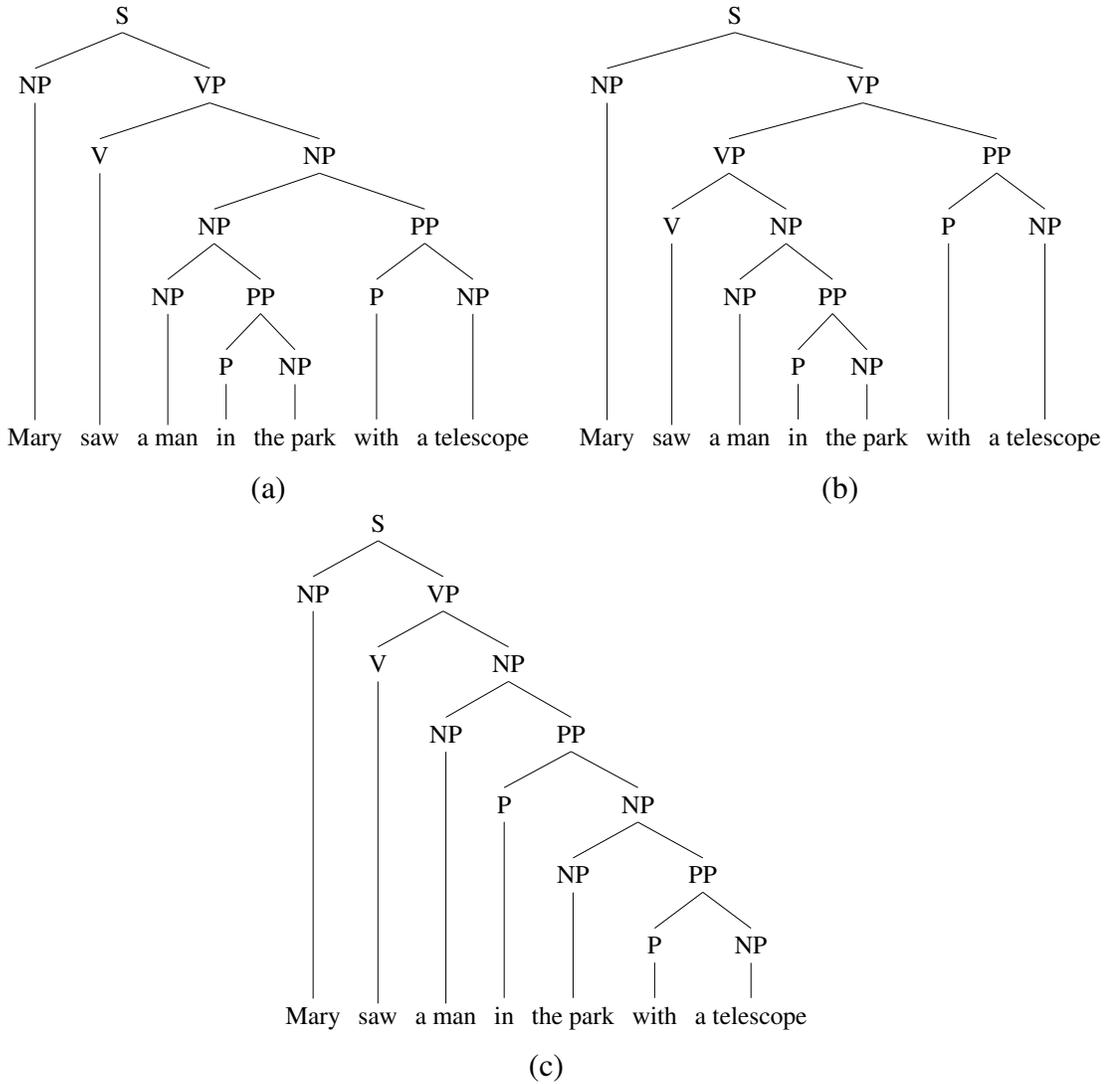
\begin{figure*}
    \centering
    \newcommand{\mysubfigure}[2]{
        \begin{minipage}[b]{0.48\textwidth}
            \centering
            \scalebox{0.8}{\begin{forest}
                #1
            \end{forest}} \\
            \Cref{#2}
        \end{minipage}
    }
    \mysubfigure{
        [ S
            [ NP [ Mary ] ]
            [ VP
                [ V [ saw ] ]
                [ NP
                    [ NP
                        [ NP [ a man ] ]
                        [ PP [ P [ in ] ] [ NP [ the park ] ] ]
                    ]
                    [ PP [ P [ with ] ] [ NP [ a telescope ] ] ]
                ]
            ]
        ]
    }{item:ambiguity-example-man}
    \mysubfigure{
        [ S
            [ NP [ Mary ] ]
            [ VP
                [ VP
                    [ V [ saw ] ]
                    [ NP
                        [ NP [ a man ] ]
                        [ PP [ P [ in ] ] [ NP [ the park ] ] ]
                    ]
                ]
                [ PP [ P [ with ] ] [ NP [ a telescope ] ] ]
            ]
        ]
    }{item:ambiguity-example-saw}
    \mysubfigure{
        [ S
            [ NP [ Mary ] ]
            [ VP
                [ V [ saw ] ]
                [ NP
                    [ NP [ a man ] ]
                    [ PP
                        [ P [ in ] ]
                        [ NP
                            [ NP [ the park ] ]
                            [ PP [ P [ with ] ] [ NP [ a telescope ] ] ]
                        ]
                    ]
                ]
            ]
        ]
    }{item:ambiguity-example-park}
    \caption[Three different ways of interpreting the syntactic structure of the sentence ``Mary saw a man in the park with a telescope.'']{Three different ways of interpreting the syntactic structure of the sentence ``Mary saw a man in the park with a telescope.'' Just as a company has a hierarchy of employees, a sentence contains a hierarchy of \term{grammatical constituents}. The acronyms connected by lines above the sentence represent types of constituent; the lines indicate which words or sub-constituents each one contains. S = sentence, NP = noun phrase, VP = verb phrase, V = verb, PP = prepositional phrase, P = preposition.}
    \label{fig:intro-example-parse-trees}
\end{figure*}

Humans can use context to disambiguate the meaning of sentences like the example above, but for machines, this kind of context might not be available. Some interpretations are arguably more plausible in the absence of context, but it is important that machines not merely assume the speaker intends the most frequent one, as all three possibilities, if not equally likely, are perfectly valid grammatically and would be accepted by a fluent English speaker. Moreover, consider how easily you, the reader, were able to interpret the example sentence three different ways; if a machine cannot do the same, it can hardly be said to process language at a human level. Rather, a machine should be able to entertain the possibility that any of the three interpretations might be what the speaker intended and react accordingly.

In this dissertation, we propose a method for incorporating this ability to entertain multiple possibilities into the predominant paradigm for building computer systems that process human language: neural networks. Neural networks, being a machine learning framework, alleviate the burden of painstakingly devising linguistic rules by hand; rather, programmers can use neural networks to \emph{learn} the rules automatically, simply by exposing the networks to large amounts of text. For example, instead of having a programmer or linguistic expert code the syntactic rules that govern the trees in \cref{fig:intro-example-parse-trees}\dash{}rules which quickly grow to dizzying complexity, and on which even expert linguists do not agree\dash{}we can now \emph{train} a neural network on millions or billions of example English sentences (a recent state-of-the-art neural network, GPT-3, was trained on no fewer than 400 \emph{billion} words \citep{brown-etal-2020-language}).

How can we facilitate neural networks' learning of syntax? One idea, drawing inspiration from the theory of computing, has been to augment them with a mechanism called a \term{stack}. Stacks are a fundamental data structure in computer science. A stack is a container of items that works like a stack of dinner plates: new plates can be added or removed at the top of the stack, but it is impossible to insert or remove plates at any other place. This restriction is referred to as \term{last-in-first-out} (LIFO) order.

The theory of computing tells us that there is a deep connection between syntax and stacks. To get an intuition for this connection, consider \cref{fig:intro-cfg}, which lists the grammatical rules that govern the parse trees in \cref{fig:intro-example-parse-trees}, and compare it with \cref{fig:stack-parsing-example-top-down}, in which a computer system uses a stack to process a sentence while reading it from left to right. In \cref{fig:stack-parsing-example-top-down}, the ``dinner plates'' on the stack are now grammatical constituents (noun phrase, verb phrase, prepositional phrase, etc.), and the LIFO mechanism allows the system to track, at any given point in the sentence, exactly which constituents may appear in the rest of the sentence. For example, thanks to the stack, the system knows that the word ``park'' can be followed by a preposition like ``with'' or by the end of the sentence, but not by a past-tense verb like ``saw.''

\begin{figure*}
    \renewcommand{\rule}[2]{\natvar{#1} &$\rightarrow$ #2 \\}
    \centering
    \begin{tabular}{r@{\hspace{0.3em}}l}
        \rule{S}{\natvar{NP} \natvar{VP}}
        \rule{VP}{\natvar{V} \natvar{NP}}
        \rule{VP}{\natvar{VP} \natvar{PP}}
        \rule{NP}{\natvar{NP} \natvar{PP}}
        \rule{PP}{\natvar{P} \natvar{NP}}
        \rule{NP}{\natterm{Mary}}
        \rule{NP}{\natterm{a man}}
        \rule{NP}{\natterm{the park}}
        \rule{NP}{\natterm{a telescope}}
        \rule{P}{\natterm{in}}
        \rule{P}{\natterm{with}}
    \end{tabular}
    \caption[The set of grammatical rules used in the parse trees shown in \cref{fig:intro-example-parse-trees,fig:stack-parsing-example-top-down}, expressed as a context-free grammar.]{The set of grammatical rules used in the parse trees shown in \cref{fig:intro-example-parse-trees,fig:stack-parsing-example-top-down}, expressed as a context-free grammar. Each line means: the symbol to the left of the $\rightarrow$ may be expanded into the symbols on the right. For example, any \natvar{VP} may be expanded into either \natvar{V} \natvar{NP} or \natvar{VP} \natvar{PP}. The topmost symbol in a parse tree must always be \natvar{S}, and the bottommost symbols must always be words.}
    \label{fig:intro-cfg}
\end{figure*}

\def\stackwidth{0.7em}
\def\stackbottomheight{0.65ex}

\newcommand{\stackbracket}[1]{
    \draw (#1) ++(-\stackwidth,\stackbottomheight) -- ++(0,-\stackbottomheight) -- ++(2*\stackwidth,0) -- ++(0,\stackbottomheight);}

\begin{landscape}

\def\seplength{1.1em}
\def\stackyshift{11ex}
\def\stackxshift{-0.5em}

\newcommand{\stacktoside}[7]{
    \node[
        align=center,
        anchor=south,
        yshift=-\stackyshift,
        xshift={#5 * (-1.6em * #3 + \stackxshift)}
    ] (step#1) at (#2.base #7) {#4};
    \stackbracket{step#1.south}
    \node[
        align=center,
        below=2ex of step#1.south
    ] (stepnum#1) {\small #1};}

\newcommand{\stacktoleft}[4]{
    \stacktoside{#1}{#2}{#3}{#4}{1}{east}{west}}

\newcommand{\stacktoright}[4]{
    \stacktoside{#1}{#2}{#3}{#4}{-1}{west}{east}}

\begin{figure*}
    \centering
    \begin{forest}
        [ S,s sep=2*\seplength
            [ NP [ Mary,name=mary ] ]
            [ VP,s sep=4*\seplength
                [ V [ saw,name=saw ] ]
                [ NP,s sep=2*\seplength
                    [ NP,s sep=2*\seplength
                        [ NP [ a man,name=aman ] ]
                        [ PP,s sep=1*\seplength [ P [ in,name=in ] ] [ NP [ the park,name=thepark ] ] ]
                    ]
                    [ PP,s sep=1*\seplength [ P [ with,name=with ] ] [ NP [ a telescope,name=atelescope ] ] ]
                ]
            ]
        ]
        \stacktoleft{1}{mary}{1}{S}
        \stacktoleft{2}{mary}{0}{NP \\ VP}
        \stacktoleft{3}{saw}{1}{VP}
        \stacktoleft{4}{saw}{0}{V \\ NP}
        \stacktoleft{5}{aman}{2}{NP}
        \stacktoleft{6}{aman}{1}{NP \\ PP}
        \stacktoleft{7}{aman}{0}{NP \\ PP \\ PP}
        \stacktoleft{8}{in}{1}{PP \\ PP}
        \stacktoleft{9}{in}{0}{P \\ NP \\ PP}
        \stacktoleft{10}{thepark}{0}{NP \\ PP}
        \stacktoleft{11}{with}{1}{PP}
        \stacktoleft{12}{with}{0}{P \\ NP}
        \stacktoleft{13}{atelescope}{0}{NP}
        \stacktoright{14}{atelescope}{0}{}
        \node[fit=(step1) (step7)] (stackrow) {};
        \draw[name=brace, decorate,decoration={brace,amplitude=10pt}] (stackrow.south west) -- (stackrow.north west);
        \node[left=20pt of stackrow, anchor=center, rotate=90] {Stack};
        \node[below=2ex] at (current bounding box.south) {Reading order $\rightarrow$};
        \node[left=20pt of stepnum1, anchor=center] {Step};
    \end{forest}
    \caption{An example of top-down parsing using a stack.}
    \label{fig:stack-parsing-example-top-down}
\end{figure*}
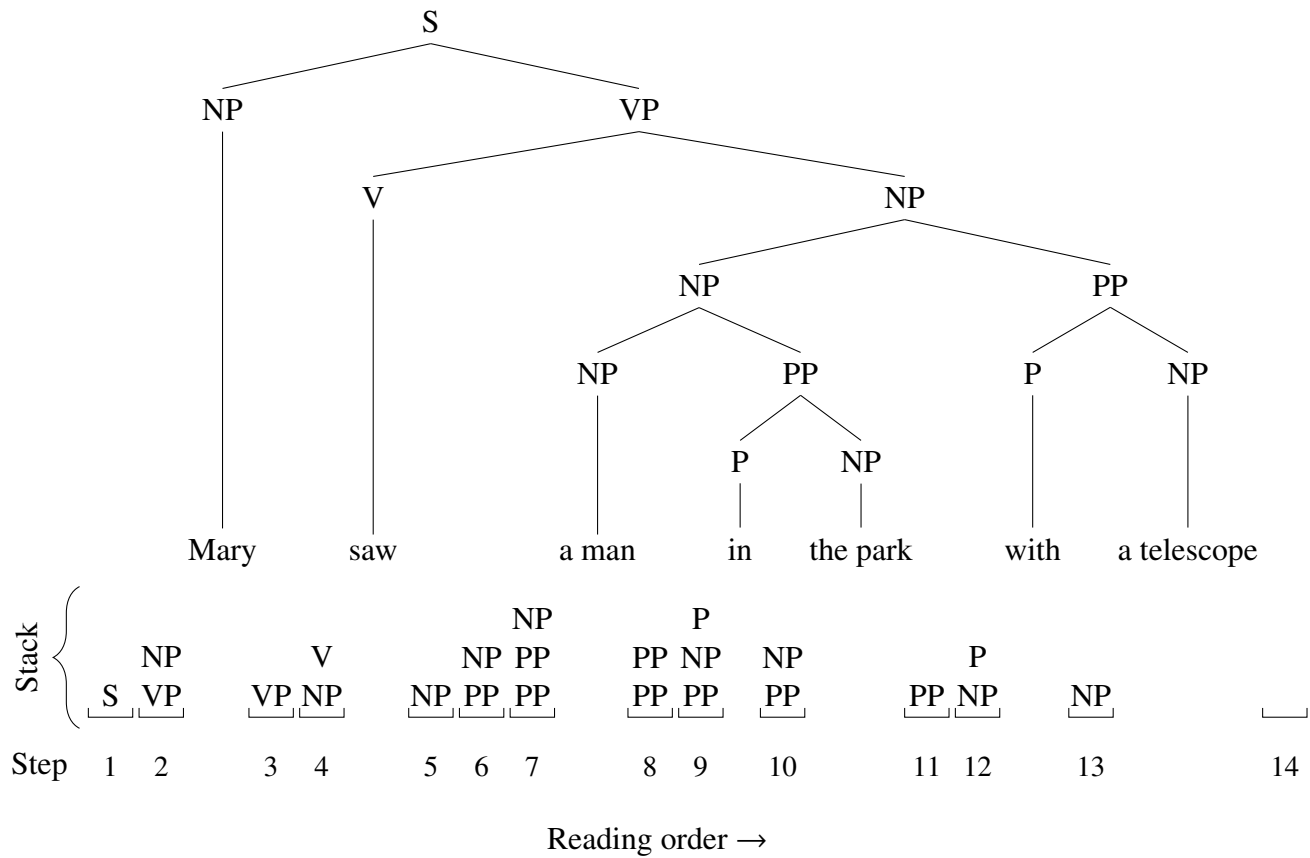
\end{landscape}

Let us examine this mechanism in a little more technical detail. The algorithm in \cref{fig:stack-parsing-example-top-down} works as follows. The stack initially contains \natvar{S} (for ``sentence''). According to \cref{fig:intro-cfg}, \natvar{S} may be replaced with \natvar{NP} \natvar{VP}, so the \natvar{S} is removed, and \natvar{NP} \natvar{VP} is added in its place. \natvar{NP} may be replaced with ``\natterm{Mary},'' which is the first word in the sentence, so \natvar{NP} is removed, and the system advances past the word ``\natterm{Mary}.'' The system continues in like fashion by repeatedly expanding the top constituent on the stack according to one of the rules in \cref{fig:intro-cfg}, and removing the top constituent whenever the next word in the sentence matches. In this way, the stack always keeps track of the \emph{unfinished} constituents. The sentence is grammatical only if the stack is empty after reading the last word.

This is not the only way stacks can track syntax; the example just discussed is called ``top-down'' parsing because the system starts with the topmost symbol in the tree (\natvar{S}) and expands downward. On the other hand, in \cref{fig:stack-parsing-example-bottom-up}, the system uses a scheme called ``bottom-up'' parsing, in which it keeps track of \emph{finished} constituents. The stack is initially empty. Every time it reads a word, it adds its corresponding constituent to the stack, and whenever the topmost symbols of the stack can be replaced with a higher-level symbol, the system replaces them on the stack. The sentence is grammatical only if the system ends with a single \natvar{S} on the stack. This is called ``bottom-up'' because the system starts with the lowest-level symbols on the stack and works \emph{upward} in the tree.

\begin{landscape}

\def\seplength{1.1em}
\def\stackyshift{17ex}
\def\stackxshift{-0.5em}

\newcommand{\stacktoside}[7]{
    \node[
        align=center,
        anchor=south,
        yshift=-\stackyshift,
        xshift={#5 * (-1.6em * #3 + \stackxshift)}
    ] (step#1) at (#2.base #7) {#4};
    \stackbracket{step#1.south}
    \node[
        align=center,
        below=2ex of step#1.south
    ] (stepnum#1) {\small #1};}

\newcommand{\stacktoleft}[4]{
    \stacktoside{#1}{#2}{#3}{#4}{1}{east}{west}}

\newcommand{\stacktoright}[4]{
    \stacktoside{#1}{#2}{#3}{#4}{-1}{west}{east}}

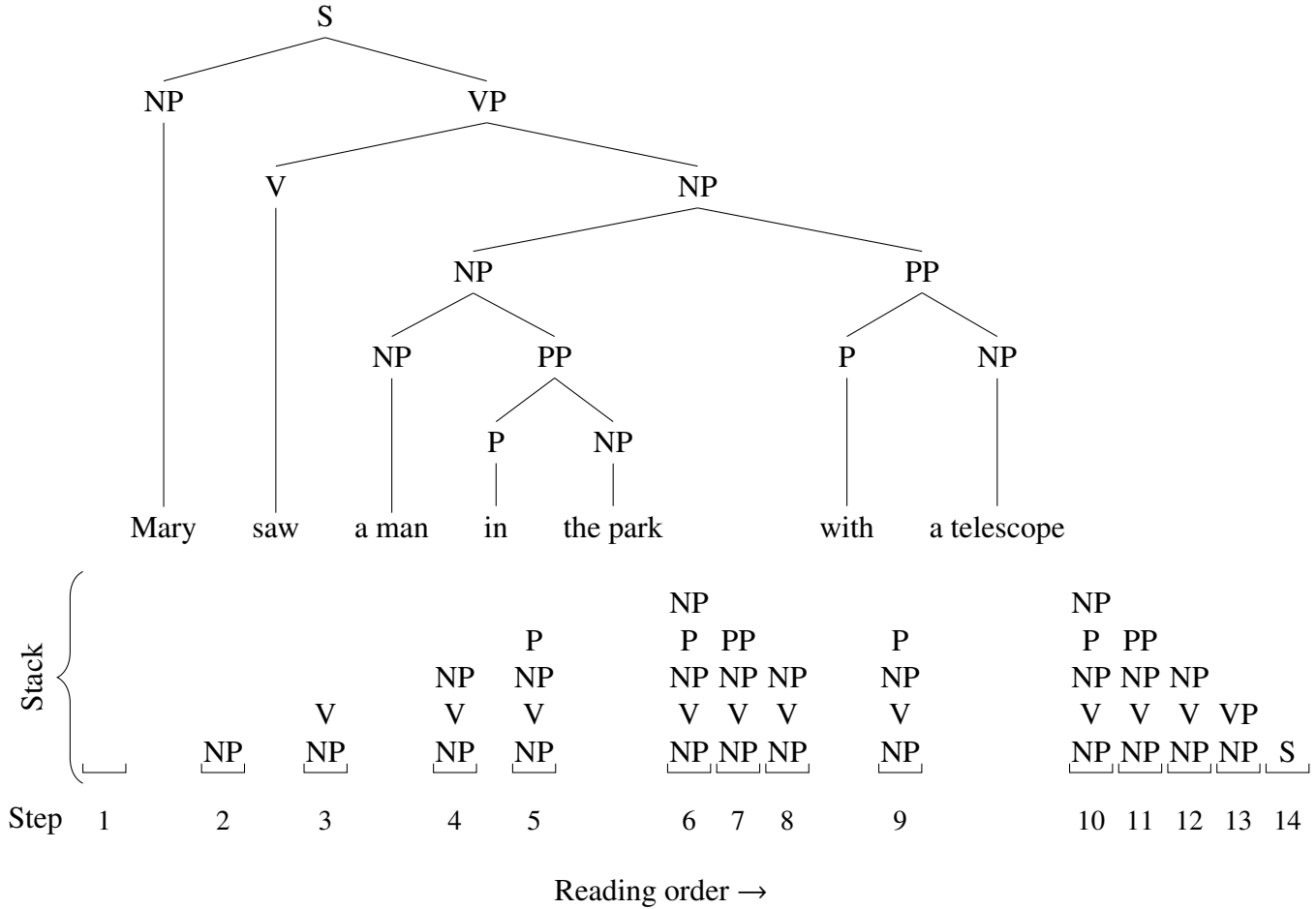
\begin{figure*}
    \centering
    \begin{forest}
        [ S,s sep=1*\seplength
            [ NP [ Mary,name=mary ] ]
            [ VP,s sep=1*\seplength
                [ V [ saw,name=saw ] ]
                [ NP,s sep=4*\seplength
                    [ NP,s sep=1*\seplength
                        [ NP [ a man,name=aman ] ]
                        [ PP,s sep=1*\seplength [ P [ in,name=in ] ] [ NP [ the park,name=thepark ] ] ]
                    ]
                    [ PP,s sep=1*\seplength [ P [ with,name=with ] ] [ NP [ a telescope,name=atelescope ] ] ]
                ]
            ]
        ]
        \stacktoleft{1}{mary}{0}{}
        \stacktoright{2}{mary}{0}{NP}
        \stacktoright{3}{saw}{0}{V \\ NP}
        \stacktoright{4}{aman}{0}{NP \\ V \\ NP}
        \stacktoright{5}{in}{0}{P \\ NP \\ V \\ NP}
        \stacktoright{6}{thepark}{0}{NP \\ P \\ NP \\ V \\ NP}
        \stacktoright{7}{thepark}{1}{PP \\ NP \\ V \\ NP}
        \stacktoright{8}{thepark}{2}{NP \\ V \\ NP}
        \stacktoright{9}{with}{0}{P \\ NP \\ V \\ NP}
        \stacktoright{10}{atelescope}{0}{NP \\ P \\ NP \\ V \\ NP}
        \stacktoright{11}{atelescope}{1}{PP \\ NP \\ V \\ NP}
        \stacktoright{12}{atelescope}{2}{NP \\ V \\ NP}
        \stacktoright{13}{atelescope}{3}{VP \\ NP}
        \stacktoright{14}{atelescope}{4}{S}
        \node[fit=(step1) (step6)] (stackrow) {};
        \draw[name=brace, decorate,decoration={brace,amplitude=10pt}] (stackrow.south west) -- (stackrow.north west);
        \node[left=20pt of stackrow, anchor=center, rotate=90] {Stack};
        \node[below=2ex] at (current bounding box.south) {Reading order $\rightarrow$};
        \node[left=20pt of stepnum1, anchor=center] {Step};
    \end{forest}
    \caption{An example of bottom-up parsing using a stack.}
    \label{fig:stack-parsing-example-bottom-up}
\end{figure*}
\end{landscape}

Several ways of adding stacks to neural networks have been proposed in past work. However, they have been designed to handle essentially one interpretation of a sentence's grammatical structure at a time\dash{}in other words, they are \term{deterministic}. As seen in \cref{fig:intro-example-parse-trees}, this can be problematic even on extremely simple sentences. To remedy this issue, the stack should be made \term{nondeterministic}\dash{}that is, able to fork into multiple versions in parallel in order to explore all possibilities. To see the problem, consider \cref{fig:stack-parsing-example-top-down}. Right before the system processes ``Mary,'' how does it know to match \natvar{NP} with ``\natterm{Mary}'' instead of expanding it into \natvar{NP} \natvar{PP}? In truth, we know from \cref{fig:intro-example-parse-trees} that \emph{both} choices lead to valid parses, and so the stack should try both possibilities. \Cref{fig:stack-parsing-example-top-down-nondeterministic} shows an example of this mechanism at work for a top-down parser. In fact, some psycholinguistic theories have suggested that humans really do process sentences in similar fashion, based on evidence from eye-tracking experiments \citep{levy-2008-expectation}.

\begin{landscape}    

\def\stackxshift{-0.5em}

\newcommand{\stacktoside}[8]{
    \node[
        align=center,
        anchor=south,
        yshift=2ex - 14ex * #8,
        xshift={#5 * (-1.6em * #3 + \stackxshift)}
    ] (step#8x#1) at (#2.base #7) {#4};
    \stackbracket{step#8x#1.south}
    \node[
        align=center,
        below=2ex of step#8x#1.south
    ] (stepnum#8x#1) {\small #1};}

\newcommand{\stacktoleft}[5]{
    \stacktoside{#2}{#3}{#4}{#5}{1}{east}{west}{#1}}

\newcommand{\stacktoright}[5]{
    \stacktoside{#2}{#3}{#4}{#5}{-1}{west}{east}{#1}}

\newcommand{\wordnodefirst}[3]{\node[text height=0pt, text depth=0pt #3] (#1) {#2};}
\newcommand{\wordnode}[4]{\wordnodefirst{#1}{#2}{, right=#4*1.4em of #3}}

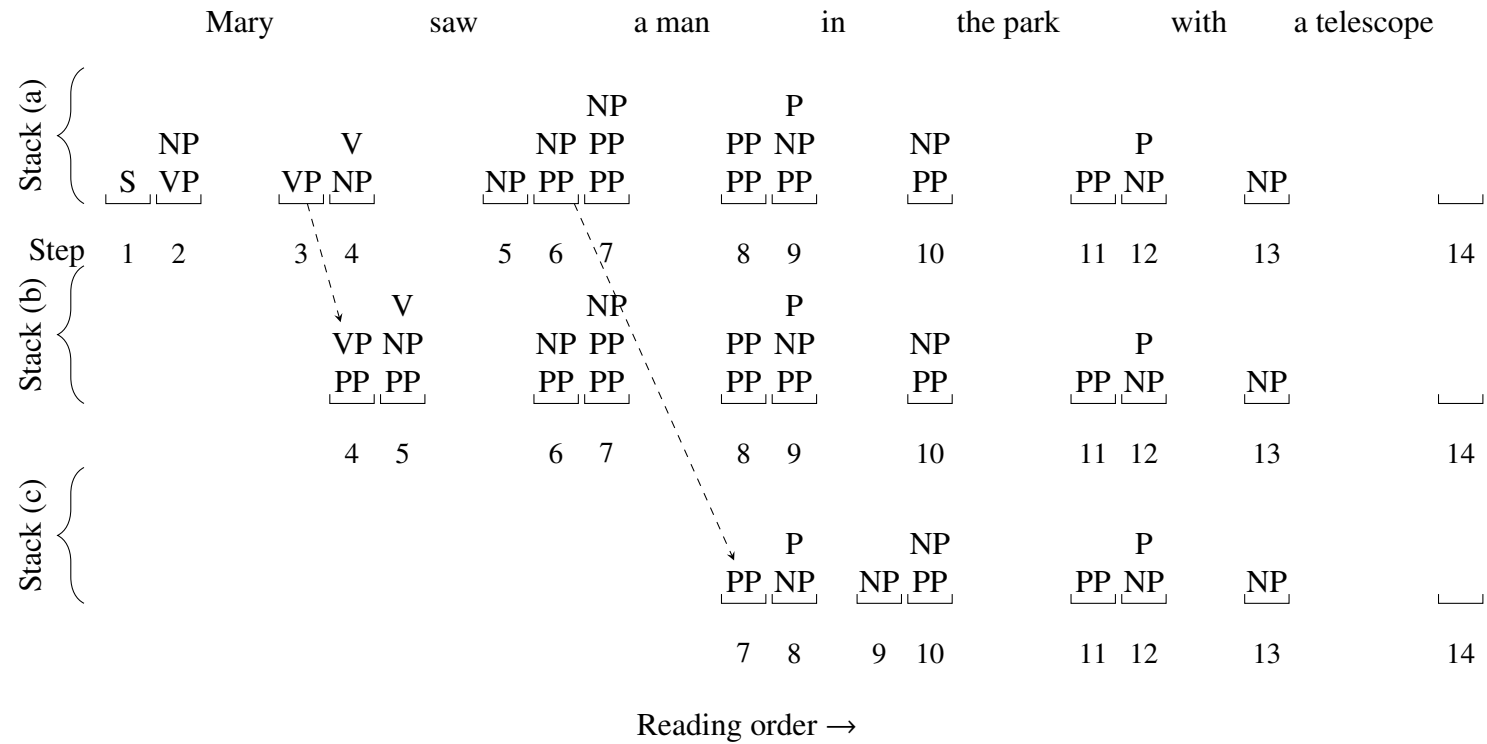
\begin{figure*}
    \centering
    \begin{tikzpicture}
        \wordnodefirst{mary}{Mary}{}
        \wordnode{saw}{saw}{mary}{3}
        \wordnode{aman}{a man}{saw}{3}
        \wordnode{in}{in}{aman}{2}
        \wordnode{thepark}{the park}{in}{2}
        \wordnode{with}{with}{thepark}{2}
        \wordnode{atelescope}{a telescope}{with}{1}
        \stacktoleft{1}{1}{mary}{1}{S}
        \stacktoleft{1}{2}{mary}{0}{NP \\ VP}
        \stacktoleft{1}{3}{saw}{2}{VP}
        \stacktoleft{1}{4}{saw}{1}{V \\ NP}
        \stacktoleft{1}{5}{aman}{2}{NP}
        \stacktoleft{1}{6}{aman}{1}{NP \\ PP}
        \stacktoleft{1}{7}{aman}{0}{NP \\ PP \\ PP}
        \stacktoleft{1}{8}{in}{1}{PP \\ PP}
        \stacktoleft{1}{9}{in}{0}{P \\ NP \\ PP}
        \stacktoleft{1}{10}{thepark}{0}{NP \\ PP}
        \stacktoleft{1}{11}{with}{1}{PP}
        \stacktoleft{1}{12}{with}{0}{P \\ NP}
        \stacktoleft{1}{13}{atelescope}{0}{NP}
        \stacktoright{1}{14}{atelescope}{0}{}
        \stacktoleft{2}{4}{saw}{1}{VP \\ PP}
        \draw[dashed, -stealth] (step1x3) -- (step2x4);
        \stacktoleft{2}{5}{saw}{0}{V \\ NP \\ PP}
        \stacktoleft{2}{6}{aman}{1}{NP \\ PP}
        \stacktoleft{2}{7}{aman}{0}{NP \\ PP \\ PP}
        \stacktoleft{2}{8}{in}{1}{PP \\ PP}
        \stacktoleft{2}{9}{in}{0}{P \\ NP \\ PP}
        \stacktoleft{2}{10}{thepark}{0}{NP \\ PP}
        \stacktoleft{2}{11}{with}{1}{PP}
        \stacktoleft{2}{12}{with}{0}{P \\ NP}
        \stacktoleft{2}{13}{atelescope}{0}{NP}
        \stacktoright{2}{14}{atelescope}{0}{}
        \stacktoleft{3}{7}{in}{1}{PP}
        \stacktoleft{3}{8}{in}{0}{ \\ P \\ NP}
        \stacktoleft{3}{9}{thepark}{1}{NP}
        \stacktoleft{3}{10}{thepark}{0}{NP \\ PP}
        \stacktoleft{3}{11}{with}{1}{PP}
        \stacktoleft{3}{12}{with}{0}{P \\ NP}
        \stacktoleft{3}{13}{atelescope}{0}{NP}
        \stacktoright{3}{14}{atelescope}{0}{}
        \draw[dashed, -stealth] (step1x6) -- (step3x7);
        \node[anchor=south] (label1) at ($(step1x1 |- step1x7.south) - (3em, 0)$) {\rotatebox{90}{Stack \cref{item:ambiguity-example-man}}};
        \node[anchor=south] (label2) at ($(step1x1 |- step2x7.south) - (3em, 0)$) {\rotatebox{90}{Stack \cref{item:ambiguity-example-saw}}};
        \node[anchor=south] (label3) at ($(step1x1 |- step3x8.south) - (3em, 0)$) {\rotatebox{90}{Stack \cref{item:ambiguity-example-park}}};
        \draw[decorate,decoration={brace,amplitude=10pt,raise=-10pt}] (label1.south east) -- (label1.north east);
        \draw[decorate,decoration={brace,amplitude=10pt,raise=-10pt}] (label2.south east) -- (label2.north east);
        \draw[decorate,decoration={brace,amplitude=10pt,raise=-10pt}] (label3.south east) -- (label3.north east);
        \node[below=2ex] at (current bounding box.south) {Reading order $\rightarrow$};
        \node[left=20pt of stepnum1x1, anchor=center] {Step};
    \end{tikzpicture}
    \caption[An example of top-down parsing using a nondeterministic stack.]{An example of top-down parsing using a \emph{nondeterministic} stack. Dashed arrows indicate points where the stack forks into a new version. For simplicity, the parse trees are omitted from this figure.}
    \label{fig:stack-parsing-example-top-down-nondeterministic}
\end{figure*}
\end{landscape}

This dissertation remedies the discrepancy between the \emph{determinism} of existing stack-augmented neural networks and the necessity of \emph{nondeterminism} in processing human language, by proposing a method of incorporating nondeterministic stacks into neural networks. We show theoretically that our method is powerful enough to express all sets of grammatical rules like the one shown in \cref{fig:intro-cfg}, and we show experimentally that it is more effective than previous stack neural networks at learning both artificial syntax and natural language. We show how to incorporate it into multiple styles of neural network, namely recurrent neural networks (RNNs) and transformers.

The rest of this document is structured as follows.
\begin{itemize}
    \item In \cref{chap:background}, we review technical material that is necessary for understanding the rest of the dissertation.
    \item In \cref{chap:stack-rnns}, we discuss prior methods of adding stacks to neural networks, as well as the general framework of adding differentiable stacks to RNNs.
    \item In \cref{chap:ns-rnn}, we present the main contribution of this dissertation: a method of incorporating nondeterministic stacks into neural networks.
    \item In \cref{chap:learning-cfls}, we show experimentally that our method makes neural networks much more effective at learning context-free languages, especially those with ambiguity, validating our approach.
    \item In \cref{chap:rns-rnn}, we present changes to our method that allow neural networks to learn long-term dependencies more easily and achieve even better results on context-free languages.
    \item In \cref{chap:incremental-execution}, we present a memory-limiting technique that decreases the computational cost of our method and makes it feasible to run on arbitrarily long strings of text. We also include experiments on natural language with analysis on a syntactic benchmark.
    \item In \cref{chap:learning-non-cfls}, we show that our method is not only expressive enough to recognize any context-free language, but that, surprisingly, it can learn many non-context-free phenomena as well, including cross-serial dependencies.
    \item In \cref{chap:vrns-rnn}, we show that our method is able to transmit much more information in the stack than expected. We propose a new version of our method that uses a stack of \emph{vectors} to take advantage of neural networks' ability to pack information into compact embedding vectors. We show strong results on context-free languages.
    \item In \cref{chap:stack-attention}, we propose a method of incorporating nondeterministic stacks into the most empirically successful neural network architecture to date\dash{}the transformer\dash{}as a new kind of attention mechanism.
    \item In \cref{chap:conclusion}, we summarize our findings and discuss directions for future work.
\end{itemize}

%% file: chapters/02-background.tex
\chapter{Background}
\label{chap:background}

In this chapter, we lay the groundwork for discussing nondeterministic stacks by reviewing essential knowledge about formal languages, automata, nondeterminism, stacks, and pushdown automata. We also briefly discuss neural networks.

\section{Formal Languages}
\label{sec:background-formal-languages}

The mathematical theory of computing is built upon the study of \term{formal languages}. Formal languages are mathematical objects used to represent types of problems that computers are tasked with solving. We introduce formal languages by defining a number of technical terms.

An \term{alphabet} is a non-empty finite set of elements called \term{symbols}. Symbols may be any mathematical object; in this document, we denote them using characters in typewriter font, e.g.\ \sym{a}, \sym{b}, \sym{0}, \sym{1}. Examples of alphabets are $\{ \sym{a}, \sym{b} \}$ and $\{ \sym{0}, \sym{1}, \sym{\#} \}$. We often use the variables $\Sigma$ and $\Gamma$ to denote alphabets.

A \term{string} over an alphabet $\Sigma$ is a finite ordered sequence of symbols in $\Sigma$. We usually denote strings by writing their symbols next to each other without spaces. For example, $\sym{abaab}$ denotes a string consisting of the symbols $\sym{a}, \sym{b}, \sym{a}, \sym{a}, \sym{b}$ in that order. We often use the variables $w$, $u$, $v$, or $s$ to denote strings.

For any two strings $u$ and $v$, we denote the concatenation of $u$ and $v$ as $uv$. We sometimes also write $u \concatop v$ to mean the same thing. We use $|w|$ to denote the length of $w$, and $w_i$ to denote the $i$th symbol of $w$. We sometimes write $w = w_1 w_2 \cdots w_n$, where each $w_i \in \Sigma$, to explicitly refer to the symbols that make up $w$. If $w = w_1 w_2 \cdots w_n$, we write $w_{[i:j)}$ to denote the (possibly empty) substring $w_i w_{i+1} \cdots w_{j-1}$. We use the special letter $\emptystring$ to denote an \term{empty string} of length 0 that contains no symbols (the letter $\emptystring$ is not a symbol in the string, but merely a notational convenience). We write $\reverse{w}$ to denote the string formed by ordering the symbols of $w$ in reverse. We write $w^i$ to denote the string $w$ repeated $i$ times.

A \term{language} or \term{formal language} over an alphabet $\Sigma$ is a (possibly infinite) set of strings over $\Sigma$. Because languages are sets, we can describe them with all the usual notation and operations for sets. We often use the variable $L$ to denote languages, and we write $w \in L$ to indicate that $L$ contains string $w$. We write $L^\ast$ to denote the language of all strings formed by concatenating any number of strings in $L$, where ${}^\ast$ is called the \term{Kleene star} operator. Formally, $L^\ast = \{ x_1 x_2 \cdots x_k \mid \text{$k \geq 0$ and each $x_i \in L$} \}$. For any alphabet $\Sigma$, $\Sigma^\ast$ denotes the language of all strings over $\Sigma$.

Several times in this document, we will make use of a particular type of function on strings called a \term{homomorphism}. A homomorphism simply modifies a string by replacing each of its symbols with another string. Formally, a homomorphism $\phi \colon \Sigma \rightarrow \Gamma^\ast$ is a function from symbols in alphabet $\Sigma$ to strings over alphabet $\Gamma$. We extend it to operate on strings over $\Sigma$ by defining $\phi(w) = \phi(w_1) \, \phi(w_2) \cdots \phi(w_n)$, where $w = w_1 w_2 \cdots w_n$ and $w \in \Sigma^\ast$.

\section{Finite Automata and Nondeterminism}
\label{sec:finite-automata}

Just as mathematics studies problems in computing through the lens of formal languages, it studies computing devices using abstract machines called \term{automata}.\footnote{The singular form of this word is \term{automaton}, and the plural form is \term{automata} or \term{automatons}. It comes from the ancient Greek word \greektext{αὐτόματον} \textit{aut{\'o}maton}, meaning ``self-acting thing.''} An automaton is a machine that reads a string as input and, by proceeding through a series of steps called \term{transitions}, produces some output. We refer to the status of the machine at a particular step of computation as a \term{configuration}. A \term{recognizer} is an automaton that outputs a decision to \term{accept} or \term{reject} its input string. The set of strings it accepts is called the \term{language of the machine}, and we say that the machine \term{recognizes} that language.

A \term{weighted automaton} is an automaton whose transitions are associated with non-negative numbers called \term{weights}. Rather than producing a binary accept/reject decision, a weighted automaton outputs a weight associated with the string, derived from the weights of the transitions executed while reading it. A weighted automaton can define a probability distribution over strings if one normalizes the weights of all possible string inputs so that they sum to one.

\subsection{Deterministic Finite Automata}

We introduce automata in more formal detail by first discussing the simplest kind, \term{finite automata}, also called \term{deterministic finite automata (DFAs)}. A finite automaton is a simple computing device whose only form of memory is a finite set of \term{states}. It always starts in a designated \term{start state}, and whenever it reads a symbol from an input string, it transitions to a new state according to a \term{transition function} and advances to the next symbol in the input. It accepts its string if it ends in one of a set of designated \term{accept states}. A finite automaton is defined formally as follows.

\begin{definition}
A \defterm{finite automaton} or \defterm{deterministic finite automaton (DFA)} is a tuple $(Q, \Sigma, \delta, q_0, F)$, where
\begin{itemize}
    \item $Q$ is a finite set of elements called \term{states},
    \item $\Sigma$ is an alphabet of input symbols,
    \item $\delta : Q \times \Sigma \rightarrow Q$ is the \term{transition function},
    \item $q_0 \in Q$ is the \term{start state}, and
    \item $F \subseteq Q$ is the set of \term{accept states}.
\end{itemize}
\end{definition}

We write transitions as $q \xrightarrow{a} r$, where $\delta(q, a) = r$. We can formally define the conditions under which a finite automaton accepts an input string as follows.
\begin{definition}
\label{def:dfa-run}
Let $M = (Q, \Sigma, \delta, q_0, F)$ be a finite automaton, and let $w = w_1 w_2 \cdots w_n$ be a string over $\Sigma$. $M$ \defterm{accepts} $w$ if there is a sequence of states $r_0, r_1, \ldots, r_n \in Q$ such that
\begin{enumerate}
    \item \label{item:dfa-run-init} $r_0 = q_0$,
    \item \label{item:dfa-run-middle} $\delta(r_i, w_{i+1}) = r_{i+1}$ for $i = 0, \ldots, n-1$, and
    \item \label{item:dfa-run-accept} $r_n \in F$.
\end{enumerate}
\end{definition}
This simply restates the informal description given above in more precise terms. We call any pair $(i, q)$ where $0 \leq i \leq n$ and $q \in Q$ a \term{configuration} of $M$, as it is enough to completely describe $M$'s progress when reading a string. If $\delta(r_i, w_{i+1}) = r_{i+1}$, we say that configuration $(i, r_i)$ \term{yields} configuration $(i+1, r_{i+1})$ via transition $\tau = r_i \xrightarrow{w_{i+1}} r_{i+1}$, and we write $(i, r_i) \Rightarrow_\tau (i+1, r_{i+1})$.

We refer to any sequence of configurations and transitions conforming to \cref{item:dfa-run-init,item:dfa-run-middle} as a \term{run} on $w$. More precisely, a run is a sequence $\pi = (0, r_0), \tau_1, (1, r_1), \tau_2, \ldots, \tau_n, (n, r_n)$, where $r_0 = q_0$ and, for each $i < n$, $(i, r_i) \Rightarrow_{\tau_{i+1}} (i+1, r_{i+1})$. We call a run $\pi$ that satisfies \cref{item:dfa-run-accept} an \term{accepting run} on $w$, and we say that $\pi$ \term{scans} $w$. We sometimes treat a run as a sequence of only configurations, transitions, or states as convenient. Let $\stringsumruns{M}{w}$ denote the set of all runs of $M$ that scan $w$, and let $\allsumruns{M}$ denote the set of all runs of $M$ that scan any string.

We call a set of languages a \term{class} of language. DFAs define a class of language called the \term{regular languages} \citep{kleene-1951-representation}.
\begin{definition}
A language $L$ is a \defterm{regular language} if there is a DFA that recognizes $L$.
\end{definition}

\subsection{Nondeterministic Finite Automata}

Because of the way the transition function for DFAs is defined, every configuration yields \emph{exactly one} configuration. For this reason, we call DFAs \term{deterministic}. In contrast, in a \term{nondeterministic} automaton, the same configuration may yield \emph{any} number of configurations (including none, one, or multiple). If a configuration yields none, the run simply terminates without accepting its input. If a configuration yields multiple, the common way of interpreting a nondeterministic machine's behavior in this scenario is that the machine ``forks'' into multiple runs that execute in parallel, each with an independent copy of the whole machine. A nondeterministic recognizer accepts its input iff at least one run accepts the input string. We formalize this notion by defining \term{nondeterministic finite automata (NFAs)}.

\begin{definition}
A \defterm{nondeterministic finite automaton (NFA)} is a tuple $(Q, \Sigma, \delta, q_0, F)$, where
\begin{itemize}
    \item $Q$ is a finite set of \term{states},
    \item $\Sigma$ is an alphabet called the \term{input alphabet},
    \item $\delta : Q \times (\Sigma \cup \{\emptystring\}) \rightarrow \powerset{Q}$ is the \term{transition function},
    \item $q_0 \in Q$ is the \term{start state}, and
    \item $F \subseteq Q$ is the set of \term{accept states}.
\end{itemize}
\end{definition}

Note that the transition function may now specify any number of destination states. We write transitions as $\tau = q \xrightarrow{a} r$, where $r \in \delta(q, a)$. We also write $\tau \in \delta$ if $r \in \delta(q, a)$.

A transition $q \xrightarrow{\emptystring} r$ means that the machine can go from $q$ to $r$ without reading an input symbol. For any transition $q \xrightarrow{a} r$, if $a \neq \emptystring$, we call it \term{scanning}, and if $a = \emptystring$, we call it \term{non-scanning}. An NFA may execute any number of non-scanning transitions in a row.

An automaton that has no non-scanning transitions is called \term{real-time}; it must scan exactly one input symbol per transition.

\begin{definition}
Let $N = (Q, \Sigma, \delta, q_0, F)$ be a nondeterministic finite automaton, and let $w$ be a string over $\Sigma$. $N$ \defterm{accepts} $w$ if we can write $w$ as $w = y_1 y_2 \cdots y_m$, where each $y_i \in \Sigma \cup \{\emptystring\}$, and there is a sequence of states $r_0, r_1, \ldots, r_m \in Q$ where
\begin{itemize}
    \item $r_0 = q_0$,
    \item $r_{i+1} \in \delta(r_i, y_{i+1})$ for $i = 0, \ldots, m-1$, and
    \item $r_m \in F$.
\end{itemize}
\end{definition}
An NFA configuration $(i, q)$ yields configuration $(j, r)$ via transition $q \xrightarrow{w_{[i+1:j+1)}} r$ (note that when $w_{[i+1:j+1)} = \emptystring$, then $i = j$). Unlike configurations in a DFA, the same configuration may yield multiple distinct configurations. Note that since runs can fork into other runs at any time, the number of runs of $N$ that scan $w$ can grow exponentially in the length of $w$. However, when simulating the computation of NFAs, it is relatively easy to avoid this exponential blowup using dynamic programming.

Although NFAs seem to be capable of more advanced computation than DFAs, they also recognize exactly the class of regular languages \citep{hopcroft-ullman-1979-introduction,sipser-2013-introduction}. For this reason, we say that DFAs and NFAs are \term{equivalent} in computational power. More generally, we say that two automata or grammars are equivalent if they recognize or generate the same language, and we say that two kinds of automaton or grammar are equivalent if they recognize the same language class.

\subsection{Weighted Finite Automata}

We can extend nondeterministic finite automata to a version that assigns non-negative weights to transitions, runs, and strings. We call this a \term{weighted finite automaton (WFA)}. As in an NFA, a configuration in a WFA can produce multiple configurations. Additionally, in a WFA, each transition has a weight, and the run proceeds to the next configuration \term{with the weight} of that transition. The weight of a run is the product of the transitions followed in the run. The weight that a WFA assigns to an input string is the sum of the weights of all runs that scan that string. We call this the \term{stringsum} of that string \citep{butoi-etal-2022-algorithms}. WFAs will be especially relevant when discussing our method in \cref{chap:ns-rnn}.

\begin{definition}
A \defterm{weighted finite automaton (WFA)} is a tuple $(Q, \Sigma, \delta, q_0, F)$, where
\begin{itemize}
    \item $Q$ is a finite set of \term{states},
    \item $\Sigma$ is the \term{input alphabet},
    \item $\delta : Q \times (\Sigma \cup \{\emptystring\}) \times Q \rightarrow \realset_{\geq 0}$ is the \term{transition function},
    \item $q_0 \in Q$ is the \term{start state}, and
    \item $F \subseteq Q$ is the set of \term{accept states}.
\end{itemize}
\end{definition}
We say that transition $\tau = q \xrightarrow{a} r$ has weight $\delta(q, a, r)$. If $p = \delta(q, a, r)$, we can also write the transition as $q \xrightarrow{a / p} r$. If $M$ is a WFA, we also write $\weightoftrans{M}{\tau}$ to denote the weight of $\tau$.

\begin{definition}
\label{def:wfa-run-weight}
The \defterm{weight} of a run $\pi = \tau_1, \ldots, \tau_m$ in WFA $M$, denoted $\runweight{M}{\pi}$, is the product of the weights of its transitions.
\begin{equation}
    \runweight{M}{\pi} = \prod_{i=1}^m \weightoftrans{M}{\tau_i} \quad \text{where } \pi = \tau_1, \ldots, \tau_m
\end{equation}
\end{definition}

\begin{definition}
\label{def:wfa-stringsum}
The \defterm{stringsum} of WFA $M$ on string $w$, denoted $\stringsum{M}{w}$, is the sum of the weights of all runs of $M$ that scan $w$.
\begin{equation}
    \stringsum{M}{w} = \sum_{\mathclap{\pi \in \stringsumruns{M}{w}}} \runweight{M}{\pi}
\end{equation}
\end{definition}

\begin{definition}
\label{def:wfa-allsum}
The \defterm{allsum} of WFA $M$, denoted $\allsum{M}$, is the sum of all runs of $M$ that scan any string.
\begin{equation}
    \allsum{M} = \sum_{\mathclap{\pi \in \allsumruns{M}}} \runweight{M}{\pi}
\end{equation}
\end{definition}

A WFA $M$ defines a probability distribution over strings:
\begin{equation}
\label{eq:wfa-probability-distribution}
p_M(w) = \frac{\stringsum{M}{w}}{\allsum{M}}.    
\end{equation}

\section{Stacks and Pushdown Automata}
\label{sec:background-stacks}

Now that we have reviewed the essentials of automata and nondeterminism, we are ready to discuss nondeterministic stacks and their theoretical relationship to syntax.

\subsection{Stacks}

\term{Stacks}, sometimes called \term{pushdown stores}, are a ubiquitous data structure in computer science. A stack is an ordered container of \term{elements}, where elements can only be added or removed in last-in-first-out (LIFO) order. The restriction to LIFO ordering is the defining characteristic of stacks. We refer to the end of the stack with the least recently accessed elements as the \term{bottom}, and the end of the stack with the most recently accessed elements as the \term{top}. In this document, we sometimes write a stack as a list of elements from bottom to top, and if the elements are symbols, we write the stack as a string. For example, the string $\sym{abc}$ denotes a stack containing the symbols $\sym{a}, \sym{b}, \sym{c}$, with $\sym{c}$ on top.

We refer to the insertion of a new element on the top of a stack as a \term{push}, and the removal of the top element from a stack as a \term{pop}. Typically only the top element of a stack can be observed.

\subsection{Pushdown Automata}

It is possible to extend the definition of finite automata by augmenting them with a stack of symbols. We call these machines \term{pushdown automata (PDAs)}. Like an NFA, a PDA has a finite set of states and is nondeterministic. Its transitions, in addition to scanning input symbols and being conditioned on the state machine, also interact with a stack. In fact, when using the term \term{nondeterministic stack} in this document, we are really referring to a PDA, complete with its own finite state control. A PDA always starts in the start state with an empty stack. PDA transitions can be conditioned on the top stack symbol, pop symbols from the stack, and push symbols to the stack. There is no limit to the number of symbols that can be stored in the stack, so, unlike an NFA, a PDA does not have entirely finite memory. This gives PDAs the ability to recognize a larger class of language than DFAs and NFAs. Every time the PDA forks nondeterministically into separate runs, each run maintains its own separate copy of the stack. A PDA accepts its input if there is at least one run that ends in an accept state. According to our definition, we do not require the stack to be empty for the PDA to accept.

\begin{definition}
A \defterm{pushdown automaton (PDA)} is a tuple $(Q, \Sigma, \Gamma, \delta, q_0, F)$, where
\begin{itemize}
    \item $Q$ is a finite set of \term{states},
    \item $\Sigma$ is the \term{input alphabet},
    \item $\Gamma$ is an alphabet called the \term{stack alphabet},
    \item $\delta \colon Q \times \Gamma^\ast \times (\Sigma \cup \{\emptystring\}) \rightarrow \powerset{Q \times \Gamma^\ast}$ is the \term{transition function},
    \item $q_0$ is the \term{start state}, and
    \item $F \subseteq Q$ is the set of \term{accept states}.
\end{itemize}
\end{definition}
We write transitions as $\tau = \pdatrans{q}{a}{u}{r}{v}$, where $(r, v) \in \delta(q, a, u)$. We also write $\tau \in \delta$ if $(r, v) \in \delta(q, a, u)$.

A transition $\pdatrans{q}{a}{u}{r}{v}$ means that if the machine is in state $q$ and the string $u$ is at the top of the stack (in bottom-to-top order), then the machine can scan $a$ (which may be $\emptystring$) from the input string, pop $u$, and push $v$ (in bottom-to-top order) to the stack.

\begin{definition}
Let $P = (Q, \Sigma, \Gamma, \delta, q_0, F)$ be a pushdown automaton, and let $w$ be a string over $\Sigma$. $P$ \defterm{accepts} $w$ if we can write $w$ as $w = y_1 y_2 \cdots y_m$, where each $y_i \in \Sigma \cup \{\emptystring\}$, there is a sequence of states $r_0, r_1, \ldots, r_m \in Q$, and there is a sequence of stacks $s_0, s_1, \ldots, s_m \in \Gamma^\ast$ such that
\begin{itemize}
    \item $r_0 = q_0$ and $s_0 = \emptystring$,
    \item for all $i = 0, \ldots, m-1$, $(r_{i+1}, v) \in \delta(r_i, y_{i+1}, u)$, where $s_i = \beta u$ and $s_{i+1} = \beta v$ for some $u, v, \beta \in \Gamma^\ast$, and
    \item $r_m \in F$.
\end{itemize}
\end{definition}

A configuration of $P$ is a tuple $\pdaconfig{i}{q}{\beta}$, where $0 \leq i \leq n$, $q \in Q$, and $\beta \in \Gamma^\ast$. The string $\beta$ represents the contents of the stack in bottom-to-top order. We have $\pdaconfig{i}{q}{\beta u} \Rightarrow_\tau \pdaconfig{j}{r}{\beta v}$ if $\tau = \pdatrans{q}{w_{[i+1,j+1)}}{u}{r}{v} \in \delta$. The initial configuration is $\pdaconfig{0}{q_0}{\emptystring}$. We write $\pdarunendsin{\pi}{i}{q}{x}$ to indicate that run $\pi$ ends in configuration $\pdaconfig{i}{q}{\beta x}$ for some $\beta \in \Gamma^\ast, x \in \Gamma$.

As with NFAs, runs of a PDA can fork into other runs at any time, so the number of runs of $P$ that scan $w$ can grow exponentially in the length of $w$. The fact that runs maintain their own separate copies of the stack complicates the use of dynamic programming to avoid exponential blowup when simulating PDAs, but as we will see in \cref{sec:langs-algorithm}, it is still possible to do this with cubic time and quadratic space.

Many equivalent variations of the definition of PDAs exist. In \cref{sec:ns-rnn-definition}, we will define an equivalent normal form that is more amenable to inclusion in a neural network.

\subsection{Weighted Pushdown Automata}
\label{sec:wpda}

Just as NFAs can be extended to WFAs, PDAs can be extended to \term{weighted pushdown automata (WPDAs)}.

\begin{definition}
\label{def:wpda}
A \defterm{weighed pushdown automaton (WPDA)} is a tuple $(Q, \Sigma, \Gamma, \delta, q_0, F)$, where
\begin{itemize}
    \item $Q$ is a finite set of \term{states},
    \item $\Sigma$ is the \term{input alphabet},
    \item $\Gamma$ is the \term{stack alphabet},
    \item $\delta \colon Q \times \Gamma^\ast \times (\Sigma \cup \{\emptystring\}) \times Q \times \Gamma^\ast \rightarrow \realset_{\geq 0}$ is the \term{transition function},
    \item $q_0$ is the \term{start state}, and
    \item $F \subseteq Q$ is the set of \term{accept states}.
\end{itemize}
\end{definition}

We say that transition $\tau = \pdatrans{q}{a}{u}{r}{v}$ has weight $\delta(q, a, u, r, v)$. If $p = \delta(q, a, u, r, v)$, we can also write the transition as $\pdatrans{q}{a / w}{u}{r}{v}$. We define the weights of runs, stringsums, allsums, and probability distributions for WPDAs analogously to \cref{def:wfa-run-weight,def:wfa-stringsum,def:wfa-allsum,eq:wfa-probability-distribution}.

\subsection{Context-Free Grammars}
\label{sec:cfgs}

How do PDAs relate to syntax? To answer this question, we first discuss a method of describing languages called \term{context-free grammars (CFGs)}. CFGs capture the essence of the compositional nature of syntax. A CFG consists of two disjoint alphabets designated \term{terminals} and \term{variables}. Terminal symbols correspond to symbols in strings described by the CFG, whereas variable symbols are intermediate symbols that can be replaced with terminals or other variables. A CFG specifies a set of \term{production rules} of the form $A \rightarrow X_1 X_2 \cdots X_m$, which indicates that the variable $A$ can be replaced with the terminal/variable symbols $X_1 X_2 \cdots X_m$. A CFG has a designated \term{start variable}, and we say that it \term{generates} string $w$ if the start variable can be rewritten into $w$ through some sequence of production rules. The order in which the rules are applied describes the syntactic structure of $w$, which can be represented as a parse tree, as in \cref{fig:intro-example-parse-trees}. Multiple parse trees may be possible for the same string. An example of a CFG was given in \cref{fig:intro-cfg}.

\begin{definition}
A \defterm{context-free grammar (CFG)} is a tuple $(V, \Sigma, R, S)$, where
\begin{itemize}
    \item $V$ is a finite set called the \term{variables} or \term{nonterminals},
    \item $\Sigma$ is a finite set disjoint from $V$ called the \term{terminals},
    \item $R$ is a finite set of \term{production rules} of the form $A \rightarrow X_1 X_2 \cdots X_m$, where $A \in V$ and each $X_i \in V \cup \Sigma$, and
    \item $S \in V$ is the \term{start variable}.
\end{itemize}
\end{definition}

For any $u, v, \beta \in (V \cup \Sigma)^\ast$, if $A \rightarrow \beta \in R$, we say that string $uAv$ \term{yields} string $u \beta v$, and we write $uAv \Rightarrow u \beta v$. We say that $u$ \term{derives} $v$ if $u = v$ or if there is a sequence $u_1, u_2, \ldots, u_k$ for $k \geq 0$ where $u \Rightarrow u_1 \Rightarrow u_2 \Rightarrow \cdots \Rightarrow u_k \Rightarrow v$.

\begin{definition}
Let $G = (V, \Sigma, R, S)$ be a context-free grammar, and let $w$ be a string over $\Sigma$. $G$ \defterm{generates} $w$ if $S$ derives $w$.
\end{definition}

The \term{language} of $G$, denoted $\langof{G}$, is the set of all strings it generates. We call a sequence of production rules that rewrites $S$ into $w$ a \term{derivation} of $w$. If a CFG has more than one derivation for $w$, then we say that $G$ generates $w$ \term{ambiguously}, and if a CFG generates any string ambiguously, we call that grammar \term{ambiguous}. Many times, it is possible to rewrite an ambiguous CFG into an equivalent unambiguous one, but sometimes, this is not possible. In this case, we say that the language of $G$ is \term{inherently ambiguous}.

For every CFG $G$, there is an equivalent CFG in \term{Greibach normal form (GNF)} \citep{greibach-1965-new}.
\begin{definition}
A CFG $G = (V, \Sigma, R, S)$ is in \defterm{Greibach normal form (GNF)} if
\begin{itemize}
    \item $S$ does not appear on the right side of any rule, and
    \item every rule has one of the following forms:
    \begin{align*}
    S &\rightarrow \emptystring \\
    A &\rightarrow a B_1 \cdots B_p & p \ge 0, a \in \Sigma, B_i \in V.
    \end{align*}
\end{itemize}
\end{definition}

\subsection{Context-Free Languages}
\label{sec:cfls}

CFGs define a class of language called the \term{context-free languages (CFLs)}, which is a strict superset of the regular languages.

\begin{definition}
A language $L$ is a \defterm{context-free language (CFL)} if there is a CFG that generates it.
\end{definition}

The relationship of PDAs to syntax is this: Any CFG can be converted into an equivalent PDA, and any PDA can be converted into an equivalent CFG \citep{hopcroft-ullman-1979-introduction,sipser-2013-introduction}. In fact, the standard conversion from CFGs to PDAs converts the CFG into a top-down parser like the one described in \cref{fig:stack-parsing-example-top-down}. So, PDAs also recognize the class of context-free languages. The nondeterminism of PDAs is essential for the construction that converts CFGs to PDAs. As shown in \cref{fig:stack-parsing-example-top-down-nondeterministic}, nondeterminism allows the machine to explore all ways of replacing variables on the stack in order to discover all valid parses of the input string.

Nondeterministic PDAs have a deterministic counterpart called \term{deterministic pushdown automata (DPDAs)}. DPDAs are defined so that, for any possible configuration, at most one transition can be taken. Unlike DFAs and NFAs, DPDAs recognize a smaller class of language than PDAs called \term{deterministic context-free languages (DCFLs)}, and so the nondeterminism of PDAs makes them strictly more powerful than deterministic ones. For any CFL $L$, we say that $L$ is \term{deterministic} if there is a DPDA that recognizes it, and we say that $L$ is \term{nondeterministic} if there is not.

When DPDAs are restricted to real-time, they recognize an even smaller class of language called \term{real-time DCFLs} \citep{ginsburg-greibach-1966-deterministic,igarashi-1985-pumping}. As we will see in \cref{sec:prior-stack-rnns}, prior stack-augmented neural networks resemble real-time DPDAs. Allowing a DPDA to execute at most $k$ non-scanning transitions in a row (for some $k \geq 0$) does not help; such a DPDA is called a \term{quasi-real-time DPDA}, and quasi-real-time DPDAs are equivalent to real-time DPDAs \citep{harrison-havel-1972-real,harrison-havel-1972-family}. In contrast, real-time nondeterministic PDAs \emph{are} equivalent to nondeterministic PDAs \citep{greibach-1965-new}, so nondeterminism is sufficient for overcoming the limitations of real-time computation in PDAs. The method we describe in \cref{chap:ns-rnn} gives neural networks the full power of nondeterministic PDAs.

\section{Neural Networks}
\label{sec:neural-networks}

Over the last decade, \term{neural networks} have become the dominant machine learning paradigm in multiple fields of computer science, including natural language processing (NLP) and computer vision (CV). They are the apparent answer to the question: how does one write a computer program that uses human language, or recognizes objects in images? Language and vision are both examples of problems whose underlying rules have a level of complexity and nuance that defy even the most dedicated programmers. Neural networks automate the process of learning complex behavior by observing examples of inputs and expected outputs in a corpus of \term{training data}.

The design of neural networks was inspired by the way neurons communicate in biological brains. Neural networks come in many varieties; two of the most commonly used architectures used for processing sequences of inputs are \term{recurrent neural networks (RNNs)} \citep{elman-1990-finding} and \term{transformers} \citep{vaswani-etal-2017-attention}.

In both cases, several parts of the network consist of an apparatus called a \term{layer}. A neural network typically includes many layers. A layer receives a vector of real numbers $\vecvar{x}$ as input and produces a vector of real numbers $\vecvar{y}$ as output. The elements of these vectors are often called \term{units}, and the elements of $\vecvar{y}$ are often called \term{neurons} or \term{activations} because they resemble the activation of neurons in the brain.

The input units of $\vecvar{x}$ are connected to the output units of $\vecvar{y}$ via \term{connections}. Each connection has a real-valued \term{weight}. Each output unit also has an associated \term{bias term}. Let $\matvar{W}[j, i]$ denote the weight of the connection from unit $\vecvar{x}[i]$ to unit $\vecvar{y}[j]$, and let $\vecvar{b}[j]$ denote the bias term for $\vecvar{y}[j]$. The value of $\vecvar{y}[j]$ is defined as $\vecvar{y}[j] = f(\sum_i \matvar{W}[j, i] \, \vecvar{x}[i] + \vecvar{b}[j])$, where $f$ is some non-linear sigmoid function, typically the logistic function $\sigma$ or $\tanh$. If $\matvar{W}$ is the matrix of the weights of all connections from $\vecvar{x}$ to $\vecvar{y}$, we can express all values of $\vecvar{y}$ at once as $\vecvar{y} = f(\matvar{W} \vecvar{x} + \vecvar{b})$. In this document, whenever the input or output of a layer is a multi-dimensional tensor instead of a one-dimensional vector, we assume that it is implicitly flattened or reshaped to the appropriate dimensions.

Connection weights and bias terms are examples of \term{parameters}. Neural networks with different parameter values exhibit different behaviors; the goal of training a neural network is to find a setting of parameter values that causes it solve a desired task. Neural networks typically consist of millions or billions of parameters. We often denote the set of all parameters in a network as a single vector $\vecvargreek{\theta}$.

During training, the degree to which the network does not conform to the training data is measured with a \term{loss function}. Its output, called the \term{loss}, quantifies the dissimilarity between the network's output and the correct output. Training a neural network consists of decreasing the loss by \term{optimizing} the parameter values. This is done using an algorithm called gradient descent, whereby $\vecvargreek{\theta}$ is incrementally nudged by small amounts in directions that decrease the loss. The amount by which it is nudged is called the \term{learning rate}, and the direction in which it is nudged is determined by the \term{gradient} of the loss with respect to $\vecvargreek{\theta}$. Modern automatic differentiation software libraries like PyTorch \citep{paszke-etal-2019-pytorch} automate the process of computing gradients. Gradient descent is not guaranteed to find an optimal solution that fits the training data; it can get stuck in \term{local minima}.

The gradient of the loss with respect to the parameters is computed using a dynamic programming algorithm called \term{backpropagation}, which is based on the chain rule from differential calculus. In order for backpropagation to compute accurate gradients, all modules of the network should have outputs that are differentiable with respect to their inputs. We call the process of calculating its outputs given its inputs the \term{forward pass}. We call the process of calculating the gradient of the loss with respect to its inputs, given the gradient of the loss with respect to its outputs, the \term{backward pass}. In \cref{chap:stack-rnns}, we will discuss differentiable data structures that implement stacks. PyTorch implements backpropagation by storing operators in a directed acyclic graph called the \term{computation graph} during the forward pass, then accumulating gradients by traversing the graph in topological order in the backward pass.

Training neural networks is computationally expensive, but many of their operations can be implemented with highly parallelizable tensor operations. For this reason, neural networks are typically run on parallel hardware called graphical processing units (GPUs).

All inputs to a neural network must be encoded as numbers or vectors. A \term{one-hot} vector is a vector with exactly one element set to 1 and all others set to 0. One-hot vectors are convenient for encoding symbols of an alphabet as vectors. Let $\Sigma$ be an alphabet, and suppose we arbitrarily define an ordering $a_1, a_2, \ldots, a_m$ of the symbols in $\Sigma$. Then we can encode symbol $a_i$ as a vector $\vecvar{x}_{a_i} \in \realset^m$, where $\vecvar{x}_{a_i}[i'] = \indicator{i' = i}$. For simplicity, we treat a symbol $a_i$ and its index $i$ as interchangeable, e.g.\ $\vecvar{z}[a_i] = \vecvar{z}[i]$. We generally do not bother to explicitly define an ordering.

In this document, we train neural networks as \term{language models}. Given a string $w = w_1 w_2 \cdots w_n$, a language model $M$ is trained, for every position $t = 0, \ldots, n$, to predict the next symbol $w_{t+1}$ given the prefix $w_{[1, t+1)}$. When $t = n$, the model is trained to predict a special symbol called $\eos$ (``end of sequence''). Let $w_{n+1} = \eos$, and let $\Sigma' = \Sigma \cup \{\eos\}$. We often refer to positions as \term{timesteps}. Each $w_t$ is encoded as a vector $\rnninputt{t}$ given as input to the network at timestep $t$. For each $t = 0, 1, \ldots, n$, the network outputs a vector $\rnnoutputt{t} \in \realset^{|\Sigma'|}$ called the \term{logits}. The logits are passed through the \term{softmax} function to define a probability distribution over $\Sigma'$ for $w_{t+1}$:
\begin{equation}
    \probdistmt{M}{t}(a) = \softmax(\rnnoutputt{t})[a] = \frac{\exp(\rnnoutputt{t}[a])}{\sum_{a' \in \Sigma'} \exp(\rnnoutputt{t}[a'])}.
\end{equation}

For a multi-dimensional tensor $\tensorvar{X}$, we write $\softmaxover{a}(\tensorvar{X})$ to indicate that for every value of $a$, $\softmaxover{a}(\tensorvar{X})[a] = \softmax(\tensorvar{X}[a])$.

The loss term at each timestep $t$ is the \term{cross-entropy} $H(\hat{p}_t, \probdistmt{M}{t})$ between the correct output distribution $\hat{p}_t(a) = \indicator{a = w_{t+1}}$ and the model's distribution $\probdistmt{M}{t}$:
\begin{align}
    H(\hat{p}_t, \probdistmt{M}{t}) &= - \sum_{a \in \Sigma'} \hat{p}_t(a) \log \probdistmt{M}{t}(a) \\
        &= -\log \probdistmt{M}{t}(w_{t+1}).
\end{align}
The loss for a whole string $w$ is formed by summing or averaging the loss terms over all timesteps. Cross-entropy is measured in \term{nats} (natural units of information).

We define $M$'s probability distribution over strings in $\Sigma^\ast$ as
\begin{equation}
    \probdistm{M}(w) = \prod_{t=0}^{|w|} \probdistmt{M}{t}(w_{t+1})
\end{equation}
where $w_{|w|+1} = \eos$.

%% file: chapters/03-stack-rnns.tex
\chapter{Stack RNNs}
\label{chap:stack-rnns}

A \term{recurrent neural network (RNN)} \citep{elman-1990-finding} is a type of neural network that operates on sequences of inputs. Like a finite automaton, it scans inputs one by one and relies on a finite memory store. This memory store is a vector of \term{hidden units} $\rnnhiddenletter$ called the \term{hidden state}. Connections between RNNs and automata have been studied extensively \citep{chen-etal-2018-recurrent,schwartz-etal-2018-bridging,peng-etal-2018-rational,merrill-2019-sequential,merrill-etal-2020-formal,merrill-2021-formal}.

If the analogy between RNNs and finite automata holds true, perhaps adding stacks will increase their computational power, just as adding a stack to a finite automaton allows it to recognize CFLs. Prior work has explored a number of different ways of doing this. In this chapter, we describe an architectural framework for RNNs augmented with stacks, or \term{stack RNNs}, and we discuss two previously proposed styles of stack RNN.

\section{Motivation}

Why is adding stacks to RNNs desirable? For one, stacks can make RNNs more expressive in mathematical terms, for the reasons discussed in \cref{sec:cfls}. A limiting factor of RNNs is their reliance on memory whose size is constant across time. For example, to predict the second half of a string of the form~$w \sym{\#} \reverse{w}$, an RNN would need to store all of $w$ in its hidden state before predicting its reversal $\reverse{w}$. A memory of finite size will inevitably fail to do this for inputs exceeding a certain length. A stack of unbounded size can alleviate this problem.

Moreover, as discussed in \cref{chap:introduction}, stacks share a deep connection with syntax and hierarchical patterns, which appear in natural language \citep{chomsky-1957-syntactic}. Many machine learning problems involve sequential data with hierarchical structures, such as modeling context-free languages \citep{grefenstette-etal-2015-learning,deletang-etal-2023-neural}, evaluating mathematical expressions \citep{nangia-bowman-2018-listops,hao-etal-2018-context}, logical inference \citep{bowman-etal-2015-tree}, and modeling syntax in natural language \citep{dyer-etal-2016-recurrent,shen-etal-2019-ordered-neurons,kim-etal-2019-compound}. However, empirically, RNNs have difficulty learning to solve these tasks, or generalizing to held-out sequences, unless they have syntactic supervision or a hierarchical inductive bias \citep{vanschijndel-etal-2019-quantity,wilcox-etal-2019-hierarchical,mccoy-etal-2020-does}. Stacks might not only improve RNNs' expressivity (i.e.\ the quality of optimal solutions), but also imbue them with a hierarchical inductive bias, which would improve their ability to generalize to unseen syntactic patterns in a compositional way, and to learn from fewer training examples when they are explainable with underlying hierarchical rules. This is important, as research continues to suggest that inductive bias plays a key role in syntactic generalization and data efficiency \citep{vanschijndel-etal-2019-quantity,hu-etal-2020-systematic}.

\section{Controller-Stack Interface}
\label{sec:controller-stack-interface}

We now discuss an architectural framework common to all of the stack RNN varieties we will study in this document (we will discuss transformers later, in \cref{chap:stack-attention}). In all cases the stack is an example of a \term{differentiable data structure}, which has three basic requirements:
\begin{enumerate}
    \item \label{item:diff-requirement-continuous} The stack must have continuous inputs and outputs, rather than discrete operations. This allows it to be connected to units in the neural network, and to be differentiable.
    \item \label{item:diff-requirement-differentiable} The outputs of the stack must be differentiable with respect to its inputs. Typically the output of the stack is its top element, and its inputs are the actions that push and pop elements on the stack. This allows the stack actions to be trained jointly with the rest of the network using standard backpropagation and gradient descent.
    \item \label{item:diff-requirement-tensor} For practical purposes, it should be implemented with tensor operations that can be run efficiently on GPUs.
\end{enumerate}

In this chapter, we discuss two examples of \term{differentiable stack} from prior work, and we start by describing an architectural framework for connecting stacks to RNNs, which we illustrate in \cref{fig:stack-rnn-diagram}. In this architecture, the model consists of a simple RNN, called the \term{controller}, connected to a differentiable stack. These two modules execute side-by-side and communicate with each other at each timestep. The controller can alternatively be a variant of RNN, such as the long short-term memory (LSTM) network \citep{hochreiter-schmidhuber-1997-long}, which is what we will use in all of our experiments.

At each timestep, the stack receives \term{actions} from the controller (e.g.\ to push and pop elements). The stack simulates those actions and produces a \term{reading} vector, which represents the updated top element of the stack. The reading is fed as an extra input to the controller at the next timestep. The actions and reading consist of continuous, differentiable weights so the whole model can be trained end-to-end with backpropagation; their form and meaning vary depending on the particular style of stack.

We assume the RNN reads an input string $\inputstring = \inputsymbolt{1} \cdots \inputsymbolt{n}$ encoded as a sequence of vectors $\rnninputt{1}, \cdots, \rnninputt{n}$. The LSTM controller's memory consists of a hidden state $\rnnhiddent{t}$ and memory cell $\lstmmemorycellt{t}$ (we set $\rnnhiddent{0} = \lstmmemorycellt{0} = \veczero$). The controller computes the next state~$(\rnnhiddent{t}, \lstmmemorycellt{t})$ given the previous state, input vector $\rnninputt{t}$, and stack reading $\stackreadingt{t-1}$:
\begin{equation}
(\rnnhiddent{t}, \lstmmemorycellt{t}) =
    \funcname{LSTM}\left((
        \rnnhiddent{t-1}, \lstmmemorycellt{t-1}),
        \begin{bmatrix} \rnninputt{t} \\ \stackreadingt{t-1} \end{bmatrix}
    \right).
\end{equation}
Note that because the network reads exactly one input per timestep, it resembles a real-time automaton.

The hidden state generates the stack actions $\stackactionst{t}$ and output $\rnnoutputt{t}$. When the RNN is used as a language model, $\rnnoutputt{t}$ is used as the logits for predicting the next symbol~$\inputsymbolt{t+1}$. The previous stack and new actions generate a new stack $\stackobjectt{t}$, which produces a new reading~$\stackreadingt{t}$:
\begin{align*}
\stackactionst{t} &= \stackactionsfunc{\rnnhiddent{t}} \\
\rnnoutputt{t} &= \affine{y}{\rnnhiddent{t}} \\
\stackobjectt{t} &= \stackobjectfunc{\stackobjectt{t-1}}{\stackactionst{t}} \\
\stackreadingt{t} &= \stackreadingfunc{\stackobjectt{t}}.
\end{align*}
Here, $\weightparam{y}$ and $\biasparam{y}$ are learned parameters. Each style of stack differs only in the definitions of $\stackactionsfuncname$, $\stackobjectfuncname$, $\stackreadingfuncname$, and $\stackobjectt{0}$.

\begin{figure*}
    \centering
    \begin{tikzpicture}[x=2.5cm,y=2.3cm]
        \node (htprev2) at (-2, 0) {$\cdots$};
        \node (stprev2) at (-1.5, -0.5) {$\cdots$};
        \node (htprev) at (-1, 0) {$\rnnhiddent{t-1}$};
        \node (xtprev) at (-1, -0.8) {$\rnninputt{t-1}$};
        \node (ytprev) at (-1, 0.45) {$\rnnoutputt{t-1}$};
        \node (stprev) at (-0.5, -0.5) {$\stackobjectt{t-1}$};
        \node (ht) at (0, 0) {$\rnnhiddent{t}$};
        \node (xt) at (0, -0.8) {$\rnninputt{t}$};
        \node (yt) at (0, 0.45) {$\rnnoutputt{t}$};
        \node (st) at (0.5, -0.5) {$\stackobjectt{t}$};
        \node (htnext) at (1, 0) {$\rnnhiddent{t+1}$};
        \node (xtnext) at (1, -0.8) {$\rnninputt{t+1}$};
        \node (ytnext) at (1, 0.45) {$\rnnoutputt{t+1}$};
        \node (stnext) at (1.5, -0.5) {$\cdots$};
        \node (htnext2) at (2, 0) {$\cdots$};
        \draw[->] (htprev2) edge (htprev);
        \draw[->] (stprev2) edge node[sloped,above] {$\stackreadingt{t-2}$} (htprev);
        \draw[->] (stprev2) edge (stprev);
        \draw[->] (xtprev) edge (htprev);
        \draw[->] (htprev) edge (ytprev);
        \draw[->] (htprev) edge node[sloped,above] {$\stackactionst{t-1}$} (stprev);
        \draw[->] (stprev) edge node[sloped,above] {$\stackreadingt{t-1}$} (ht);
        \draw[->] (htprev) edge (ht);
        \draw[->] (stprev) edge (st);
        \draw[->] (ht) edge (htnext);
        \draw[->] (xt) edge (ht);
        \draw[->] (ht) edge (yt);
        \draw[->] (ht) edge node[sloped,above] {$\stackactionst{t}$} (st);
        \draw[->] (st) edge node[sloped,above] {$\stackreadingt{t}$} (htnext);
        \draw[->] (st) edge (stnext);
        \draw[->] (xtnext) edge (htnext);
        \draw[->] (htnext) edge (ytnext);
        \draw[->] (htnext) edge node[sloped,above] {$\stackactionst{t+1}$} (stnext);
        \draw[->] (htnext) edge (htnext2);
        \node[anchor=east] at (-2.1,0.45) {output \big\{};
        \node[anchor=east] at (-2.1,0) {controller \big\{};
        \node[anchor=east] at (-2.1,-0.5) {stack \big\{};
        \node[anchor=east] at (-2.1,-0.8) {input \big\{};
    \end{tikzpicture}
    \caption[Conceptual diagram of the RNN controller-stack interface.]{Conceptual diagram of the RNN controller-stack interface, unrolled across a portion of time. The LSTM memory cell $\lstmmemorycellt{t}$ is not shown.}
    \label{fig:stack-rnn-diagram}
\end{figure*}
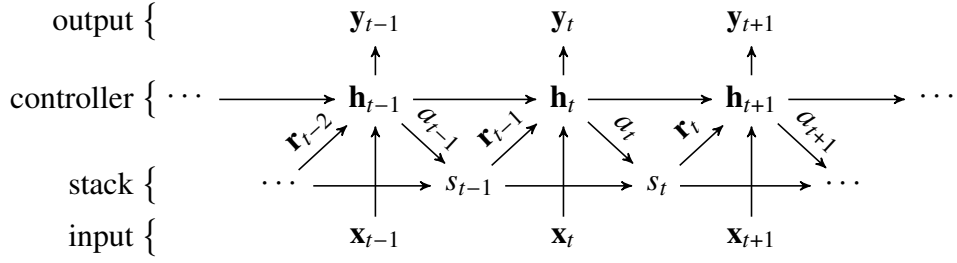

\section{Prior Stack RNNs}
\label{sec:prior-stack-rnns}

In this section, we discuss two previously proposed stack RNNs, each of which uses a different style of differentiable stack: stratification \citep{das-etal-1992-learning,sun-etal-1993-neural,grefenstette-etal-2015-learning} and superposition \citep{joulin-mikolov-2015-inferring}. We make minor changes to the original model definitions given by \citet{grefenstette-etal-2015-learning} and \citet{joulin-mikolov-2015-inferring} to ensure that all stack RNN models conform to the controller-stack interface of \cref{sec:controller-stack-interface}. This allows us to isolate differences in the style of stack data structure employed while keeping other parts of the network the same.

\subsection{Stratification}
\label{sec:stratification-stack}

Based on work by \citet{das-etal-1992-learning} and \citet{sun-etal-1993-neural}, the stack of \citet{grefenstette-etal-2015-learning} relies on a strategy we have dubbed \term{stratification}. The elements of the stack are vectors, each of which is associated with a ``thickness'' between 0 and 1, which represents the degree to which the vector element is present on the stack. A helpful analogy is that of layers of a cake; the stack elements are like cake layers of varying thickness. The stack reading is computed by examining the top slice of unit thickness and interpolating the vectors proportional to their thicknesses within that slice.

In this model, $\stackactionst{t} = (\stratpopt{t}, \stratpusht{t}, \pushedstackvectort{t})$, where the pop signal $\stratpopt{t} \in (0, 1)$ indicates the amount to be removed from the top of the stack, $\pushedstackvectort{t}$ is a learned vector to be pushed as a new element onto the stack, and the push signal $\stratpusht{t} \in (0, 1)$ is the thickness of that newly pushed vector. \Citet{grefenstette-etal-2015-learning} denote $\stratpopt{t}$ and $\stratpusht{t}$ as $\stratpoporigt{t}$ and $\stratpushorigt{t}$, respectively.

We implement the stratification stack of \citet{grefenstette-etal-2015-learning} with the following equations. In the original definition, the controller produces a hidden state $\rnnhiddent{t}$ and a separate output $\vecvar{o}'_t$ that is used to compute $\stackactionst{t}$ and $\rnnoutputt{t}$, but for simplicity and conformity to \cref{sec:controller-stack-interface}, we set $\vecvar{o}'_t = \rnnhiddent{t}$. Let $\stackvectorsizeletter = |\pushedstackvectort{t}|$ be the stack embedding size.
\begin{align*}
    \stackactionst{t} = \stackactionsfunc{\rnnhiddent{t}} &= (\stratpopt{t}, \stratpusht{t}, \pushedstackvectort{t}) \\
    \stratpopt{t} &= \logistic{\affine{push}{\rnnhiddent{t}}} \\
    \stratpusht{t} &= \logistic{\affine{pop}{\rnnhiddent{t}}} \\
    \pushedstackvectort{t} &= \tanh(\affine{v}{\rnnhiddent{t}}) \\
    \stackobjectt{t} = \stackobjectfunc{\stackobjectt{t-1}}{\stackactionst{t}} &= (\stratstacktensort{t}, \stratstrengthtensort{t}) \\
    \stratstackvector{t}{i} &= \begin{cases}
        \stratstackvector{t-1}{i} & 1 \leq i < t \\
        \pushedstackvectort{t} & i = t
    \end{cases} \\
    \stratstrength{t}{i} &= \begin{cases}
        \max(0, \stratstrength{t-1}{i} - \max(0, \stratpopt{t} - \displaystyle \sum_{j=i+1}^{t-1} \stratstrength{t-1}{j})) & 1 \leq i < t \\
        \stratpusht{t} & i = t
    \end{cases} \\
    \stratstacktensort{0} &\text{ is a $0 \times m$ matrix} \\
    \stratstrengthtensort{0} &\text{ is a vector of size 0} \\
    \stackreadingt{t} = \stackreadingfunc{\stackobjectt{t}} &= \sum_{i=1}^t (\min(\stratstrength{t}{i}, \max(0, 1 - \sum_{j=i+1}^t \stratstrength{t}{j}))) \; \stratstackvector{t}{i}
\end{align*}

This stack has quadratic time and space complexity with respect to input length (if only the forward pass is needed, the space complexity is linear). This model affords less opportunity for parallelization than the superposition stack (\cref{sec:superposition-stack}) because of the interdependence of stack elements within the same timestep.

Note that this model relies on $\min$ and $\max$ operations, which are sub-differentiable and can have gradients equal to zero. This can make training difficult, because depending on the state of the stack, their inputs might receive zero gradient and make no progress toward a better solution. In practice, the model can get trapped in local optima and requires random restarts \citep{hao-etal-2018-context}.

\Citet{yogatama-etal-2018-memory} noted that the stratification stack can implement multiple pops per timestep by allowing $\stratpopt{t} > 1$, although the push action immediately following would still be conditioned on the previous stack top $\stackreadingt{t}$. \Citet{hao-etal-2018-context} augmented this model with differentiable queues that allow it to buffer input and output and act as a transducer. \Citet{merrill-etal-2019-finding} experimented with variations of this model where $\stratpopt{t} = 1$ and $\stratpusht{t} \in (0, 4)$, $\stratpusht{t} = 1$ and $\stratpopt{t} \in (0, 4)$, and $\stratpopt{t} \in (0, 4)$ and $\stratpusht{t} \in (0, 1)$.

\subsection{Superposition}
\label{sec:superposition-stack}

The differentiable stack of \citet{joulin-mikolov-2015-inferring} uses a strategy we have dubbed \term{superposition}. It simulates a combination of partial stack actions by computing three new, separate stacks: one with all cells shifted down (push), kept the same (no-op), and shifted up (pop). The new stack is an element-wise interpolation (``superposition'') of these three stacks. Another way of viewing this is that each element is the weighted interpolation of the elements above, at, and below it at the previous timestep, weighted by push, no-op, and pop probabilities respectively. Stack elements are vectors, and $\stackactionst{t} = (\supprobst{t}, \pushedstackvectort{t})$, where the vector $\supprobst{t}$ is a probability distribution over the three stack operations. The push operation pushes vector $\pushedstackvectort{t}$, which can be learned or set to $\rnnhiddent{t}$. The stack reading is the top vector element.

This model has quadratic time and space complexity with respect to input length (if only the forward pass is needed, the space complexity is linear). Because each stack element depends only on a constant number of elements (three) from the previous timestep, the stack update can largely be parallelized.

We implement the superposition stack of \citet{joulin-mikolov-2015-inferring} with the following equations. We deviate slightly from the original definition by adding the bias terms~$\biasparam{a}$ and~$\biasparam{v}$. The original definition also connects the controller to multiple stacks that push \emph{scalars}; instead, we push a vector to a single stack, which is equivalent to multiple scalar stacks whose push/pop actions are synchronized. The original definition includes the top $k$ stack elements in the stack reading, but we only include the top element. We also treat the value of the bottom of the stack as 0 instead of $-1$.
\begin{align*}
    \stackactionst{t} = \stackactionsfunc{\rnnhiddent{t}} &= (\supprobst{t}, \pushedstackvectort{t}) \\
    \supprobst{t} &= \begin{bmatrix}
        \suppusht{t} \\
        \supnoopt{t} \\
        \suppopt{t}
    \end{bmatrix} = \softmax(\affine{a}{\rnnhiddent{t}}) \\
    \pushedstackvectort{t} &= \logistic{\affine{v}{\rnnhiddent{t}}} \\
    \stackobjectfunc{\stackobjectt{t-1}}{\stackactionst{t}} &= \supstacktensort{t} \\
    \supstackvector{t}{i} &= \begin{cases}
        \pushedstackvectort{t+1} & i = 0 \\
        \suppusht{t} \supstackvector{t-1}{i-1} + \supnoopt{t} \supstackvector{t-1}{i} + \suppopt{t} \supstackvector{t-1}{i+1} & 0 < i \leq t \\
        \veczero & i > t
    \end{cases} \\
    \stackreadingt{t} = \stackreadingfunc{\stackobjectt{t}} &= \supstackvector{t}{1}
\end{align*}

\Citet{yogatama-etal-2018-memory} developed an extension to this model called the Multipop Adaptive Computation Stack that executes a variable number of pops per timestep, up to a fixed limit $K$. They also restricted the stack to a maximum size of 10 elements, where the bottom element of a full stack is discarded when a new element is pushed; in other words, $\supstackvector{t}{i} = \veczero$ for $i > K$. \Citet{suzgun-etal-2019-memory} experimented with a modification of the parameterization of $\supprobst{t}$ and different softmax operators for normalizing the weights used to compute $\supprobst{t}$. \Citet{stogin-etal-2020-provably} proved that any real-time DPDA can be converted to a variant of the superposition stack RNN. \Citet{deletang-etal-2023-neural} tested the superposition stack RNN on a variety of string transduction tasks, with an emphasis on generalization to strings longer than those seen in training.

\section{Other Related Work}

\citet{shen-etal-2019-ordered-neurons} proposed another differentiable stack called the ordered neurons LSTM (ON-LSTM). This model modifies the gating mechanism of an LSTM so that neurons are activated in LIFO order. However, the stack is limited to a fixed depth, so it is not expected to handle strings of arbitrary length in a CFL.

Past work has proposed incorporating other data structures specialized for hierarchical patterns into neural networks, including context-free grammars \citep{kim-etal-2019-compound,kim-etal-2019-unsupervised,kim-2021-sequence}, trees \citep{tai-etal-2015-improved,zhu-etal-2015-long,kim-etal-2017-structured,choi-etal-2018-learning,havrylov-etal-2019-cooperative,corro-titov-2019-learning,xu-etal-2021-improved}, chart parsers \citep{le-zuidema-2015-forest,maillard-etal-2017-jointly,drozdov-etal-2019-unsupervised,maveli-cohen-2022-co}, and transition-based parsers \citep{dyer-etal-2015-transition,bowman-etal-2016-fast,dyer-etal-2016-recurrent,shen-etal-2019-ordered-memory}.

%% file: chapters/04-ns-rnn.tex
\chapter{The Nondeterministic Stack RNN}
\label{chap:ns-rnn}

This chapter presents the main contribution of this dissertation: a method of incorporating nondeterministic stacks into neural networks as a differentiable module. We show how to incorporate it into the stack RNN architecture discussed in \cref{sec:controller-stack-interface}, and we call the resulting model the \term{Nondeterministic Stack RNN (NS-RNN)}.

\section{Motivation}
\label{sec:ns-rnn-motivation}

In both types of stack RNN described in \cref{sec:prior-stack-rnns}, the stack is essentially deterministic in design. Even if they represent a mixture of the results of different stack actions, they represent one version of the stack's contents at a time. In order to model strings in a nondeterministic CFL like $\{w \reverse{w} \mid w \in \{\sym{0}, \sym{1}\}^\ast \}$ scanning from left to right, it must be possible, at each timestep, for the stack to track all prefixes of the input string read so far. None of the foregoing models, to our knowledge, can represent a set of possibilities like this.

In natural language, a sentence's syntactic structure often cannot be fully resolved until its conclusion (if ever), requiring a human listener to track multiple possibilities while hearing the sentence. Past work in psycholinguistics has suggested that models that keep multiple candidate parses in memory at once can explain human reading times better than models which assume harsher computational constraints. This ability also plays an important role in calculating expectations that facilitate more efficient language processing \citep{levy-2008-expectation}. Prior RNNs do not track multiple parses, if they learn syntax generalizations at all \citep{futrell-etal-2019-neural,wilcox-etal-2019-hierarchical,mccoy-etal-2020-does}.

Therefore, we propose a new differentiable stack data structure that explicitly models a nondeterministic stack\dash{}more precisely, a nondeterministic WPDA complete with its own finite state control. We implement this efficiently using a dynamic programming algorithm for PDAs by \citet{lang-1974-deterministic}, which we reformulate in terms of differentiable tensor operations. The algorithm is able to represent an exponential number of stack configurations at once using cubic time and quadratic space complexity. As with existing stack RNN architectures, we combine this differentiable data structure with an RNN controller, and we call the resulting model a \term{Nondeterministic Stack RNN (NS-RNN)}.

We predict that nondeterminism can help language processing in two ways. First, it should improve expressivity, as a nondeterministic stack is able to model concurrent parses in ways that a deterministic stack cannot. This is important linguistically because natural language is high in syntactic ambiguity, and mathematically because CFLs form a larger language class than DCFLs.

Second, nondeterminism should improve trainability, even for DCFLs, because \emph{all} possible sequences of stack operations will contribute to the loss function, and the gradient will reward runs according to the degree to which they provide useful information in the stack reading. Intuitively, in order for a neural network to receive a reward for an action, it must try the action (that is, give it nonzero weight so that it receives gradient during backpropagation). For example, in the digit-recognition task, a classifier tries all ten digits, and is rewarded for the correct one. In a stack-augmented model, however, the space of possible action sequences is very large. Whereas a deterministic stack can only try one of them, a nondeterministic stack can try all of them and always receives a reward for the best ones. Contrast this with the way a deterministic differentiable stack is trained: at each timestep, training can only update the model from the vantage point of a single stack configuration, making the model prone to getting stuck in local minima. In \cref{chap:learning-cfls}, we will demonstrate these claims by comparing the NS-RNN to deterministic stack RNNs on formal language modeling tasks of varying complexity.

\section{Model Definition}
\label{sec:ns-rnn-definition}

The Nondeterministic Stack RNN (NS-RNN) follows the architecture of \cref{sec:controller-stack-interface}, consisting of an RNN controller connected to a differentiable data structure that simulates a WPDA (\cref{sec:wpda}). We call this differentiable data structure a \term{differentiable WPDA}.

It receives, at each timestep $t$, a tensor $\nstranst{t}$ of non-negative weights that specify the transition weights used at timestep $t$. The stack actions $\stackactionst{t}$ are $\nstranst{t}$. We denote the element of $\nstranst{t}$ that corresponds to the weight of transition $\tau = \pdatrans{q}{w_t}{u}{r}{v}$ as $\nstransweight{q}{t}{u}{r}{v}$ or $\nstranst{t}[\tau]$. Note that unlike a WPDA (\cref{def:wpda}), the transition weights of a differentiable WPDA can \emph{change} from timestep to timestep. They also do not depend directly on an input symbol $a \in \Sigma$, but because they can change for each $t$, they can still be switched on the value of $w_t$ just like a WPDA. Similarly to \cref{def:wfa-run-weight}, the values of $\nstranst{1}, \ldots, \nstranst{t}$ determine the weight of each run $\pi$, using $\nstranst{t}[\tau]$ instead of $\weightoftrans{M}{\tau}$ for each transition $\tau$.
\begin{equation}
    \wpdarunweight{\pi} = \prod_{i=1}^m \nstranst{i}[\tau_i] \quad \text{where } \pi = \tau_1, \ldots, \tau_m
\end{equation}

At each timestep, the differentiable WPDA produces a stack reading vector $\stackreadingt{t} \in \realset^{|\Gamma|}$ that depends on $\nstranst{i}$ for all $i = 1, \ldots, t$. The stack reading is the marginal distribution of top stack symbols over all possible runs ending at timestep $t$, and so summarizes all of its nondeterministic branches. Let $\pdarunendsin{\pi}{t}{r}{y}$ mean that run $\pi$ ends at timestep $t$ in state $r$ with top stack symbol $y$.
\begin{equation}
    \nsstackreadingnostate{t}{y} = \frac{
        \sum_{r' \in Q} \sum_{\pdarunendsin{\pi}{t}{r'}{y}} \wpdarunweight{\pi}
    }{
        \sum_{y' \in \Gamma} \sum_{r' \in Q} \sum_{\pdarunendsin{\pi}{t}{r'}{y'}} \wpdarunweight{\pi}
    }
    \label{eq:ns-rnn-reading-definition}
\end{equation}

Note that $\stackreadingt{t}$ is differentiable with respect to $\nstranst{1}, \ldots, \nstranst{t}$, as desired. Although \cref{eq:ns-rnn-reading-definition} appears to require a summation over an exponential number of WPDA runs, we will show in \cref{sec:ns-rnn-langs-algorithm} that it can be computed in polynomial time.

Computing \cref{eq:ns-rnn-reading-definition} under the general definition of WPDA (\cref{def:wpda}), however, presents a number of challenges.
\begin{enumerate}
    \item Since $u, v \in \Gamma^\ast$, and $\Gamma^\ast$ is infinite, $\nstranst{t}$ would need to store an infinite number of weights if it is to be learned automatically. Thus, the WPDA's ability to pop or push arbitrary numbers of symbols per transition is problematic.
    \item If we wish to make the WPDA \term{probabilistic}, it is difficult to normalize weights conditioned on $(q, u)$ when the length of $u$ is unbounded.
    \item A WPDA can execute any number of non-scanning transitions in a row, but the stack RNN architecture reads exactly one symbol per timestep. When non-scanning transitions form a cycle, they produce an infinite number of runs to be summed over. It is possible to compute this infinite summation using matrix inversions \citep{stolcke-1995-efficient}, but the matrix inversion operator could destabilize training because of numerical stability issues when computing its gradient.
    \item If a run has a stack of $\emptystring$, it has no top symbol, so it cannot contribute to \cref{eq:ns-rnn-reading-definition}, and there is no use in allowing such runs.
\end{enumerate}

In order to work around these problems, we define a normal form for PDAs called \term{restricted form}.

\begin{definition}
\label{def:restricted-pda}
A \defterm{restricted PDA} is a PDA whose transitions have one of the following forms, where $q, r \in Q$, $a \in \Sigma$, and $x, y \in \Gamma$:
\begin{align*}
    &\pdatrans{q}{a}{x}{r}{xy} && \text{push $y$ on top of $x$} \\
    &\pdatrans{q}{a}{x}{r}{y} && \text{replace $x$ with $y$} \\
    &\pdatrans{q}{a}{x}{r}{\emptystring} && \text{pop $x$.}
\end{align*}
Additionally, $\Gamma$ contains a designated \term{bottom symbol} $\bot$. The initial configuration is $\pdaconfig{0}{q_0}{\bot}$, and the PDA accepts when the final state is in $F$ \emph{and} the top stack symbol is $\bot$. We say that a PDA that has such a bottom symbol is \term{bottom-marked}.
\end{definition}
Note that according to this definition, the initial $\bot$ symbol may be replaced, and the $\bot$ symbol type may be used freely elsewhere in the stack, e.g.\ a stack of $\sym{a} \sym{b} \bot \sym{a} \bot \sym{b}$ is possible. Thanks to the presence of $\bot$ in the initial configuration, $\stackreadingt{0}$ has a value, namely $\nsstackreadingnostate{0}{y} = \indicator{y = \bot}$. If the initial stack were $\emptystring$, no runs at $t = 0$ would have a top symbol, and the denominator of \cref{eq:ns-rnn-reading-definition} would be 0.

Transitions are now limited to changing the stack size by at most 1. Let $\nsactionset{x} = \{x\}\Gamma \cup \Gamma \cup \{\emptystring\}$. Since transition weights are now restricted to $\nstransweight{q}{t}{x}{r}{v}$ where $v \in \nsactionset{x}$, $\nstranst{t}$ can be represented as a dense tensor of size $|Q| \times |\Gamma| \times |Q| \times (2 |\Gamma| + 1)$. Note that a restricted PDA can perform no-ops on the stack with transitions of the form $\pdatrans{q}{a}{x}{r}{x}$. All of the PDA's transitions are scanning, so restricted PDAs are real-time. As we will show in \cref{sec:restricted-pdas-recognize-all-cfls}, this does not reduce their recognition power.

The differentiable WPDA is based upon the weighted version of restricted PDAs.
\begin{definition}
A \defterm{restricted WPDA} is a WPDA under the same restrictions described in \cref{def:restricted-pda}.
\end{definition}

In the NS-RNN, the WPDA is probabilistic. Since all transitions are conditioned on exactly one stack symbol $x$, it is easy to make the weights of $\nstranst{t}$ probabilistic by normalizing them for each $(q, x) \in Q \times \Gamma$.
\begin{definition}
A \defterm{probabilistic restricted WPDA} is a restricted WPDA where, for all $q \in Q, a \in \Sigma, x \in \Gamma$,
\begin{equation}
    \sum_{r \in Q} \sum_{v \in \nsactionset{x}} \wpdaweight{q}{a}{x}{r}{v} = 1.
\end{equation}
\end{definition}
Whereas many definitions of a probabilistic PDA make the model generate symbols \citep{abney-etal-1999-relating}, our definition makes the PDA operations conditional on the input symbol $a$, as well as the top stack symbol $x$.

The NS-RNN computes $\nstranst{t}$ from the controller's hidden state and normalizes the weights using a $\softmax$.
\begin{equation}
    \stackactionsfunc{\rnnhiddent{t}} = \nstranst{t} = \softmaxover{q, x}(\affine{a}{\rnnhiddent{t}})
    \label{eq:ns-rnn-action-weights}
\end{equation}

\section{Equivalence of PDAs and Restricted PDAs}
\label{sec:restricted-pdas-recognize-all-cfls}

In this section, we prove that despite their restrictions, restricted PDAs lose no recognition power compared to PDAs.

\begin{proposition}
\label{thm:restricted-pdas-recognize-all-cfls}
For any PDA $P$, there is an equivalent restricted PDA.
\end{proposition}

\begin{proof}[Proof sketch]
We know that for every PDA $P$, there is an equivalent CFL, and for every CFL, there is a CFG $G$ in Greibach normal form (GNF) that generates it. We show how to convert $G$ to a special form called 2-GNF, which we use to construct a PDA like \cref{def:restricted-pda}, except transitions can push multiple symbols at once. We convert this to a PDA that pushes at most one symbol by increasing the stack alphabet size.
\end{proof}

The usual construction for removing non-scanning transitions \citep{autebert-etal-1997-context} involves converting to GNF and then converting to a PDA. However, this construction produces transitions of the form $\pdatrans{q}{a}{x}{r}{zy}$ where $x \not= z$, which our restricted form does not allow. Simulating such transitions in a restricted PDA presents a challenge because it requires performing a replace and then a push while scanning only one symbol. To simulate such transitions, we need to use a modified GNF, defined below.

\begin{lemma}
\label{thm:2gnf}
For any CFG $G$, there is a CFG equivalent to $G$ that has the following form (called \term{2-Greibach normal form}):
\begin{itemize}
    \item The start symbol $S$ does not appear on any right-hand side.
    \item Every rule has one of the following forms:
    \begin{align*}
        S &\rightarrow \emptystring \\
        A &\rightarrow a \\
        A &\rightarrow abB_1 \cdots B_p & p \ge 0.
    \end{align*}
\end{itemize}
\end{lemma}
\begin{proof}
Convert $G$ to Greibach normal form. Then for every rule $A \rightarrow aA_1 \cdots A_m$ and every rule $A_1 \rightarrow b B_1 \cdots B_\ell$, substitute the second rule into the first to obtain $A \rightarrow ab B_1 \cdots B_\ell A_2 \cdots A_m$. Then discard the first rule.
\end{proof}

\begin{lemma}
For any CFG $G$ in 2-GNF, there is a bottom-marked PDA equivalent to $G$ whose transitions all have one of the forms:
\begin{equation}
    \begin{aligned}
        &\pdatrans{q}{a}{x}{r}{x y_1 \cdots y_k} \\
        &\pdatrans{q}{a}{x}{r}{y} \\
        &\pdatrans{q}{a}{x}{r}{\emptystring}.
    \end{aligned}
    \label{eq:almost-restricted-pda}
\end{equation}
\end{lemma}
\begin{proof}
We split the rules into four cases:
\begin{align*}
    S &\rightarrow \emptystring \\
    A &\rightarrow a \\
    A &\rightarrow ab \\
    A &\rightarrow a b B_1 \cdots B_p \qquad p \geq 1.
\end{align*}
The PDA has an initial state $q_0$ and a main loop state $q_{\mathrm{loop}}$. It works by maintaining all of the unclosed constituents on the stack, which initially is $S$. The state $q_{\mathrm{loop}}$ is an accept state. For now, we allow the PDA to have one non-scanning transition, $\pdatrans{q_0}{\emptystring}{\bot}{q_{\mathrm{loop}}}{\bot S}$.

If $G$ has the rule $S \rightarrow \emptystring$, we make state $q_0$ an accept state.

For each rule in $G$ of the form $A \rightarrow a$, we add a pop transition $\pdatrans{q_{\mathrm{loop}}}{a}{A}{q_{\mathrm{loop}}}{\emptystring}$.

For each rule in $G$ of the form $A \rightarrow a b$, we add a new state $q$, a replace transition $\pdatrans{q_{\mathrm{loop}}}{a}{A}{q}{A}$, and a pop transition $\pdatrans{q}{b}{A}{q_\text{loop}}{\emptystring}$. This simply scans two symbols while popping $A$.

For each rule in $G$ of the form $A \rightarrow a b B_1 \cdots B_p$ where $p \geq 1$, we add a new state $q$, a replace transition $\pdatrans{q_{\mathrm{loop}}}{a}{A}{q}{B_p}$, and a push transition $\pdatrans{q}{b}{B_p}{q_{\mathrm{loop}}}{B_p B_{p-1} \cdots B_1}$. These two transitions are equivalent to scanning two symbols while replacing $A$ with $B_p B_{p-1} \cdots B_1$ on the stack. Note that we have taken advantage of the fact that 2-GNF affords us \emph{two} scanned input symbols to work around the restriction that push transitions of the form $\pdatrans q a x r {xy_1 \cdots y_k}$ cannot modify the top symbol and push new symbols in the same step. We have split this action into a replace transition followed by a push transition, using the state machine to remember what to scan and push after the replace transition.

Finally, we remove the non-scanning transition $\pdatrans{q_0}{\emptystring}{\bot}{q_{\mathrm{loop}}}{\bot S}$, and for every transition $\pdatrans{q_{\mathrm{loop}}}{a}{S}{r}{\alpha}$, we add a push transition $\pdatrans{q_0}{a}{\bot}{r}{\bot \alpha}$. Now all transitions are scanning.
\end{proof}

\begin{lemma}
For any PDA $P$ in the form (\ref{eq:almost-restricted-pda}), there is a PDA equivalent to $P$ in restricted form.
\end{lemma}
\begin{proof}
At this point, there is a maximum length $k$ such that $\pdatrans{q}{a}{x}{r}{x \alpha}$ is a transition and $k = |\alpha|$. We redefine the stack alphabet of the PDA to be $\Gamma' = \Gamma \cup \Gamma^2 \cdots \cup \Gamma^k$. Stack symbols now represent strings of the original stack symbols. Let $\multisym{\alpha}$ denote a single stack symbol for any string $\alpha$. We replace every push transition $\pdatrans{q}{a}{x}{r}{x \beta}$ with push transitions $\pdatrans{q}{a}{\multisym{\alpha x}}{r}{\multisym{\alpha x}\, \multisym{\beta}}$ for all $\alpha \in \bigcup_{i=0}^{k-1} \Gamma^i$. We replace every replace transition $\pdatrans{q}{a}{x}{r}{y}$ with replace transitions $\pdatrans{q}{a}{\multisym{\alpha x}}{r}{\multisym{\alpha y}}$ for all $\alpha$. And we replace every pop transition $\pdatrans{q}{a}{x}{r}{\emptystring}$ with replace transitions $\pdatrans{q}{a}{\multisym{\alpha x}}{r}{\multisym{\alpha}}$ for all $\alpha$ with $|\alpha| \geq 1$, and a pop transition $\pdatrans{q}{a}{\multisym{x}}{r}{\emptystring}$.
\end{proof}

So for every CFL $L$, there exists a restricted PDA $P$ that recognizes $L$.

\section{Lang's Algorithm}
\label{sec:langs-algorithm}

In this section, we discuss how the differentiable WPDA computes $\stackreadingt{t}$ efficiently given $\nstranst{1}, \ldots, \nstranst{t}$. A naive approach to implementing \cref{eq:ns-rnn-reading-definition} on a real (deterministic) computer would be to maintain a set of PDA configurations and explicitly fork each one into $|Q| (2 |\Gamma| + 1)$ runs at every timestep. However, this would require exponential space and time. Instead, we use a dynamic programming algorithm called Lang's algorithm \citep{lang-1974-deterministic,butoi-etal-2022-algorithms} to simulate all runs in only cubic time and quadratic space. Lang's algorithm can work on arbitrary WPDAs provided they are converted to a normal form \citep{butoi-etal-2022-algorithms}. Here, we give a version of Lang's algorithm for restricted WPDAs.

Intuitively, Lang's algorithm exploits structural similarities in PDA runs. First, multiple runs can result in the same stack. For example, see how the parser in \cref{fig:stack-parsing-example-top-down-nondeterministic} has the opportunity to recombine stacks (a) and (b) at step 6. Second, for $k > 0$, a stack of height $k$ must have been derived from a stack of height~$k-1$, so in principle representing a stack of height $k$ requires only storing its top symbol and a pointer to a stack of height $k-1$. The resulting data structure is a weighted graph where edges represent individual stack symbols, and paths (of which there are exponentially many) represent stacks. In this way, it resembles the graph-structured stack used in GLR parsing for CFGs \citep{tomita-1987-efficient}. One node is designated as the root, and a set of nodes is designated as stack tops. The set of all paths from the root to the stack tops is the set of all stacks.

We may equivalently view this graph as an NFA that encodes the language of stacks. Indeed, the set of stacks at a given timestep $t$ is always a regular language \citep{autebert-etal-1997-context}, and Lang's algorithm gives an explicit construction for the NFA encoding this language. The root node is the start state, and the stack top nodes are the accept states. We call this NFA the \term{stack NFA}. Analogously, we can construct a WFA encoding the distribution over stacks of a WPDA, which we call the \term{stack WFA}.

Each stack NFA state is a tuple of the form $\stackwfastate{i}{q}{x}$, representing all configurations where the PDA is in state $q$ with stack top $x$ at time $i$. We call such a tuple a \term{configuration type}. Crucially, the stack NFA state does not need to remember the entire stack contents, just the top symbol $x$.

The stack NFA has a transition from $\stackwfastate{i}{q}{x}$ to $\stackwfastate{t}{r}{y}$ iff some run of the PDA undergoes a \term{push computation} \citep{butoi-etal-2022-algorithms} from a configuration of type $\stackwfastate{i}{q}{x}$ to a configuration of type $\stackwfastate{t}{r}{y}$. A push computation from configuration $\pdaconfig{i}{q}{\beta x}$ to configuration $\pdaconfig{t}{r}{\beta x y}$ is a partial run starting at $\pdaconfig{i}{q}{\beta x}$ and ending at $\pdaconfig{t}{r}{\beta x y}$, where the run never exposes the $x$ to the top of the stack in between, although it may push and pop any number of symbols on top of it. Upon reaching $\pdaconfig{t}{r}{\beta x y}$, it has the net effect of having pushed a single $y$ on top of the $x$ without modifying $x$. The guarantee that $x$ has not changed is essential for the correctness of Lang's algorithm.

A transition in a stack WFA from $\stackwfastate{i}{q}{x}$ to $\stackwfastate{t}{r}{y}$ with weight $p$ means that the total weight of push computations from $\pdaconfig{i}{q}{\beta x}$ to $\pdaconfig{t}{r}{\beta x y}$ for all $\beta$ is $p$. The set of accept states of the stack NFA/WFA changes from timestep to timestep; at step $t$, the accept states are $\{\stackwfastate{t}{q}{x} \mid q \in Q, x \in \Gamma\}$.

\Cref{fig:langs-algorithm-inference-rules} specifies Lang's algorithm for restricted WPDAs as a set of inference rules, similar to a deductive parser \citep{shieber-etal-1995-principles,goodman-1999-semiring,butoi-etal-2022-algorithms}. Each inference rule is drawn as a fragment of the stack WFA. If the transitions drawn with solid lines are present in the stack WFA, and the side conditions in the right column are met, then the transition drawn with a dashed line can be added to the stack WFA. Conceptually, the algorithm repeatedly applies inference rules to add states and transitions to the stack WFA. This program for building the stack WFA is purely functional; the inference rules by themselves do not specify the order in which they are applied, and no states or transitions are ever deleted. This will make it more amenable to being converted to differentiable tensor operations.

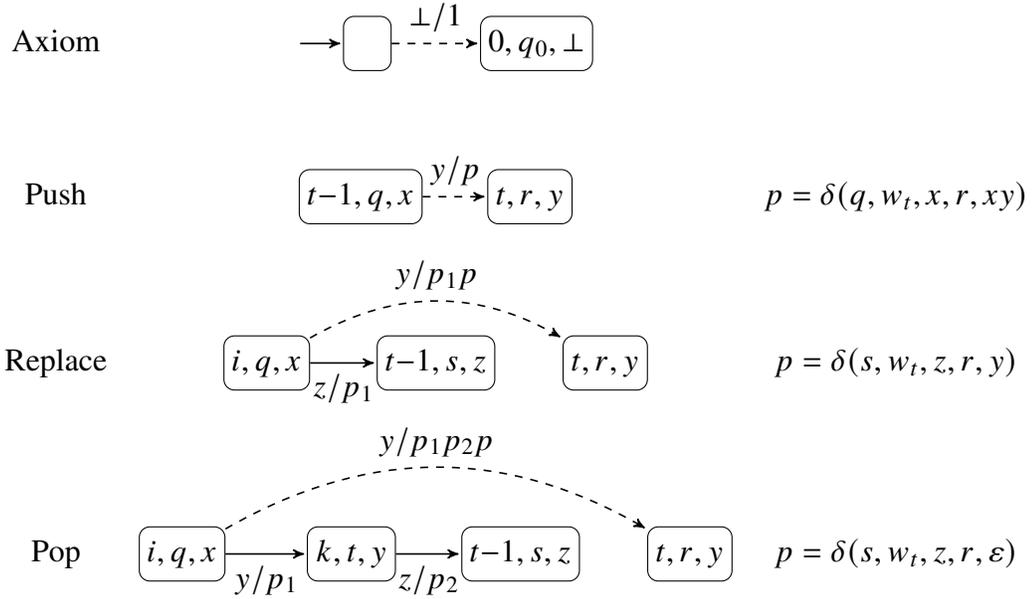
\begin{figure*}
    \tikzset{
        state/.append style={
            rectangle,
            rounded corners,
            inner sep=3pt,
            anchor=base,
            execute at begin node={\strut}
        }
    }
    \tikzset{
        x=2.25cm,
        baseline=0
    }
    \renewcommand{\arraystretch}{4}
    \centering
    \begin{tabular}{ccc}
        Axiom & 
        \begin{tikzpicture}
            \node[initial,state](q) at (0,0) {};
            \node[state](r) at (1,0) {$\stackwfastatecontent{0}{q_0}{\bot}$};
            \draw[dashed] (q) edge node {$\bot/1$} (r);
        \end{tikzpicture}
        & \\
        Push & 
        \begin{tikzpicture}
            \node[state](q1) at (1,0) {$\stackwfastatecontent{t\mathord-1}{q}{x}$};
            \node[state](q2) at (2,0) {$\stackwfastatecontent{t}{r}{y}$};
            \draw[dashed] (q1) edge node {$y / p$} (q2);
        \end{tikzpicture} &
        $p = \wpdaweight{q}{w_t}{x}{r}{xy}$
        \\
        Replace &
        \begin{tikzpicture}
            \node[state](q0) at (0,0) {$\stackwfastatecontent{i}{q}{x}$};
            \node[state](q1) at (1,0) {$\stackwfastatecontent{t\mathord-1}{s}{z}$};
            \node[state](q2) at (2,0) {$\stackwfastatecontent{t}{r}{y}$};
            \draw (q0) edge node[below] {$z / p_1$} (q1);
            \draw[dashed,bend left] (q0) edge node {$y / p_1 p$}(q2);
        \end{tikzpicture} &
        $p = \wpdaweight{s}{w_t}{z}{r}{y}$
        \\
        Pop &
        \begin{tikzpicture}
            \node[state](q0) at (0,0) {$\stackwfastatecontent{i}{q}{x}$};
            \node[state](q1) at (1,0) {$\stackwfastatecontent{k}{t}{y}$};
            \node[state](q2) at (2,0) {$\stackwfastatecontent{t\mathord-1}{s}{z}$};
            \node[state](q3) at (3,0) {$\stackwfastatecontent{t}{r}{y}$};
            \draw (q0) edge node[below] {$y / p_1$} (q1);
            \draw (q1) edge node[below] {$z / p_2$} (q2);
            \draw[dashed,bend left] (q0) edge node {$y / p_1 p_2 p$} (q3);
        \end{tikzpicture} &
        $p = \wpdaweight{s}{w_t}{z}{r}{\emptystring}$
    \end{tabular}
    \caption[Lang's algorithm as inference rules on the stack WFA.]{Lang's algorithm drawn as operations on the stack WFA. Solid edges indicate existing transitions; dashed edges indicate transitions that are added as a result of the WPDA transition shown to the right.}
    \label{fig:langs-algorithm-inference-rules}
\end{figure*}

Here we explain each of the inference rules:
\begin{description}
    \item[Axiom] creates an initial state and places the initial $\bot$ symbol on the stack.
    \item[Push] pushes a $y$ on top of an $x$ by adding a single WFA edge.
    \item[Replace] pops a $z$ and pushes a $y$ by backing up the $z$ transition (without deleting it) and adding a new $y$ transition. Because we can assume that the stack beneath $y$ has not been modified, it is safe for the new transition to share it with the old.
    \item[Pop] pops a $z$ by backing up the $z$ transition as well as the preceding $y$ transition (without deleting them) and adding a new $y$ transition. Again, the new transition shares the contents lower in the stack.
\end{description}

The total running time of the algorithm is proportional to the number of ways that the inference rules can be instantiated. Since the pop rule contains three string positions ($i$, $t$, and $k$), the time complexity is $\bigo{n^3}$. The total space requirement is characterized by the number of possible WFA transitions. Since transitions connect two states, each with a string position ($i$ and $t$), the space complexity is $\bigo{n^2}$. More precisely, the time complexity of this algorithm is $\bigo{{|Q|}^4 {|\Gamma|}^3 n^3}$, and its space complexity is $\bigo{{|Q|}^2 {|\Gamma|}^2 n^2}$. In \cref{sec:langs-algorithm-speedup}, we will show how to reduce the time complexity to $\bigo{{|Q|}^3 {|\Gamma|}^3 n^2 + {|Q|}^3 {|\Gamma|}^2 n^3}$.

\section{Example Run of Lang's Algorithm}
\label{sec:langs-algorithm-example}

As an example, consider the following restricted WPDA for the language $\{w \reverse{w} \mid w \in \{\sym{0}, \sym{1}\}^\ast\}$:
\begin{align*}
P &= (Q, \Sigma, \Gamma, \delta, q_1, F) \\
Q &= \{q_1, q_2\} \\
\Sigma &= \{\sym{0}, \sym{1}\} \\
\Gamma &= \{\sym{0}, \sym{1}, \bot\} \\
F &= Q
\end{align*}
where $\delta$ contains the transitions
\begin{align*}
\wpdaweight{q_1}{a}{x}{q_1}{xa} = 1 && x &\in \Gamma, a \in \Sigma \\
\wpdaweight{q_1}{a}{a}{q_2}{\emptystring} = 1 && a &\in \Sigma \\
\wpdaweight{q_2}{a}{a}{q_2}{\emptystring} = 1 && a &\in \Sigma
\end{align*}
with all other transition weights set to 0. This WPDA has a non-zero-weighted run ending in a configuration with an empty stack ($\bot$) iff the input string read so far is of the form $w \reverse{w}$.

An example run of Lang's algorithm is shown in \cref{fig:langs-algorithm-example}, using our example WPDA and the string $\sym{0110}$. At timestep $t=3$, the PDA reads $\sym{1}$ and either pushes a $\sym{1}$ (path ending in state $\stackwfastate{3}{q_1}{\sym{1}}$) or pops a $\sym{1}$ (path ending in state $\stackwfastate{3}{q_2}{\sym{0}}$). A similar thing happens at timestep $t=4$, and the existence of a state with top stack symbol $\bot$ indicates that the string is of the form~$w \reverse{w}$.

\begin{figure*}
    \def\rightcolx{12.3cm}
    \tikzset{
        state/.append style={
            rectangle,
            rounded corners,
            inner sep=3pt,
            anchor=base,
            execute at begin node={\strut}
        }
    }
    \tikzset{
        label/.style={
            anchor=base,
            execute at begin node={\strut}
        }
    }
    \tikzset{
        x=2cm,
        baseline=0pt,
        node distance=2cm
    }
    \centering
    \scalebox{0.9}{
    \begin{tabular}{ll}
        $t=0$ &
        \begin{tikzpicture}
            \node[initial,state] (start) {};
            \node[accepting,state,right of=start] (0q1bot) {$\stackwfastatecontent{0}{q_1}{\bot}$};
            \draw[dashed] (start) edge node {$\bot$} (0q1bot);
        \end{tikzpicture}
        \\[0.5cm]
        $t=1$ &
        \begin{tikzpicture}
            \node[initial,state] (start) {};
            \node[state,right of=start] (0q1bot) {$\stackwfastatecontent{0}{q_1}{\bot}$};
            \draw (start) edge node {$\bot$} (0q1bot);
            \node[accepting,state,right of=0q1bot](1q10) {$\stackwfastatecontent{1}{q_1}{\sym{0}}$};
            \draw[dashed] (0q1bot) edge node {$\sym{0}$} (1q10);
            \coordinate (p) at (\rightcolx,0);
            \node[label] at (1q10.base -| p) {$\pdatrans{q_1}{\sym{0}}{\bot}{q_1}{\sym{0}}$};
        \end{tikzpicture}
        \\[0.5cm]
        $t=2$ &
        \begin{tikzpicture}
            \node[initial,state] (start) {};
            \node[state,right of=start] (0q1bot) {$\stackwfastatecontent{0}{q_1}{\bot}$};
            \draw (start) edge node {$\bot$} (0q1bot);
            \node[state,right of=0q1bot](1q10) {$\stackwfastatecontent{1}{q_1}{\sym{0}}$};
            \draw (0q1bot) edge node {$\sym{0}$} (1q10);
            \node[accepting,state,right of=1q10](2q11) {$\stackwfastatecontent{2}{q_1}{\sym{1}}$};
            \draw[dashed] (1q10) edge node {$\sym{1}$} (2q11);
            \coordinate (p) at (\rightcolx,0);
            \node[anchor=base] at (2q11.base -| p) {$\pdatrans{q_1}{\sym{1}}{\sym{0}}{q_1}{\sym{1}}$};
        \end{tikzpicture}
        \\[0.5cm]
        $t=3$ &
        \begin{tikzpicture}
            \node[initial,state] (start) {};
            \node[state,right of=start] (0q1bot) {$\stackwfastatecontent{0}{q_1}{\bot}$};
            \draw (start) edge node {$\bot$} (0q1bot);
            \node[state,right of=0q1bot](1q10) {$\stackwfastatecontent{1}{q_1}{\sym{0}}$};
            \draw (0q1bot) edge node {$\sym{0}$} (1q10);
            \node[state,right of=1q10](2q11) {$\stackwfastatecontent{2}{q_1}{\sym{1}}$};
            \draw (1q10) edge node {$\sym{1}$} (2q11);
            \node[accepting,state,right of=2q11](3q11) {$\stackwfastatecontent{3}{q_1}{\sym{1}}$};
            \draw[dashed] (2q11) edge node {$\sym{1}$} (3q11);
            \node[accepting,state,below=0.5cm of 3q11](3q20) {$\stackwfastatecontent{3}{q_2}{\sym{0}}$};
            \draw[dashed,out=-30,in=180] (0q1bot) edge node {$\sym{0}$} (3q20);
            \coordinate (p) at (\rightcolx,0);
            \node[label] at (3q11.base -| p) {$\pdatrans{q_1}{\sym{1}}{\sym{1}}{q_1}{\sym{1}}$};
            \node[label] at (3q20.base -| p) {$\pdatrans{q_1}{\sym{1}}{\sym{1}}{q_2}{\emptystring}$};
        \end{tikzpicture}
        \\[1.7cm]
        $t=4$ &
        \begin{tikzpicture}
            \node[initial,state] (start) {};
            \node[state,right of=start] (0q1bot) {$\stackwfastatecontent{0}{q_1}{\bot}$};
            \draw (start) edge node {$\bot$} (0q1bot);
            \node[state,right of=0q1bot](1q10) {$\stackwfastatecontent{1}{q_1}{\sym{0}}$};
            \draw (0q1bot) edge node {$\sym{0}$} (1q10);
            \node[state,right of=1q10](2q11) {$\stackwfastatecontent{2}{q_1}{\sym{1}}$};
            \draw (1q10) edge node {$\sym{1}$} (2q11);
            \node[state,right of=2q11](3q11) {$\stackwfastatecontent{3}{q_1}{\sym{1}}$};
            \draw (2q11) edge node {$\sym{1}$} (3q11);
            \node[state,below=0.5cm of 3q11](3q20) {$\stackwfastatecontent{3}{q_2}{\sym{0}}$};
            \draw[out=-30,in=180] (0q1bot) edge node {$\sym{0}$} (3q20);
            \node[accepting,state,right of=3q11](4q10) {$\stackwfastatecontent{4}{q_1}{\sym{0}}$};
            \draw[dashed] (3q11) edge node {$\sym{0}$} (4q10);
            \node[accepting,state,below=0.5cm of 4q10](4q2bot) {$\stackwfastatecontent{4}{q_2}{\bot}$};
            \draw[dashed,every edge,rounded corners=5mm] (start) -- ($(start|-4q2bot)+(1.5,-0.75)$) to node {$\bot$} ($(4q2bot)+(-0.5,-0.75)$) -- (4q2bot);
            \coordinate (p) at (\rightcolx,0);
            \node[label] at (4q10.base -| p) {$\pdatrans{q_1}{\sym{0}}{\sym{1}}{q_1}{\sym{0}}$};
            \node[label] at (4q2bot.base -| p) {$\pdatrans{q_2}{\sym{0}}{\sym{0}}{q_2}{\emptystring}$};
        \end{tikzpicture}
    \end{tabular}
    }
    \caption[Example run of Lang's algorithm.]{Run of Lang's algorithm on our example WPDA $P$ and the string $\sym{0110}$. The WPDA transitions used are shown at right. For simplicity, only stack WFA transitions with non-zero weight are shown, and for those that are shown, their weights (which are all 1) are omitted.}
    \label{fig:langs-algorithm-example}
\end{figure*}

\section{Lang's Algorithm as Tensor Operations}
\label{sec:ns-rnn-langs-algorithm}

We now describe how to implement Lang's algorithm efficiently using differentiable tensor operations.

We represent the transition weights of the stack WFA as a tensor $\nsinnerweightletter$ of size $(n-1) \times (n-1) \times |Q| \times |\Gamma| \times |Q| \times |\Gamma|$ called the \term{inner weights}. The element of $\nsinnerweightletter$ that contains the weight of stack WFA transition $\stackwfastate{i}{q}{x} \xrightarrow{y} \stackwfastate{t}{r}{y}$ is denoted as $\nsinnerweight{i}{q}{x}{t}{r}{y}$. For $0 \leq i < t \leq n-1$,
\begin{equation}
    \begin{split}
        &\nsinnerweight{i}{q}{x}{t}{r}{y} = \\
        &\hspace{0.1in} \begin{aligned}
            &\indicator{i=t\!-\!1} \; \nspushweight{q}{t}{x}{r}{y} && \text{push} \\
            & +\! \sum_{s,z} \nsinnerweight{i}{q}{x}{t\!-\!1}{s}{z} \; \nsreplweight{s}{t}{z}{r}{y} && \text{repl.} \\
            & +\! \sum_{\mathclap{k=i+1}}^{t-2} \sum_{u} \sum_{s,z} \nsinnerweight{i}{q}{x}{k}{u}{y} \; \nsinnerweight{k}{u}{y}{t\!-\!1}{s}{z} \; \nspopweight{s}{t}{z}{r} && \text{pop}
        \end{aligned}
    \end{split}
    \label{eq:ns-rnn-gamma}
\end{equation}
Note that these equations are a relatively straightforward conversion of the inference rules given in \cref{fig:langs-algorithm-inference-rules}. Each term tallies the total weight of forming a WFA transition (the dashed edge) through a particular type of transition, multiplying together the weights of the solid edges and summing over free variables. Because these equations are a recurrence relation, we cannot compute $\nsinnerweightletter$ all at once, but (for example) in order of increasing $t$.

Additionally, we compute a tensor $\nsforwardweightletter$ of size $n \times |Q| \times |\Gamma|$ called the \term{forward weights}. The weight $\nsforwardweight{t}{r}{y}$ is the total weight of reaching state $\stackwfastate{t}{r}{y}$ in the stack WFA from the start state.
\begin{align}
    \nsforwardweight{0}{r}{y} &= \indicator{r = q_0 \wedge y = \bot} \label{eq:ns-rnn-alpha-init} \\
    \nsforwardweight{t}{r}{y} &= \sum_{i=0}^{t-1} \sum_{q, x} \nsforwardweight{i}{q}{x} \, \nsinnerweight{i}{q}{x}{t}{r}{y} \qquad (1 \leq t \leq n-1) \label{eq:ns-rnn-alpha-recurrence}
\end{align}

Finally, we normalize the forward weights to get the marginal distribution over top stack symbols $\stackreadingt{t}$.
\begin{align}
    \nsstackreadingnostate{t}{y} &= \frac{
        \sum_r \nsforwardweight{t}{r}{y}
    }{
        \sum_{y'} \sum_r \nsforwardweight{t}{r}{y'}
    }
    \label{eq:ns-rnn-reading}
\end{align}

Note that according to \cref{sec:controller-stack-interface}, since the output logits $\rnnoutputt{t}$ of the RNN depend on $\stackreadingt{t-1}$, not $\stackreadingt{t}$, computing $\stackreadingt{n}$ is unnecessary. So, $\nstranst{t}$, $\nsinnerweightletter$, and $\nsforwardweightletter$ only need to be computed for $t \leq n-1$. This is why we use $t \leq n-1$ for \cref{eq:ns-rnn-gamma,eq:ns-rnn-alpha-recurrence}.

Also note that these equations implement slightly different behavior regarding the initial $\bot$ symbol than that described in \cref{def:restricted-pda}. In this version, the initial $\bot$ can never be replaced, and once a symbol is pushed on top of it after the first timestep, the initial $\bot$ can never be exposed on the top of the stack again. In \cref{sec:rns-rnn-bottom-symbol-fix}, we will modify \cref{eq:ns-rnn-gamma,eq:ns-rnn-alpha-init,eq:ns-rnn-alpha-recurrence} to implement the exact behavior of \cref{def:restricted-pda}.

\section{Implementation Details}
\label{sec:ns-rnn-implementation-details}

We implemented the differentiable WPDA and the NS-RNN using PyTorch \citep{paszke-etal-2019-pytorch}, and doing so efficiently required a few crucial tricks.

First, we implemented a workaround to update the $\nsinnerweightletter$ and $\nsforwardweightletter$ tensors in-place in a way that was compatible with PyTorch's automatic differentiation. This was necessary to achieve the theoretical quadratic space complexity, as the alternative was to create an out-of-place copy of both tensors at each timestep.

Second, since the weight of each run in \cref{eq:ns-rnn-reading-definition} is the product of many probabilities, we compute $\nstranst{t}$, $\nsinnerweightletter$, and $\nsforwardweightletter$ in log space to avoid underflow. All of the summations and multiplications in \cref{sec:ns-rnn-langs-algorithm} are implemented in the log semiring rather than the real semiring.

Third, we had to overcome GPU memory issues when implementing the summations of \cref{sec:ns-rnn-langs-algorithm} in the log semiring. The summations in \cref{eq:ns-rnn-gamma,eq:ns-rnn-alpha-recurrence} can be implemented efficiently using the \texttt{einsum} (short for ``Einstein summation'') function, an operator provided by many tensor libraries that works like a generalization of matrix multiplication for multi-dimensional tensors. PyTorch includes an \texttt{einsum} operator for the real semiring only.

Reimplementing \texttt{einsum} in the log semiring naively involves creating an intermediate tensor of $\bigo{n^3}$ size for the pop rule in \cref{eq:ns-rnn-gamma}, when only $\bigo{n^2}$ should be required if doing the summation in-place. To work around this, we developed a custom implementation of \texttt{einsum}\footnote{\url{https://github.com/bdusell/semiring-einsum}} that supports the log semiring (and other semirings) in a way that achieves $\bigo{n^2}$ space complexity for the pop rule. It does this by splitting the summation into fixed-size blocks, enforcing an upper bound on memory usage that does not incur an extra factor of $n$. It calculates sums on each block in serial, accumulating the results in-place in a single tensor. The multiplication and summation of terms within each block can still be fully parallelized, so the overall runtime is still quite fast compared to fully parallelizing the entire \texttt{einsum} operation. Because the operation now performs in-place operations, PyTorch's auto-differentiation cannot automate the calculation of its gradient, so we also reimplemented its backward pass, using the same blocking trick to achieve $O(n^2)$ space complexity.

\section{Benefits of Nondeterminism for Learning}
\label{sec:ns-rnn-nondeterminism-helps-learning}

In this section, we describe in more detail why we expect nondeterminism to aid training. To make a good prediction at time $t$, the NS-RNN may need a certain top stack symbol $y$, which may in turn require previous WPDA actions to be orchestrated correctly. For example, consider the language $\{v\sym{\#}\reverse{v} \mid v \in \{\sym{0}, \sym{1}\}^\ast\}$, where $n$ is odd and $\inputsymbolt{t} = \inputsymbolt{n-t+1}$ for all $t$. In order to do better than chance when predicting $\inputsymbolt{t}$ (for $t$ in the second half), the model has to push a stack symbol that encodes $\inputsymbolt{t}$ at time $(n-t+1)$, and that same symbol must be on top at time $t$. How does the model learn to do this? Assume that the gradient of the log-likelihood with respect to $\nsstackreadingnostate{t}{y}$ is positive; this gradient flows to the PDA transition probabilities via (among other things) the partial derivatives of $\log \nsforwardweightletter$ with respect to $\log \nstranstensorletter$.

To calculate these derivatives more easily, we express $\nsforwardweightletter$ directly (albeit less efficiently) in terms of $\nstranstensorletter$:
\begin{equation*}
    \nsforwardweight{t}{r}{y} = \sum_{ \pdarunendsinnot{\pdatranst{\pdatransletter}{1} \cdots \pdatranst{\pdatransletter}{t}}{r}{y}} \prod_{i=1}^t \nstransttype{i}{\pdatranst{\pdatransletter}{i}}
\end{equation*}
where each $\pdatranst{\pdatransletter}{i}$ is a PDA transition of the form $\pdatransnoinput{q_1}{x_1}{q_2}{x_2}$, and the summation over $\pdarunendsinnot{\pdatranst{\pdatransletter}{1} \cdots \pdatranst{\pdatransletter}{t}}{r}{y}$ means that after following transitions $\pdatranst{\pdatransletter}{1}, \ldots, \pdatranst{\pdatransletter}{t}$, then the PDA will be in state $r$ and its top stack symbol will be $y$. Then the partial derivatives are:
\begin{align*}
    \frac
        {\partial \log \nsforwardweight{t}{r}{y}}
        {\partial \log \nstransttype{i}{\pdatransletter}} &=
    \frac
        {
            \sum_{\pdarunendsinnot{\pdatranst{\pdatransletter}{1} \cdots \pdatranst{\pdatransletter}{t}}{r}{y}}
                \left(\prod_{i'=1}^{t} \nstransttype{i'}{\pdatranst{\pdatransletter}{i'}}\right) \, \indicator{\pdatranst{\pdatransletter}{i} = \pdatransletter}
        }
        {
            \sum_{\pdarunendsinnot{\pdatranst{\pdatransletter}{1} \cdots \pdatranst{\pdatransletter}{t}}{r}{y}}
                \prod_{i'=1}^{t} \nstransttype{i}{\pdatranst{\pdatransletter}{i'}}
        }.
\end{align*}
This is the posterior probability of having used transition $\pdatransletter$ at time $i$, given that the PDA has read the input up to time $t$ and reached state $r$ and top stack symbol $y$. 

So if a correct prediction at time $t$ depends on a stack action at an earlier time $i$, the gradient flow to that action is proportional to its probability given the correct prediction. As desired, this probability is always nonzero.

\section{Conclusion}

We presented the Nondeterministic Stack RNN (NS-RNN), a stack RNN in which the differentiable stack module is a differentiable WPDA. The differentiable WPDA satisfies all three requirements of a differentiable stack as laid out in \cref{sec:controller-stack-interface}.
\begin{enumerate}
    \item The RNN controller provides the differentiable WPDA with transition weights $\nstranst{t}$, and the differentiable WPDA returns a stack reading in the form of a marginal distribution of top stack symbols over all runs (\cref{sec:ns-rnn-definition}). So, the inputs and outputs of the differentiable WPDA are continuous.
    \item The stack reading is differentiable with respect to the stack actions (\cref{sec:ns-rnn-definition,sec:ns-rnn-nondeterminism-helps-learning}).
    \item It can be implemented with efficient tensor operations (\cref{sec:ns-rnn-langs-algorithm}).
\end{enumerate}
In the next chapter, we will examine how well the NS-RNN works in practice.

%% file: chapters/05-learning-cfls.tex
\chapter{Learning Context-Free Languages}
\label{chap:learning-cfls}

In this chapter, we train our NS-RNN as a language model on various context-free languages, comparing it against a baseline LSTM and the two prior stack RNNs discussed in \cref{sec:prior-stack-rnns}. We see that the NS-RNN achieves much lower cross-entropy, in much fewer parameter updates, on the deterministic language \MarkedReversal{}, indicating that nondeterminism aids training as claimed in \cref{chap:ns-rnn}. We also see that the NS-RNN achieves lower cross-entropy on two nondeterministic CFLs (although not a third), indicating that nondeterminism empirically improves expressivity. This includes the ``hardest CFL'' \citep{greibach-1973-hardest}, a language with maximal parsing difficulty that inherently requires nondeterminism.

These results appeared in a paper published at CoNLL 2020 \citep{dusell-chiang-2020-learning}. The code used for the experiments in this chapter is publicly available.\footnote{\url{https://github.com/bdusell/nondeterministic-stack-rnn/tree/conll2020}}

\section{Context-Free Language Tasks}
\label{sec:cfl-tasks}

We examine five CFLs, described below.

\begin{description}
    \item[$\bm{\MarkedReversal{}}$] The language $\{ w \sym{\#} \reverse{w} \mid w \in \{\sym{0}, \sym{1}\}^\ast \}$, also called \term{marked reversal}. This task should be easily solvable by a model with a real-time deterministic stack, as the model can push the first $w$ to the stack, change states upon reading $\sym{\#}$, and predict $\reverse{w}$ by popping $w$ from the stack in reverse.
    \item[$\bm{\UnmarkedReversal{}}$] The language $\{ w \reverse{w} \mid w \in \{\sym{0}, \sym{1}\}^\ast \}$, also called \term{unmarked reversal}. When the length of $w$ can vary, a language model reading the string from left to right must use nondeterminism to guess where the boundary between $w$ and $\reverse{w}$ lies. As shown in \cref{sec:langs-algorithm-example}, at each position, it must either push the input symbol to the stack, or else guess that the middle point has been reached and start popping symbols from the stack. An optimal language model will interpolate among all possible split points to produce a final prediction.
    \item[$\bm{\PaddedReversal{}}$] Like \UnmarkedReversal{}, but with a higher tendency to have a long stretch of the same symbol repeated in the middle. Also called \term{padded reversal}. Strings are of the form $w a^p \reverse{w}$, where $w \in \{\sym{0}, \sym{1}\}^\ast$, $a \in \{\sym{0}, \sym{1}\}$, and $p \geq 0$. The purpose of the padding $a^p$ is to confuse a language model attempting to guess where the middle of the palindrome is based on the content of the string. In the general case of \UnmarkedReversal{}, a language model can disregard split points where a valid palindrome does not occur locally. Since all substrings of $a^p$ are palindromes, the language model must deal with a larger number of candidates simultaneously.
    \item[Dyck] The language $D_2$ of strings with two types of balanced brackets.
    \item[Hardest CFL] A language proven by \citet{greibach-1973-hardest} to be at least as difficult to parse as any other CFL. We describe it in more detail in the next section.
\end{description}

The \MarkedReversal{} and Dyck languages are deterministic CFLs that can be solved optimally with a real-time deterministic PDA. On the other hand, the \UnmarkedReversal{}, \PaddedReversal{}, and hardest CFL tasks require nondeterminism, with the hardest CFL requiring the most.

\section{The Hardest CFL}
\label{sec:hardest-cfl}

\Citet{greibach-1973-hardest} describes a CFL, $L_0$, which is the ``hardest'' CFL in the sense that an efficient parser for $L_0$ is also an efficient parser for any other CFL $L$. It is defined as follows. (We deviate from Greibach's original notation for the sake of clarity.) Every string in $L_0$ is of the following form:
\begin{equation*}
    \alpha_1 \sym{;} \alpha_2 \sym{;} \cdots \alpha_n \sym{;}
\end{equation*}
that is, a sequence of strings $\alpha_i$, each terminated by~$\sym{;}$. No $\alpha_i$ can contain $\sym{;}$. Each $\alpha_i$, in turn, is divided into three parts, separated by commas:
\begin{equation*}
    \alpha_i = x_i \sym{,} y_i \sym{,} z_i
\end{equation*}
The middle part, $y_i$, is a substring of a string in the language $\sym{\$}D_2$. The brackets in $y_i$ do not need to be balanced, but all of the $y_i$'s concatenated must form a string in $D_2$, prefixed by $\sym{\$}$. The catch is that $x_i$ and $z_i$ can be any sequence of bracket, comma, and $\sym{\$}$ symbols, so it is impossible to tell, in a single $\alpha_i$, where $y_i$ begins and ends. A parser must nondeterministically guess where each $y_i$ is, and cannot verify a guess until the end of the string is reached. We show examples of strings in $L_0$ in \cref{fig:hardest-cfl-examples}.

\begin{figure*}
    \newcommand{\unmarked}[1]{\sym{#1}}
    \newcommand{\marked}[1]{\textbf{\textcolor[HTML]{00dd00}{\sym{#1}}}}
    \centering
    \unmarked{,],} \marked{\$([} \unmarked{,;],} \marked{[[([} \unmarked{,;,} \marked{]} \unmarked{,;,[((,} \marked{)} \unmarked{,)\$)[,;,} \marked{]]]} \unmarked{,\$(;(,} \marked{)} \unmarked{,;} \\
    \unmarked{,],} \marked{\$} \unmarked{} \unmarked{,\$([;,),],\$,()),} \marked{} \marked{[} \unmarked{,\$],;,} \marked{[} \unmarked{,;,} \marked{]} \unmarked{,(;(,} \marked{]} \unmarked{,\$;} \\
    \unmarked{,(,} \marked{\$} \unmarked{} \unmarked{,[;,} \marked{} \marked{[[} \unmarked{,;,} \marked{(} \unmarked{,(,\$,[;))),} \marked{)} \unmarked{,;],} \marked{[]]]} \unmarked{,((,][,;} \\
    \unmarked{,} \marked{\$([} \unmarked{,;,} \marked{[]]} \unmarked{,\$([,;,} \marked{)} \unmarked{,)))\$[,;,} \marked{[]} \unmarked{,\$,]()][(,;} \\
    \unmarked{,} \marked{\$(([} \unmarked{,(];,} \marked{[]])([]} \unmarked{,\$,;],} \marked{)[} \unmarked{,\$;),)\$),} \marked{])} \unmarked{,(\$\$\$),)\$,;} \\
    \unmarked{,[])[,} \marked{\$} \unmarked{} \unmarked{,;(,\$,\$)),} \marked{} \marked{()[} \unmarked{,[;,} \marked{][]((} \unmarked{,),(;,((],} \marked{))} \unmarked{,;} \\
    \unmarked{,} \marked{\$(([} \unmarked{,][[]\$;,} \marked{]))(([]} \unmarked{,\$\$[\$]],),[,\$((,;,)],} \marked{))} \unmarked{,]\$;} \\
    \unmarked{],} \marked{\$((} \unmarked{,;,} \marked{(} \unmarked{,],)(;,} \marked{))} \unmarked{,]([][\$,(,[\$,)((,;,),),} \marked{)} \unmarked{,;} \\
    \unmarked{,} \marked{\$[(} \unmarked{,;,} \marked{[[]]} \unmarked{,\$[,;,[\$(,} \marked{)[((} \unmarked{,;),} \marked{(} \unmarked{,;,} \marked{)))]()]} \unmarked{,](,);} \\
    \unmarked{,](]],(((,],],} \marked{\$(} \unmarked{,;]\$,} \marked{[](} \unmarked{,;,} \marked{)[]} \unmarked{,(,;(],),} \marked{)} \unmarked{,;} \\
    \caption[Examples of strings in the hardest CFL.]{Examples of strings in the hardest CFL, sampled randomly from the PCFG in \cref{fig:hardest-cfl-pcfg} according to the procedure in \cref{sec:learning-cfls-data-sampling}. The pieces of the string from $D_2$ are shown in bold green text. The coloring is not part of the original language and is for the reader's benefit only.}
    \label{fig:hardest-cfl-examples}
\end{figure*}

The design of $L_0$ is justified as follows. Suppose we have a parser for $L_0$ which, as part of its output, identifies the start and end of each $y_i$. Given a CFG~$G$ in Greibach normal form (GNF), we can adapt the parser for $L_0$ to parse $\langof{G}$ by constructing a string homomorphism $h$, such that $w \in \langof{G}$ iff $h(w) \in L_0$, and the concatenated $y_i$'s encode a leftmost derivation of $w$ under $G$.

The homomorphism $h$ always exists and can be constructed from $G$ as follows. Let the nonterminals of $G$ be $V = \{A_1, \ldots, A_{|V|}\}$. Recall that in GNF, every rule is of the form $A_i \rightarrow a A_{j_1} \cdots A_{j_m}$, and $S$ does not appear on any right-hand side. Define
\begin{align*}
    \funcname{push}(A_i) &= \sym{(}\sym{[}^i\sym{(} \\
    \funcname{pop}(A_i) &= \begin{cases}
        \; \sym{)}\sym{]}^i\sym{)} & A_i \neq S \\
        \; \sym{\$} & A_i = S.
    \end{cases}
\end{align*}
We encode each rule of $G$ as
\begin{equation*}
    \funcname{rule}(A_i \rightarrow a A_{j_1} \cdots A_{j_m}) = \funcname{pop}(A_i) \; \funcname{push}(A_{j_1}) \cdots \funcname{push}(A_{j_m}).
\end{equation*}
Finally, we can define $h$ as
\begin{equation*}
    h(b) = \joinwithcommas_{A \rightarrow b \gamma \in G} \funcname{rule}(A \rightarrow b \gamma) \; \sym{;}
\end{equation*}
where $\joinwithcommas$ concatenates strings together delimited by commas. There is a valid string of $y_i$'s iff there is a valid derivation of $w$ with respect to~$G$.

\section{Grammars}
\label{sec:task-pcfgs}

For each task, we construct a probabilistic context-free grammar (PCFG) for the language. We list here the production rules and weights for the PCFG used for each of our tasks. Let $f(\mu) = 1 - \frac{1}{\mu + 1}$, which is the probability of failure associated with a negative binomial distribution with a mean of $\mu$ failures before one success. For a recursive PCFG rule, a probability of $f(\mu)$ results in an average of $\mu$ applications of the recursive rule.

\subsection{Grammar for \MarkedReversal{}}

We set $\mu = 60$.
\begin{alignat*}{3}
\pcfgrule{S}{\sym{0} S \sym{0}}{\tfrac{1}{2} f(\mu)}
\pcfgrule{S}{\sym{1} S \sym{1}}{\tfrac{1}{2} f(\mu)}
\pcfgrule{S}{\sym{\#}}{1 - f(\mu)}
\end{alignat*}

\subsection{Grammar for \UnmarkedReversal{}}

We set $\mu = 60$.
\begin{alignat*}{3}
\pcfgrule{S}{\sym{0} S \sym{0}}{\tfrac{1}{2} f(\mu)}
\pcfgrule{S}{\sym{1} S \sym{1}}{\tfrac{1}{2} f(\mu)}
\pcfgrule{S}{\emptystring}{1 - f(\mu)}
\end{alignat*}

\subsection{Grammar for \PaddedReversal{}}

Let $\mu_c$ be the mean length of the reversed content, and let $\mu_p$ be the mean padding length. We set $\mu_c = 60$ and $\mu_p = 30$.
\begin{alignat*}{3}
\pcfgrule{S}{\sym{0} S \sym{0}}{\tfrac{1}{2} f(\mu_c)}
\pcfgrule{S}{\sym{1} S \sym{1}}{\tfrac{1}{2} f(\mu_c)}
\pcfgrule{S}{T_0}{\tfrac{1}{2} (1 - f(\mu_c))}
\pcfgrule{S}{T_1}{\tfrac{1}{2} (1 - f(\mu_c))}
\pcfgrule{T_0}{\sym{0} T_0}{f(\mu_p)}
\pcfgrule{T_0}{\emptystring}{1 - f(\mu_p)}
\pcfgrule{T_1}{\sym{1} T_1}{f(\mu_p)}
\pcfgrule{T_1}{\emptystring}{1 - f(\mu_p)}
\end{alignat*}

\subsection{Grammar for Dyck Language}

Let $\mu_s$ be the mean number of splits, and let $\mu_n$ be the mean nesting depth. We set $\mu_s = 1$ and $\mu_n = 40$.
\begin{alignat*}{3}
\pcfgrule{S}{ST}{f(\mu_s)}
\pcfgrule{S}{T}{1 - f(\mu_s)}
\pcfgrule{T}{\sym{(} S \sym{)}}{\tfrac{1}{2} f(\mu_n)}
\pcfgrule{T}{\sym{[} S \sym{]}}{\tfrac{1}{2} f(\mu_n)}
\pcfgrule{T}{\sym{(} \sym{)}}{\tfrac{1}{2} (1 - f(\mu_n))}
\pcfgrule{T}{\sym{[} \sym{]}}{\tfrac{1}{2} (1 - f(\mu_n))}
\end{alignat*}

\subsection{Grammar for Hardest CFL}

Let $\mu_c$ be the mean number of commas, $\mu_{sf}$ be the mean short filler length, $\mu_{lf}$ be the mean long filler length, $p_s$ be the probability of a semicolon, $\mu_s$ be the mean number of bracket splits, and $\mu_n$ be the mean bracket nesting depth. We set $\mu_c = 0.5$, $\mu_{sf} = 0.5$, $\mu_{lf} = 2$, $p_s = 0.25$, $\mu_s = 1.5$, and $\mu_n = 3$. We show the PCFG in \cref{fig:hardest-cfl-pcfg}.

\begin{figure*}
    \centering
    \begin{alignat*}{3}
    \pcfgrule{S'}{R \sym{\$} Q\, S L \sym{;}}{1}
    \pcfgrule{L}{L' \sym{,} U}{1}
    \pcfgrule{L'}{\sym{,} V L'}{f(\mu_c)}
    \pcfgrule{L'}{\emptystring}{1 - f(\mu_c)}
    \pcfgrule{R}{U \sym{,} R'}{1}
    \pcfgrule{R'}{R' V \sym{,}}{f(\mu_c)}
    \pcfgrule{R'}{\emptystring}{1 - f(\mu_c)}
    \pcfgrule{U}{WU}{f(\mu_{sf})}
    \pcfgrule{U}{\emptystring}{1 - f(\mu_{sf})}
    \pcfgrule{V}{WV}{f(\mu_{lf} - 1)}
    \pcfgrule{V}{W}{1 - f(\mu_{lf} - 1)}
    \pcfgrule{W}{\sym{(}}{0.2}
    \pcfgrule{W}{\sym{)}}{0.2}
    \pcfgrule{W}{\sym{[}}{0.2}
    \pcfgrule{W}{\sym{]}}{0.2}
    \pcfgrule{W}{\sym{\$}}{0.2}
    \pcfgrule{Q}{L \sym{;} R}{p_s}
    \pcfgrule{Q}{\emptystring}{1 - p_s}
    \pcfgrule{S}{SQ\,T}{f(\mu_s)}
    \pcfgrule{S}{T}{1 - f(\mu_s)}
    \pcfgrule{T}{\sym{(} Q\,SQ \sym{)}}{\tfrac{1}{2} f(\mu_n)}
    \pcfgrule{T}{\sym{[} Q\,SQ \sym{]}}{\tfrac{1}{2} f(\mu_n)}
    \pcfgrule{T}{\sym{(} Q \sym{)}}{\tfrac{1}{2} (1 - f(\mu_n))}
    \pcfgrule{T}{\sym{[} Q \sym{]}}{\tfrac{1}{2} (1 - f(\mu_n))}
    \end{alignat*}
    \caption{PCFG for the hardest CFL.}
    \label{fig:hardest-cfl-pcfg}
\end{figure*}

\section{Data Sampling}
\label{sec:learning-cfls-data-sampling}

Every time we train a model, we randomly sample a training set of 10,000 examples from the task's PCFG, filtering samples so that the length of a string is in the interval $[40, 80]$. The training set remains the same throughout the training process and is not re-sampled from epoch to epoch, since we want to test how well the model can infer the probability distribution from a finite sample. We also randomly sample a validation set of 1,000 examples from the same distribution every time we train a model.

For each task, we sample a test set once with string lengths varying from 40 to 100, with 100 examples per length; the test set remains the same across all models and random restarts. For simplicity, we do not filter training samples from the validation or test sets, assuming that the chance of overlap is very small.

Let us describe our sampling procedure in more detail. For practical reasons, we restrict strings we sample from PCFGs to those whose lengths lie within a certain interval, say $[\ell_{\mathrm{min}}, \ell_{\mathrm{max}}]$. This is because the lengths of strings sampled randomly from PCFGs tend to have high variance, and we often want datasets to consist of strings of a certain length (e.g.\ longer strings in the test set than in the training set).

If $L$ is a language, let $\stringsoflength{L}{\ell}$ be the set of all strings in $L$ of length $\ell$. To sample a string $w \in L$, we first uniformly sample a length $\ell$ from $[\ell_{\mathrm{min}}, \ell_{\mathrm{max}}]$, then sample from $\stringsoflength{L}{\ell}$ according to its PCFG (we avoid sampling lengths for which $\stringsoflength{L}{\ell}$ is empty). This means that the distribution we are effectively sampling from is as follows. Let $p_G(w)$ be the probability of $w$ under PCFG $G$, and let $p_G(\ell)$ be the probability of all strings of length $\ell$, that is,
\begin{equation}
    p_G(\ell) = \sum_{\mathclap{w \in L_\ell }} p_G(w).    
\end{equation}
Then the distribution we are sampling from is
\begin{equation}
    p_L(w) = \frac{1}{|\{ \ell \in [\ell_{\mathrm{min}}, \ell_{\mathrm{max}}] \mid L_\ell \neq \emptyset \}|} \frac{p_G(w)}{p_G(|w|)}.
    \label{eq:cfg-task-distribution}
\end{equation}

\section{PCFG Sampling Algorithm}

We use a dynamic programming algorithm to sample strings directly from the distribution of strings in the PCFG with length $\ell$. This algorithm is adapted from an algorithm by \citet{aguinaga-etal-2019-learning} for sampling graphs of a specific size from a hyperedge replacement grammar.

The algorithm operates in two phases. The first (\cref{alg:pcfg-sampling-table}) computes a table $T$ such that every entry $T[A, \ell]$ contains the total probability of sampling a string from the PCFG with length $\ell$. The second (\cref{alg:pcfg-sampling}) uses $T$ to randomly sample a string from the PCFG (using $S$ as the nonterminal parameter $X$), restricted to those with a length of exactly $\ell$.

Let $\nonterminalsinstring{\beta}$ be an ordered sequence consisting of the nonterminals in $\beta$. Let $\compositionsoflengthof{n}{\ell}$ be a function that returns a (possibly empty) list of all compositions of $\ell$ that are of length $n$ (that is, all ordered sequences of $n$ positive integers that add up to $\ell$).

\begin{algorithm}[t]
    \caption{Computing the probability table $T$}
    \begin{algorithmic}[1]
        \Require $G$ has no $\emptystring$-rules or unary rules
        
        \Function{ComputeWeights}{$G, T, X, \ell$}
            \ForAll{rules $X \rightarrow \beta ~ / ~ p$ in $G$}
                \State $N \gets \nonterminalsinstring{\beta}$
                \State $\ell' = \ell - |\beta| + |N|$
                \For{$C$ in $\compositionsoflengthof{|N|}{\ell'}$}
                    \State $\displaystyle t[\beta, C] \gets p \times \prod_{i=1}^{|N|} T[N_i, C_{i}]$
                \EndFor
            \EndFor
            \State \Return $t$
        \EndFunction
        \Function{ComputeTable}{$G, n$}
            \For{$\ell$ from $1$ to $n$}
                \ForAll{nonterminals $X$}
                    \State $t \gets \Call{ComputeWeights}{G, T, X, \ell}$
                    \State $\displaystyle T[X, \ell] \gets \sum_{\beta, C} t[\beta, C]$
                \EndFor
            \EndFor
            \State \Return $T$
        \EndFunction
    \end{algorithmic}
    \label{alg:pcfg-sampling-table}
\end{algorithm}

\begin{algorithm}[t]
    \caption{Sampling a string using $T$}
    \begin{algorithmic}[1]
        \Require $T$ is the output of $\Call{ComputeTable}{G, \ell}$
        \Function{SampleSized}{$G, T, X, \ell$}
            \If{$T[X,\ell] = 0$}
                \State \textbf{error} \label{alg:pcfg-sampling:error}
            \EndIf
            \State $t \gets \Call{ComputeWeights}{G, T, X, \ell}$ 
            \State sample $(\beta, C)$ with probability $\displaystyle\frac{t[\beta, C]}{T[X, \ell]}$
            \State $s \gets \emptystring$
            \State $i \gets 1$
            \For{$j$ from $1$ to $|\beta|$}
                \If{$\beta_j$ is a terminal}
                    \State append $\beta_j$ to $s$
                \Else
                    \State $s' \gets \Call{SampleSized}{G, T, \beta_j, C_i}$ \label{alg:sample:expand}
                    \State append $s'$ to $s$
                    \State $i \gets i + 1$
                \EndIf
            \EndFor
            \State \Return $s$
        \EndFunction
    \end{algorithmic}
    \label{alg:pcfg-sampling}
\end{algorithm}

Because this algorithm only works on PCFGs that are free of $\emptystring$-rules and unary rules, we automatically refactor our PCFGs to remove them before providing them to the algorithm. Some of our PCFGs do not generate any strings for certain lengths, which is detected at line \ref{alg:pcfg-sampling:error} of \cref{alg:pcfg-sampling}. In this case, we avoid sampling length $\ell$ again and start over.

\section{Evaluation}
\label{sec:ns-rnn-cfl-evaluation}

In our experiments, we train neural networks as language models on strings belonging to each language. To evaluate a language model's performance, we measure its per-symbol cross-entropy on a set of strings $S$ sampled according to \cref{sec:learning-cfls-data-sampling}. Although prior work has commonly used symbol prediction accuracy, whole-sequence prediction accuracy, or string recognition accuracy to evaluate neural networks on formal languages, training neural networks as language models and measuring their cross-entropy has a number of benefits.
\begin{itemize}
    \item Whereas symbol prediction accuracy requires defining task-specific subsequences of symbols on which to evaluate, using cross-entropy does not; it only requires specifying a PCFG.
    \item For our nondeterministic CFLs, it is not possible to define subsequences of symbols that can be predicted deterministically, so accuracy is not meaningful. Cross-entropy is a soft version of accuracy that does not have this problem.
    \item Cross-entropy is more fine-grained than accuracy in the sense that it not only considers which symbols are assigned the highest probabilities, but also what the values of those probabilities are. Moreover, we can measure how close they are to the true distribution from which the data is sampled.
    \item It is difficult to train models solely as recognizers. Forcing the model to predict the next symbol at every timestep helps stabilize training. Training recognizers also requires generating both positive and negative examples, whereas language modeling only requires a positive sampling procedure.
    \item If a model accumulates just enough error to make an incorrect prediction before the end of the string, it receives no credit according to whole-sequence accuracy or recognition accuracy. According to cross-entropy, the model gets partial credit for the parts of the string it predicts well.
\end{itemize}

Let $p$ be any distribution over strings. The per-symbol cross-entropy of $p$ on a finite set of strings $S$ is defined as
\begin{equation*}
    \crossentropyofdistonset{\probletter}{S} = \frac{-\sum_{w \in S} \log \probletter(w)}{\sum_{w \in S} |w|}.
\end{equation*}
We evaluate each model according to the \term{cross-entropy difference} between its learned distribution and the true distribution. This can be seen as an approximation of the KL divergence of the neural network from the true distribution. Lower values are better. The cross-entropy of the true distribution $p_L$ is the \term{lower bound cross-entropy}.

In this chapter, the RNN models are not required to predict $\eos$. Technically, this means they estimate $\probletter(w \mid |w|)$, not $\probletter(w)$. However, they do not actually use any knowledge of the length, so it seems reasonable to compare the RNN's estimate of $\probletter(w \mid |w|)$ with the true $\probletter(w)$. (This is why, when we bin by length on the test sets in \cref{fig:learning-cfls-results}, some of the differences are negative.) In subsequent chapters, we require models to predict an $\eos$ symbol, so that they do estimate $\probletter(w)$.

When computing the lower-bound cross-entropy of the validation and test sets, we need to compute $p_L(w)$ for each string $w$, using \cref{eq:cfg-task-distribution}. Finding $p_G(w)$ requires re-parsing $w$ with respect to $G$ and summing the probabilities of all valid parses using the Inside algorithm (we actually use a weighted version of Lang's algorithm). We look up the value of $p_G(|w|)$ in the table entry $T[S, |w|]$ produced in \cref{alg:pcfg-sampling-table}.

\section{Models}

We compare our NS-RNN against three baselines: an LSTM, a stratification stack RNN (``Strat.''), and a superposition stack RNN (``Sup.''). We implement these models as described in \cref{sec:prior-stack-rnns}. For all three stack RNNs, we use an LSTM controller, and we use only one stack for the stratification and superposition stack RNNs. (In early experiments, we found that using multiple stacks did not make a meaningful difference in performance.)

We encode all input symbols as one-hot vectors. For all models, we use a single-layer LSTM with 20 hidden units. We selected this number because we found that an LSTM of this size could not completely solve \MarkedReversal{}, indicating that the hidden state is a memory bottleneck. For each task and model, we perform a hyperparameter grid search. We search for the initial learning rate, which has a large impact on performance, from the set $\{0.01, 0.005, 0.001, 0.0005\}$. For Strat.\ and Sup., we search for stack embedding sizes in $\{2, 20, 40\}$. The pushed vector in Sup.\ is learned rather than set to the hidden state. We manually choose a small number of PDA states and stack symbol types for the NS-RNN for each task. For \MarkedReversal{}, \UnmarkedReversal{}, and Dyck, we use 2 states and 2 stack symbol types. For \PaddedReversal{}, we use 3 states and 2 stack symbol types. For the hardest CFL, we use 3 states and 3 stack symbol types.

\section{Training}
\label{sec:learning-cfls-training}

As noted by \citet{grefenstette-etal-2015-learning}, initialization can play a large role in whether a stack RNN converges on algorithmic behavior or becomes trapped in a local optimum. To mitigate this, for each hyperparameter setting in the grid search, we run five random restarts and select the hyperparameter setting with the lowest average cross-entropy difference on the validation set. This gives us a picture not only of the model's performance, but of its rate of success. We initialize all fully-connected layers except for the recurrent LSTM layer with Xavier uniform initialization \citep{glorot-bengio-2010-understanding}, and all other parameters uniformly from $[-0.1, 0.1]$.

We optimize parameters with Adam \citep{kingma-ba-2015-adam} and clip gradients whose magnitude is above~5. We use mini-batches of size~10; to generate a batch, we first select a length and then sample~10 strings of that length. We randomly shuffle batches before each epoch. We train models until convergence, multiplying the learning rate by 0.9 after~5 epochs of no improvement in cross-entropy on the validation set, and stopping early after 10 epochs of no improvement.

\section{Results}
\label{sec:ns-rnn-cfl-results}

We show plots of the cross-entropy difference on the validation set between each model and the source distribution on the left of \cref{fig:learning-cfls-results}. (Note that using the relative difference in cross-entropy is important for comparing results, as the validation set is randomly sampled for every experiment.) For all tasks, stack RNNs outperform the LSTM baseline, indicating that the tasks are effective benchmarks for differentiable stacks. For the \MarkedReversal{}, \UnmarkedReversal{}, and hardest CFL tasks, our model consistently achieves cross-entropy closer to the source distribution than any other model. Even for \MarkedReversal{}, which can be solved deterministically, the NS-RNN, besides achieving lower cross-entropy on average, learns to solve the task in fewer updates and with much higher reliability across random restarts. In the case of the nondeterministic \UnmarkedReversal{} and hardest CFL tasks, the NS-RNN converges on the lowest validation cross-entropy. On the Dyck language, which is a deterministic task, all stack models converge quickly on the source distribution. We hypothesize that this is because the Dyck language represents a case where stack usage is locally advantageous at many positions within the same string, so it is particularly conducive for learning stack-like behavior. On the other hand, we note that our model struggles on \PaddedReversal{}, in which stack-friendly signals are intentionally made very distant. Although the NS-RNN outperforms the LSTM baseline, Sup.\ solves the task most effectively, though still imperfectly. We remedy this shortcoming in the NS-RNN in \cref{chap:rns-rnn}.

On the right of \cref{fig:learning-cfls-results}, we show cross-entropy on separately-sampled test data binned by string length. Recall that the training data consists of strings of length 40 to 80; in order to show how each model performs when evaluated on strings longer than those seen during training, we include strings of length up to 100 in the test set. For each length $\ell$, the lower-bound cross-entropy for that bin is calculated assuming $\ell_{\mathrm{min}} = \ell_{\mathrm{max}} = \ell$. The NS-RNN consistently performs well on string lengths it was trained on, but it is sometimes surpassed by other stack models on longer strings. Specifically, on \MarkedReversal{}, the NS-RNN generalizes poorly to strings longer than 80 symbols, and its variance increases sharply, despite being very small for lengths 40 to 80. On longer strings in the Dyck and hardest CFL tasks, the NS-RNN's performance becomes slightly worse on average than the other stack RNNs.

\begin{figure*}
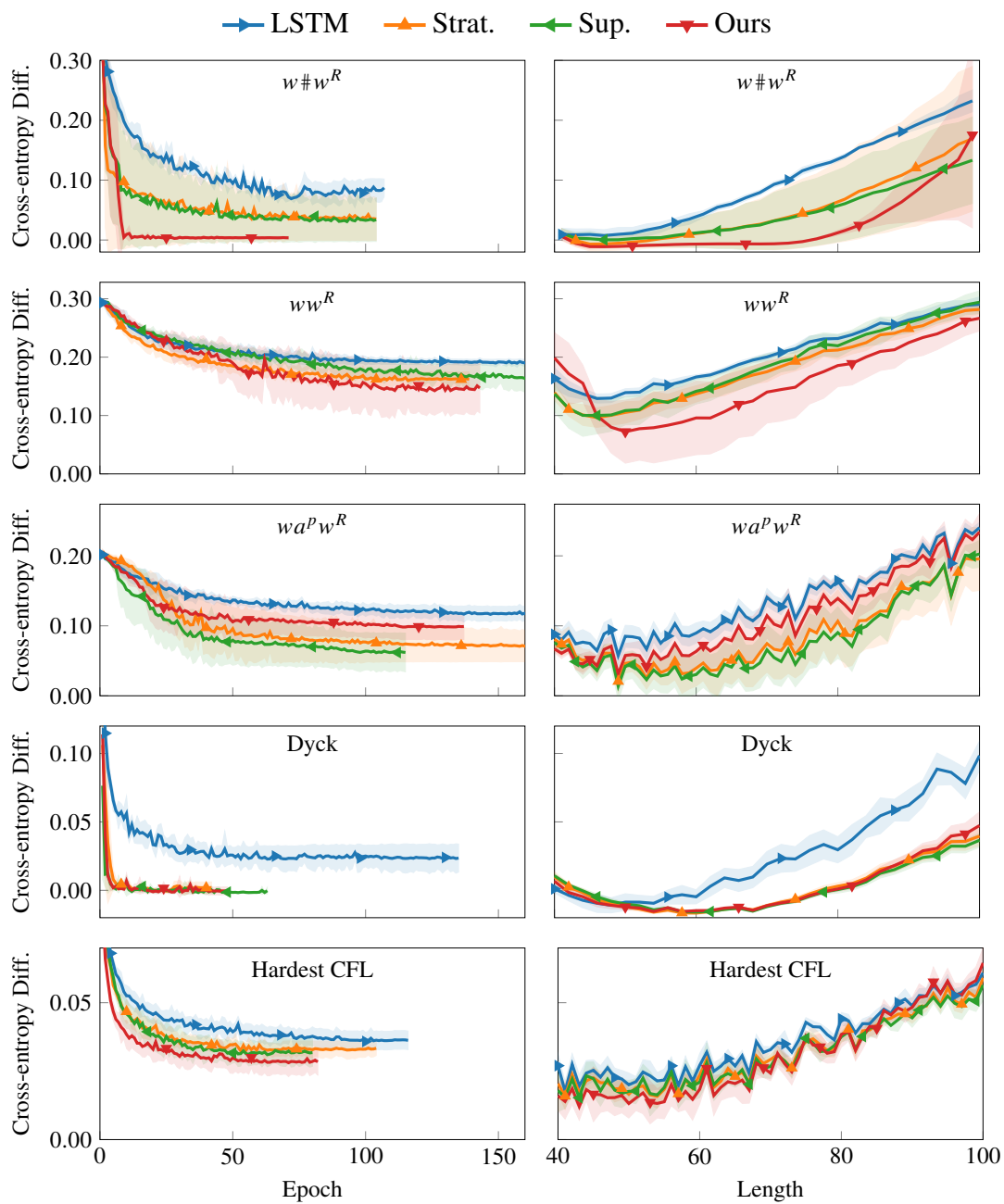

    \definecolor{color0}{rgb}{0.12156862745098,0.466666666666667,0.705882352941177}
    \definecolor{color1}{rgb}{1,0.498039215686275,0.0549019607843137}
    \definecolor{color2}{rgb}{0.172549019607843,0.627450980392157,0.172549019607843}
    \definecolor{color3}{rgb}{0.83921568627451,0.152941176470588,0.156862745098039}
    \pgfplotsset{
        line0/.style={cedline, color0, mark=triangle*, mark options={rotate=30}},
        line1/.style={cedline, color1, mark=triangle*, mark options={rotate=120}},
        line2/.style={cedline, color2, mark=triangle*, mark options={rotate=210}},
        line3/.style={cedline, color3, mark=triangle*, mark options={rotate=300}},
        every axis/.style={
            width=3.625in,
            height=1.9775in
        },
        title style={yshift=-4.7ex},
        y tick label style={
            /pgf/number format/.cd,
            fixed,
            fixed zerofill
        },
        scaled y ticks=false,
        tick pos=left
    }
    \tikzset{
        bars/.style={opacity=0.12},
        linelabel/.style={black,inner sep=2pt,font={\footnotesize}}
    }
    \centering
    \newcommand{\lastrow}[2]{
        \pgfplotsset{
            every axis/.append style={
                xmin=0,
                xmax=160,
                xtick distance=50,
                mark repeat=32,
            },
            line0/.append style={mark phase=0},
            line1/.append style={mark phase=8},
            line2/.append style={mark phase=16},
            line3/.append style={mark phase=24}
        }
        #2
        \scalebox{0.8}{\input{figures/01-ns-rnn-cfls/train/#1.tex}}
        &
        \pgfplotsset{
            every axis/.append style={
                xmin=40,
                xmax=100,
                xtick={40,60,80,100},
                mark repeat=8,
            }
        }
        \pgfplotsset{line0/.append style={mark phase=0}}
        \pgfplotsset{line1/.append style={mark phase=2}}
        \pgfplotsset{line2/.append style={mark phase=4}}
        \pgfplotsset{line3/.append style={mark phase=6}}
        #2
        \pgfplotsset{every axis/.append style={yticklabels={,,}}}
        \scalebox{0.8}{\input{figures/01-ns-rnn-cfls/test/#1.tex}} \\
    }
    \newcommand{\row}[1]{
        \pgfplotsset{
            every axis/.append style={
                xmin=0,
                xmax=160,
                xtick distance=50,
                mark repeat=32,
            },
            line0/.append style={mark phase=0},
            line1/.append style={mark phase=8},
            line2/.append style={mark phase=16},
            line3/.append style={mark phase=24}
        }
        \pgfplotsset{every axis/.append style={xticklabels={,,}}}
        \scalebox{0.8}{\input{figures/01-ns-rnn-cfls/train/#1.tex}}
        &
        \pgfplotsset{
            every axis/.append style={
                xmin=40,
                xmax=100,
                xtick={40,60,80,100},
                mark repeat=8,
            }
        }
        \pgfplotsset{line0/.append style={mark phase=0}}
        \pgfplotsset{line1/.append style={mark phase=2}}
        \pgfplotsset{line2/.append style={mark phase=4}}
        \pgfplotsset{line3/.append style={mark phase=6}}
        \pgfplotsset{every axis/.append style={xticklabels={,,}}}
        \pgfplotsset{every axis/.append style={yticklabels={,,}}}
        \scalebox{0.8}{\input{figures/01-ns-rnn-cfls/test/#1.tex}} \\
    }
    \begin{tabular}{@{}l@{\hspace{-0.25in}}l@{}}
        \multicolumn{2}{c}{
        \input{figures/01-ns-rnn-cfls/legend}
        } \\[-1ex]
        \row{marked-reversal}
        \row{unmarked-reversal}
        \row{padded-reversal}
        \row{dyck}
        
        \pgfplotsset{
            every axis/.append style={
                xmin=0,
                xmax=160,
                xtick distance=50,
                mark repeat=32,
            },
            line0/.append style={mark phase=0},
            line1/.append style={mark phase=8},
            line2/.append style={mark phase=16},
            line3/.append style={mark phase=24}
        }
        
        \scalebox{0.8}{\input{figures/01-ns-rnn-cfls/train/hardest-cfl.tex}}
        &
        \pgfplotsset{
            every axis/.append style={
                xmin=40,
                xmax=100,
                xtick={40,60,80,100},
                mark repeat=8,
            }
        }
        \pgfplotsset{line0/.append style={mark phase=0}}
        \pgfplotsset{line1/.append style={mark phase=2}}
        \pgfplotsset{line2/.append style={mark phase=4}}
        \pgfplotsset{line3/.append style={mark phase=6}}
        
        \pgfplotsset{every axis/.append style={yticklabels={,,}}}
        \scalebox{0.8}{\input{figures/01-ns-rnn-cfls/test/hardest-cfl.tex}} \\
    
    \end{tabular}
    \caption[Results of training stack RNNs on CFLs.]{Left: Cross-entropy difference between model and source distribution on validation set, as a function of training time. Lines are averages of five random restarts, and shaded regions show one standard deviation. After a random restart converges, the value of its last epoch is used in the average for later epochs. Right: Cross-entropy difference on the test set, binned by string length. Some models achieve a negative difference, for reasons explained in \cref{sec:ns-rnn-cfl-evaluation}. Each line is the average of the same five random restarts shown to the left.}
    \label{fig:learning-cfls-results}
\end{figure*}

Why does the NS-RNN not always generalize well to longer strings? First, note that it is impossible to sample strings longer than 80 symbols from $p_L$, so if a model were to learn $p_L$ perfectly, it would assign symbols infinitely high cross-entropy after the 80th symbol. Longer strings, then, do not test a model's ability to fit the training data, but rather reveal their inductive biases for learning the algorithm behind the process that generated the training data. When training a model with a stack or PDA on a finite set of strings sampled from a CFL, one would expect the model to learn a single algorithm that uses the stack or PDA, rather than a more complicated solution that learns length-dependent rules.

One possible (but unverified) explanation for cases in which the NS-RNN does not generalize well has to do with the shape of $p_L$. The true data distribution $p_L$ is the product of a uniform distribution over lengths and a length-based marginal distribution over a PCFG. The NS-RNN can use its WPDA to learn the PCFG's distribution, but the WPDA does not lend itself to modeling the length-based marginalization. In order to fit the training data better, the NS-RNN may compensate by using the LSTM controller to keep track of the current length of the input string and readjust the WPDA weights accordingly. The length-tracking may be done in such a way that it does not generalize well to unseen lengths. We do see evidence of length-dependent behavior in \cref{sec:rns-rnn-evolution-of-stack-actions}. Note also that the inherent difficulty of a task\dash{}that is, the cross-entropy of $p_L$\dash{}can vary with length, which may play a role in these trends (although for \MarkedReversal{} and \UnmarkedReversal{}, the per-symbol cross-entropy is almost constant with respect to length).

The interaction between the LSTM controller and the differentiable stack also points to a broader issue that affects training. For instance, on \MarkedReversal{}, Strat.\ and Sup.\ clearly both have the expressive power to fit the training data perfectly. However, they have worse performance than the NS-RNN on average. In many experiments with stack RNNs, we noticed that the LSTM controller becomes an overachiever during training, competing with the stack rather than using it, when only minimal interaction with the stack is sufficient for optimal performance. As a result, the model gets stuck in a local minimum that ignores the stack and fails to outperform a baseline LSTM. Stack RNNs seem to be especially prone to this when the learning rate is not tuned well. We have heard similar anecdotes from other researchers.

\section{Conclusion}

We showed empirically that the NS-RNN offers improved trainability and expressivity over prior stack RNNs on a number of tasks. Despite taking longer in wall-clock time to train, our model often learns to solve the task better and with a higher rate of success. The NS-RNN learns \MarkedReversal{} much more effectively than other stack RNNs, and it achieves the best results on a challenging nondeterministic context-free language, the hardest CFL. However, we note that the NS-RNN struggles on a task where matching symbols are distant, and sometimes does not generalize to longer lengths as well as other stack RNNs. We address these shortcomings in \cref{chap:rns-rnn}.

%% file: chapters/06-rns-rnn.tex
\chapter{The Renormalizing Nondeterministic Stack RNN}
\label{chap:rns-rnn}

In this chapter we present a new model, the \term{Renormalizing NS-RNN (RNS-RNN)}, which is based on the NS-RNN, but improves its performance on all of the CFL tasks from \cref{chap:learning-cfls}, thanks to two key changes. First, the RNS-RNN now assigns unnormalized positive weights instead of probabilities to WPDA transitions. The transitions therefore define an unnormalized distribution over stacks that is \term{renormalized} whenever the controller queries the stack reading. We provide an analysis of why unnormalized weights improve training. Second, the RNS-RNN can directly observe the state of the underlying WPDA via the stack reading. The RNS-RNN achieves lower cross-entropy than the NS-RNN, stratification stack RNN, and superposition stack RNN on all five CFLs from \cref{chap:learning-cfls} (within 0.05 nats of the information-theoretic lower bound), including \PaddedReversal{}, on which the NS-RNN previously failed to outperform the superposition stack RNN.

This work appeared in a paper published at ICLR 2022 \citep{dusell-chiang-2022-learning}. The code used for the experiments in this chapter is publicly available.\footnote{\url{https://github.com/bdusell/nondeterministic-stack-rnn/tree/iclr2022}}

\section{Model Definition}

Here, we introduce the Renormalizing NS-RNN (RNS-RNN), which differs from the NS-RNN in two ways.

\subsection{Unnormalized Transition Weights}

In \cref{sec:ns-rnn-cfl-results}, why did the NS-RNN fail to learn \PaddedReversal{} when it outperformed the baselines on all other CFLs? One likely explanation is the requirement to learn long-distance dependencies. To see why, consider again \cref{sec:ns-rnn-nondeterminism-helps-learning}, in which we discussed how the NS-RNN learns to orchestrate stack actions for strings in \MarkedReversal{}. The probability of a run in the NS-RNN's differentiable WPDA is the product of individual action probabilities, which are always strictly less than one. If a correct prediction depends on orchestrating many stack actions, then this probability may become very small. This may cause the model to learn slowly, as the amount of reward a good run receives during backpropagation is proportional to its probability. In the case of \MarkedReversal{}, we expect the model to begin by learning to predict the middle of the string, where only a few stack actions must be orchestrated, then working its way outwards, more and more slowly as more and more actions must be orchestrated. In \cref{sec:rns-rnn-evolution-of-stack-actions}, we verify empirically that this is the case.

The solution for this problem we propose is to use unnormalized non-negative weights in $\nstranst{t}$, not probabilities, and to normalize weights only when reading. \Cref{eq:ns-rnn-action-weights} now becomes
\begin{equation*}
    \stackactionsfunc{\rnnhiddent{t}} = \nstranst{t} = \exp(\affine{a}{\rnnhiddent{t}}).
\end{equation*}
The gradient flowing to a transition is still proportional to its posterior probability, but now each transition weight has the ability to ``amplify'' \citep{lafferty-etal-2001-conditional} other transitions in shared runs. The equation for the stack reading is not changed (yet), but its interpretation is. The NS-RNN maintains a probability distribution over stacks and updates it by performing probabilistic operations. Now, the model maintains an unnormalized weight distribution, and when it reads from the stack at each timestep, it renormalizes this distribution and marginalizes it to get a probability distribution over readings. For this reason, we call our new model a Renormalizing NS-RNN.

Note that computing $\nstranst{t}$, $\nsinnerweightletter$, and $\nsforwardweightletter$ in the log semiring, as discussed in \cref{sec:ns-rnn-implementation-details}, is essential for numerical stability when weights are unnormalized, as the weights of runs are liable to overflow in the real semiring. Otherwise, \cref{eq:ns-rnn-gamma,eq:ns-rnn-alpha-init,eq:ns-rnn-alpha-recurrence} are unchanged.

\subsection{PDA States Included in Stack Reading}
\label{sec:pda_states}

In the NS-RNN, the controller can read the distribution over the PDA's current top stack symbol, but it cannot observe its current state. To see why this is a problem, consider the language \UnmarkedReversal{}. While reading $w$, the controller should predict the uniform distribution, but while reading $\reverse{w}$, it should predict based on the top stack symbol. A PDA with two states can nondeterministically guess whether the current position is in $w$ or $\reverse{w}$. The controller should interpolate the two distributions based on the weight of being in each state, but it cannot do this without input from the differentiable WPDA, since the state is entangled with the stack contents. We solve this in the RNS-RNN by computing a joint distribution over top stack symbols \emph{and} PDA states, making $\stackreadingt{t}$ a vector of size $|Q| \cdot |\Gamma|$. \Cref{eq:ns-rnn-reading-definition} becomes
\begin{equation}
    \nsstackreading{t}{r}{y} = \frac{
        \sum_{\pdarunendsin{\pdarunletter}{t}{r}{y}} \wpdarunweight{\pdarunletter}
    }{
        \sum_{r' \in Q} \sum_{y' \in \Gamma} \sum_{\pdarunendsin{\pi}{t}{r'}{y'}} \wpdarunweight{\pi}
    }.
    \label{eq:rns-rnn-reading-definition}
\end{equation}
We compute this from $\nsforwardweightletter$ with
\begin{equation}
    \nsstackreading{t}{r}{y} = \frac{ \nsforwardweight{t}{r}{y} }{ \sum_{r', y'} \nsforwardweight{t}{r'}{y'} }.
\end{equation}

\section{Experiments on CFLs}
\label{sec:rns-rnn-cfl-experiments}

In order to assess the benefits of using unnormalized transition weights and including PDA states in the stack reading, we run the RNS-RNN with and without the two proposed modifications on the same five CFL language modeling tasks used in \cref{chap:learning-cfls}. We use the same experimental setup and PCFG settings, except for one important difference: we require the model to predict $\eos$ at the end of every string. This way, the model defines a proper probability distribution over strings, improving the interpretability of the results.

\subsection{Evaluation}
\label{sec:rns-rnn-cfl-evaluation}

We evaluate models using cross-entropy difference, slightly changing our definition from \cref{sec:ns-rnn-cfl-evaluation} to account for $\eos$. Let the per-symbol cross-entropy of a probability distribution $\probletter$ on a finite set of strings $S$, measured in nats, now be defined as
\begin{equation*}
    \crossentropyofdistonset{\probletter}{S} = \frac{-\sum_{w \in S} \log \probletter(w)}{\sum_{w \in S} (|w|+1)}.
\end{equation*}
The $+1$ in the denominator accounts for the fact that each model must predict $\eos$. Since each model $M$ defines a probability distribution $\modelprobname{M}$, and since we know the exact distribution from which the data is sampled, we can evaluate model $M$ by measuring the \term{cross-entropy difference} between the learned and true distributions, or $\crossentropyofdistonset{\modelprobname{M}}{S} - \crossentropyofdistonset{\sampleprobname{L}}{S}$. Lower is better, and 0 is optimal.

\subsection{Models and Training}
\label{sec:rns-rnn-cfl-models-and-training}

We compare seven architectures on the CFL tasks.
\begin{description}
    \item[LSTM] A baseline LSTM equivalent to the controller used in the stack RNNs.
    \item[Strat.] Stratification stack RNN.
    \item[Sup.] Superposition stack RNN.
    \item[NS] NS-RNN.
    \item[NS+S] NS-RNN with PDA states in the stack reading and normalized action weights.
    \item[NS+U] NS-RNN with no states in the stack reading and unnormalized action weights.
    \item[NS+S+U] RNS-RNN, i.e.\ an NS-RNN with PDA states in the stack reading and unnormalized action weights.
\end{description}

All stack RNNs use an LSTM controller. In all cases, the LSTM has a single layer with 20 hidden units. We grid-search the initial learning rate from $\{0.01, 0.005, 0.001, 0.0005\}$. For Strat.\ and Sup., we search for stack vector element sizes in $\{2, 20, 40\}$ (the pushed vector in Sup.\ is learned). For the NS models, we manually choose a small number of PDA states and stack symbol types based on how we would expect a PDA to solve the task. For \MarkedReversal{}, \UnmarkedReversal{}, and Dyck, we use $|Q| = 2$ and $|\Gamma| = 3$; and for \PaddedReversal{} and hardest CFL, we use $|Q| = 3$ and $|\Gamma| = 3$. For each hyperparameter setting searched, we run five random restarts. For each architecture, when reporting test performance, we select the model with the lowest cross-entropy difference on the validation set. We use the same initialization and optimization settings as in \cref{chap:learning-cfls} and train for a maximum of 200 epochs.

\subsection{Results}
\label{sec:rns-rnn-cfl-results}

We show validation performance as a function of training time and test performance binned by string length in \cref{fig:rns-rnn-cfl-results}. For all tasks, we see that our RNS-RNN (denoted NS+S+U) attains near-optimal cross-entropy (within 0.05 nats) on the validation set. All stack models effectively solve the deterministic \MarkedReversal{} and Dyck tasks, although we note that on \MarkedReversal{} the NS models do not generalize well on held-out lengths. Our new model excels on the three nondeterministic tasks: \UnmarkedReversal{}, \PaddedReversal{}, and hardest CFL. We find that the combination of both enhancements (+S+U) greatly improves performance on \UnmarkedReversal{} and hardest CFL over previous work. For \UnmarkedReversal{}, merely changing the task by adding $\eos$ causes the baseline NS model to perform worse than Strat.\ and Sup.; this may be because it requires the NS-RNN to learn a correlation between the two most distant timesteps. Both enhancements (+S+U) in the RNS-RNN are essential here; without unnormalized weights, the model does not find a good solution during training, and without PDA states, the model does not have enough information to make optimal decisions. For \PaddedReversal{}, we see that the addition of PDA states in the stack reading (+S) proves essential to improving performance. Although NS+S and NS+S+U have comparable performance on \PaddedReversal{}, NS+S+U converges much faster. On hardest CFL, using unnormalized weights by itself (+U) improves performance, but only both modifications together (+S+U) achieve the best performance.

\begin{figure*}
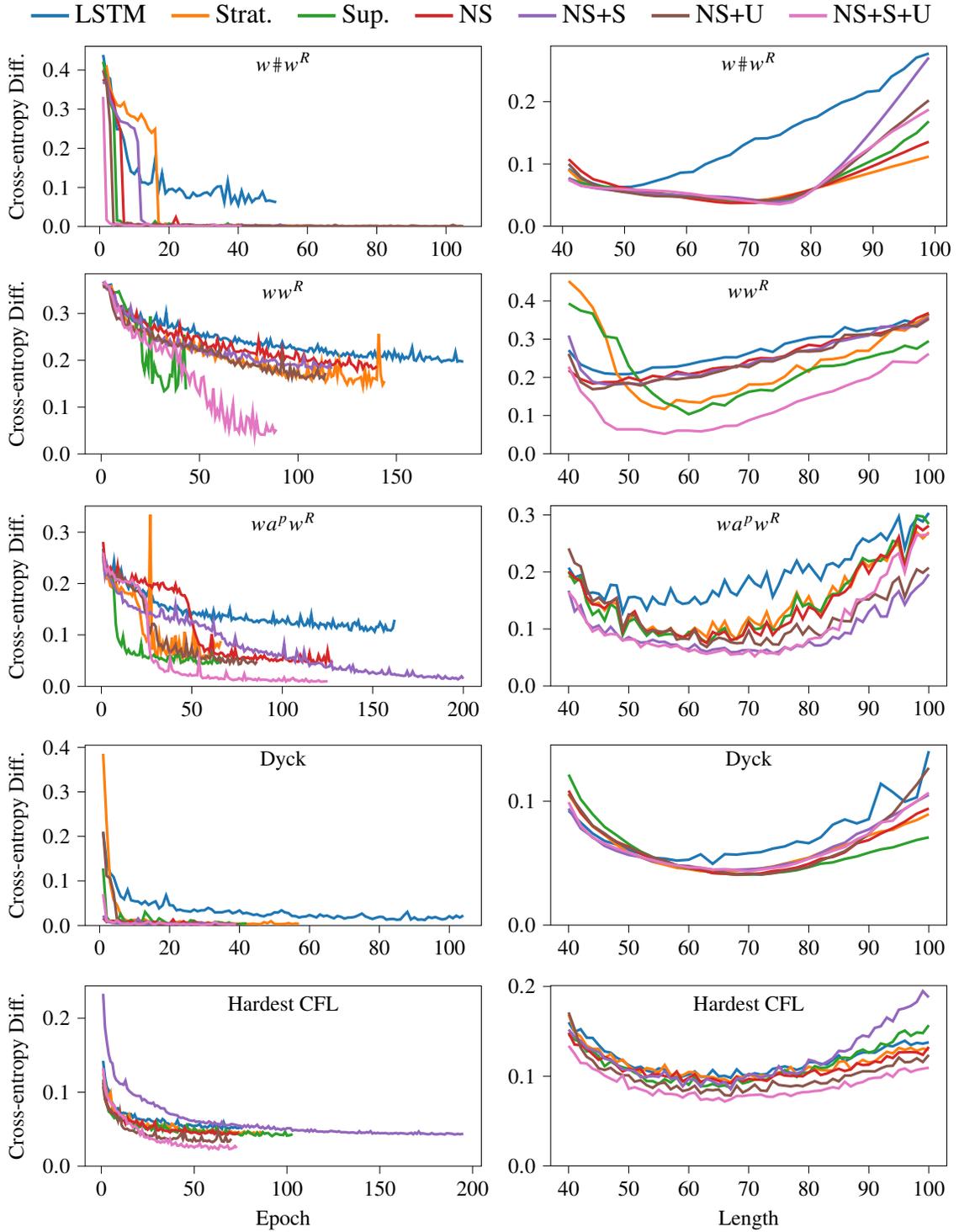

    \pgfplotsset{
        every axis/.style={
            height=2.025in,
            width=3.704in,
            ytick distance=0.1
        },
        title style={yshift=-4.5ex},
        y tick label style={
            /pgf/number format/.cd,
            precision=1,
            fixed,
            fixed zerofill
        }
    }
    \centering
    \begin{tabular}{@{}l@{\hspace{0in}}l@{}}
        \multicolumn{2}{c}{\input{figures/02-rns-rnn-cfls/cfl/legend}} \\
        \scalebox{0.8}{\input{figures/02-rns-rnn-cfls/cfl/train/marked-reversal}}
        &\scalebox{0.8}{\input{figures/02-rns-rnn-cfls/cfl/test/marked-reversal}}
        \\ 
        \scalebox{0.8}{\input{figures/02-rns-rnn-cfls/cfl/train/unmarked-reversal}}
        &\scalebox{0.8}{\input{figures/02-rns-rnn-cfls/cfl/test/unmarked-reversal}}
        \\ 
        \scalebox{0.8}{\input{figures/02-rns-rnn-cfls/cfl/train/padded-reversal}}
        &\scalebox{0.8}{\input{figures/02-rns-rnn-cfls/cfl/test/padded-reversal}}
        \\ 
        \scalebox{0.8}{\input{figures/02-rns-rnn-cfls/cfl/train/dyck}}
        &\scalebox{0.8}{\input{figures/02-rns-rnn-cfls/cfl/test/dyck}}
        \\ 
        \scalebox{0.8}{\input{figures/02-rns-rnn-cfls/cfl/train/hardest-cfl}} 
        &\scalebox{0.8}{\input{figures/02-rns-rnn-cfls/cfl/test/hardest-cfl}}
    \end{tabular}
    \caption[Results of training stack RNNs on CFLs, including the RNS-RNN.]{Left: Cross-entropy difference between model and source distribution on validation set vs.\ training time. Each line corresponds to the model which attains the lowest cross-entropy difference out of all random restarts. Right: Cross-entropy difference on the test set, binned by string length. These models are the same as those shown to the left.}
    \label{fig:rns-rnn-cfl-results}
\end{figure*}

\section{Wall-Clock Training Time}
\label{sec:rns-rnn-wall-clock-time}

We report wall-clock execution time for each model on the \MarkedReversal{} task in \cref{tab:rns-rnn-wall-clock-time}. We run the LSTM, Strat., and Sup.\ models in CPU mode, as this is faster than running on GPU due to the small model size. We run experiments for the NS models in GPU mode on a pool of the following NVIDIA GPU models, automatically selected based on availability: GeForce GTX TITAN X, TITAN X (Pascal), and GeForce GTX 1080 Ti.

\begin{table}
    \caption{Wall-clock execution time on \MarkedReversal{}}
    \label{tab:rns-rnn-wall-clock-time}
    \begin{center}
        \begin{tabular}{@{}lc@{}}
            \toprule
            Model & Time per epoch (s) \\
            \midrule
            LSTM & 51 \\
            Strat. & 801 \\
            Sup. & 169 \\
            NS & 1022 \\
            NS+S & 980 \\
            NS+U & 960 \\
            NS+S+U & 1060 \\
            \bottomrule
        \end{tabular}
    \end{center}
    \generalnote{Times are measured in seconds per epoch of training (averaged over all epochs). The speed of the NS models is roughly the same; there is some variation here due to differences in training data and GPU model.}
\end{table}

\section{Evolution of Stack Actions}
\label{sec:rns-rnn-evolution-of-stack-actions}

In \cref{fig:rns-rnn-evolution-of-stack-actions}, we show the evolution of stack actions for the NS+S (normalized) and NS+S+U (unnormalized) models over training time on the simplest of the CFL tasks: \MarkedReversal{}. We see that the normalized model begins solving the task by learning to push and pop symbols close to the middle marker. It then gradually learns to push and pop matching pairs of symbols further and further away from the middle marker. On the other hand, the unnormalized model learns the correct actions for all timesteps almost immediately. Oddly, the normalized model never learns to push the first 5 symbols with the same confidence as the other symbols in the first half; the controller may be memorizing those symbols, keeping track of string position to determine when to switch to using the stack.

\begin{figure*}
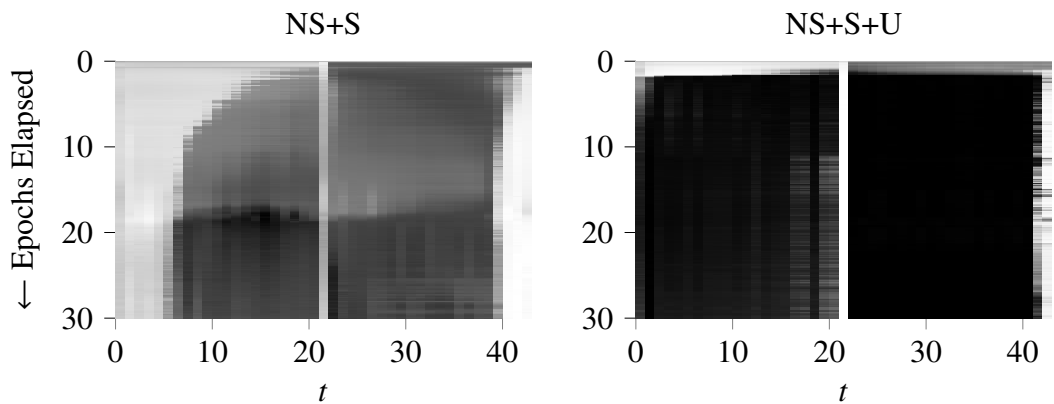

\pgfplotsset{every axis/.style={width=2.8in}}
    \centering
    \begin{tabular}{@{}cc@{}}
    \input{figures/02-rns-rnn-cfls/heatmaps/norm} &
    \input{figures/02-rns-rnn-cfls/heatmaps/unnorm}
    \end{tabular}
    \caption[Evolution of stack actions during training on the \MarkedReversal{} task.]{Visualization of the first 30 epochs of training (top to bottom) on a single string sampled from the \MarkedReversal{} task. In each plot, the horizontal axis is the string position (timestep). Darkness indicates the weight assigned to the correct stack action type, normalized by the weight of all actions at time~$t$ (black = correct, white = incorrect). The white band in the middle occurs because ``replace'' is considered the correct action type for the middle timestep, but the models apparently learn to perform a different action without affecting the results. Both models are trained with learning rate 0.005.}
    \label{fig:rns-rnn-evolution-of-stack-actions}
\end{figure*}

\section{Conclusion}

The Renormalizing NS-RNN (RNS-RNN) builds upon the strengths of the NS-RNN by letting stack action weights remain unnormalized and providing information about PDA states to the controller. We showed that both of these changes substantially improve learning, allowing the RNS-RNN to surpass other stack RNNs on a range of CFL modeling tasks.

%% file: chapters/07-incremental-execution.tex
\chapter{Incremental Execution on Natural Language}
\label{chap:incremental-execution}

In this chapter, we develop \term{limited-memory} versions of the NS-RNN and RNS-RNN that require only $\bigo{n}$ time and space, rather than $\bigo{n^3}$ time and $\bigo{n^2}$ space. The limited-memory (R)NS-RNN can be run incrementally on arbitrarily long sequences, which is a necessity for language modeling on natural language. We provide experimental results on the Penn Treebank language modeling benchmark \citep{marcus-etal-1994-penn,mikolov-etal-2011-empirical} and analyze the models' performance on a comprehensive syntactic generalization benchmark \citep{hu-etal-2020-systematic,gauthier-etal-2020-syntaxgym}.

This work appeared in a paper published at ICLR 2022 \citep{dusell-chiang-2022-learning}. The code used for the experiments in this chapter is publicly available.\footnote{\url{https://github.com/bdusell/nondeterministic-stack-rnn/tree/iclr2022}}

\section{Motivation}

For standard language modeling benchmarks, during both training and evaluation, RNN language models customarily process the entire dataset in order as if it were one long sequence, since being able to retain contextual knowledge of past sentences significantly improves predictions for future sentences. Running a full forward and backward pass during training on such a long sequence would be infeasible, so the dataset is processed incrementally using a technique called truncated backpropagation through time (BPTT). This technique is feasible for models whose time and space complexity is linear with respect to sequence length, but for memory-augmented models such as stack RNNs, something must be done to limit the time and storage requirements. \Citet{yogatama-etal-2018-memory} did this for the superposition stack by limiting the stack to 10 elements. In this section, we propose a technique for limiting the space and time requirements of the NS-RNN and RNS-RNN, allowing us to use truncated BPTT and retain contextual information.

\section{Limited-Memory NS-RNN}
\label{sec:limited-ns-rnn}

To limit the space and time requirements of the NS-RNN or RNS-RNN, we introduce the constraint that the stack WFA can only contain transitions $\nsinnerweight{i}{q}{x}{t}{r}{y}$ where $t - i$ does not exceed a hyperparameter $D$; all other transitions are treated as having zero weight. It may be easier to get an intuition for this constraint in terms of CFGs. The stack WFA formulation, which is based on Lang's algorithm, can be thought of as converting a PDA to a CFG \citep{hopcroft-ullman-1979-introduction,sipser-2013-introduction} and then parsing with a CKY-style algorithm. The equation for $\nsinnerweightletter$ (\cref{eq:ns-rnn-gamma}) has three terms, corresponding to rules of the form $A \rightarrow b$ (push), $A \rightarrow Bc$ (replace), and $A \rightarrow BCd$ (pop). The constraint $t-i \leq D$ on $\nsinnerweightletter$ means that these rules can only be used when they span at most $D$ positions. Note that this effectively bounds the depth of the stack, since at most one symbol can be pushed per timestep.

The equations for $\nsforwardweightletter$ (\cref{eq:ns-rnn-alpha-init,eq:ns-rnn-alpha-recurrence}) have two cases, which correspond to rules of the form $A \rightarrow \emptystring$ and $A \rightarrow AB$. The definition of $\nsforwardweightletter$ allows these rules to be used for spans starting at 0 and ending anywhere. This is essentially equivalent to the constraint used in the Hiero machine translation system~\citep{chiang-2005-hierarchical}, which uses synchronous CFGs under the constraint that no nonterminal spans more than 10 symbols, with the exception of so-called glue rules $S \rightarrow X$ and $S \rightarrow SX$.

\begin{figure*}
    \def\D{3}
    \def\T{4}
    \def\numchunks{3}
    \edef\Dminone{\number\numexpr\D-1\relax}
    \centering
    \begin{tikzpicture}
        \pgfmathsetmacro{\size}{0.7}
        \pgfmathsetmacro{\gap}{0.1}
        \pgfmathsetmacro{\xmax}{\T*\numchunks-1}
        \pgfmathsetmacro{\alphay}{\size*-(\T*\numchunks+1)}
        \pgfmathsetmacro{\alphaoffset}{0.1}
        \pgfmathsetmacro{\arrowlength}{2}
        \foreach \x in {0,...,\xmax} {
            \pgfmathsetmacro{\ymin}{\x-\D+1<0 ? 0 : \x-\D+1}
            \foreach \y in {\ymin,...,\x} {
                \draw (\size*\x+\gap,-\size*\y-\gap) rectangle ++(\size-2*\gap,-\size+2*\gap);
            }
        }
        \pgfmathsetmacro{\imax}{\numchunks-1}
        \foreach \i in {0,...,\imax} {
            \pgfmathsetmacro{\chunkno}{int(\i+1)}
            \draw[dashed] ({\size*(-(\D-1)+\i*\T)},{\size*(\D-1-\i*\T)}) rectangle ++({\size*(\T+\D-1)},{-(\size*(\T+\D-1))});
            \node[anchor=north east] at ({\size*((\i+1)*\T)},{\size*(\D-1-\i*\T)}) {\footnotesize Batch \chunkno};
            \draw[dashed] ({\size*(-(\D-1)+\i*\T)},{\alphay+mod(\i,2)*\alphaoffset}) rectangle ++({\size*(\T+\D-1)},{-\size-mod(\i,2)*2*\alphaoffset});
        }
        \pgfmathsetmacro{\imax}{\T*\numchunks}
        \foreach \i in {0,...,\imax} {
            \node[anchor=south] at ({-\size/2+\size*\i},0) {\footnotesize \i};
        }
        \pgfmathsetmacro{\imax}{\T*\numchunks-1}
        \foreach \i in {0,...,\imax} {
            \node[anchor=east] at ({-\size},{\size*-(\i+0.5)}) {\footnotesize \i};
        }
        \node[anchor=south] at (\T*\numchunks*\size/2,{\size}) {$t$};
        \node[anchor=east] at ({\size*-2},-\T*\numchunks*\size/2) {$i$};
        \draw[->] (\T*\numchunks*\size/2-\arrowlength/2, \size) -- ++(\arrowlength, 0);
        \draw[->] (-2*\size,-\T*\numchunks*\size/2+\arrowlength/2) -- ++(0,-\arrowlength);
        \node[anchor=east] at ({-3*\size},{-\T*\numchunks*\size/2}) {$\nsinnerweightletter$};
        \node[anchor=east] at ({-3*\size},{\alphay-0.5*\size}) {$\nsforwardweightletter$};
        \pgfmathsetmacro{\imax}{\T*\numchunks}
        \foreach \i in {0,...,\imax} {
            \draw ({\size*(\i-1)+\gap},{\alphay-\gap}) rectangle ++(\size-2*\gap,-\size+2*\gap);
        }
    \end{tikzpicture}
    \caption[Visualization of incremental execution in the NS-RNN or RNS-RNN.]{Visualization of incremental execution in the NS-RNN or RNS-RNN when $D = \D$ and the length of batches in truncated BPTT is \T{}. Squares with solid edges represent non-zero entries of $\nsinnerweightletter$ or $\nsforwardweightletter$. Rectangles with dashed edges represent batches used in truncated BPTT. Each batch computes \T{} new columns of $\nsinnerweightletter$ and entries of $\nsforwardweightletter$, requiring $D-1 = \Dminone$ previous columns of $\nsinnerweightletter$ and entries of $\nsforwardweightletter$. Batch 1 is padded with zero entries back in time. Places where two dashed rectangles intersect represent slices of $\nsinnerweightletter$ and $\nsforwardweightletter$ that are forwarded from one batch to the next. The hidden state of the LSTM controller is also forwarded (not shown). Crucially, processing the sequence in \numchunks{} batches is equivalent to processing it as one larger batch.}
    \label{fig:incremental-execution}
\end{figure*}

If we consider the tensor $\nsinnerweightletter$, which contains the weights of the stack WFA, as a matrix with axes for the variables $i$ and $t$, then the only non-zero entries in $\nsinnerweightletter$ lie in a band of height $D$ along the diagonal. Crucially, column $t$ of $\nsinnerweightletter$ depends only on $\nsinnerweightit{i}{t'}$ for $t-D \leq i \leq t-2$ and $t-D+1 \leq t' \leq t-1$. Similarly, $\nsforwardweightt{t}$ depends only on $\nsforwardweightt{i}$ for $t-D \leq i \leq t-1$ and $\nsinnerweightit{i}{t}$ for $t-i \leq D$. So, just as truncated BPTT for an RNN involves freezing the hidden state and forwarding it to the next forward-backward pass, truncated BPTT for the (R)NS-RNN involves forwarding the hidden state of the controller \emph{and} forwarding a slice of $\nsinnerweightletter$ and $\nsforwardweightletter$. This reduces the time complexity of the (R)NS-RNN to $\bigo{{|Q|}^4 {|\Gamma|}^3 D^2 n}$ and its space complexity to $\bigo{{|Q|}^2 {|\Gamma|}^2 D n}$. \Cref{fig:incremental-execution} explains this visually.

\section{Experiments}
\label{sec:limited-ns-rnn-ptb-experiments}

Limiting the memory of the (R)NS-RNN now makes it feasible to run experiments on natural language modeling benchmarks, although the high computational cost of increasing $|Q|$ and $|\Gamma|$ still limits us to settings with little information bandwidth in the stack. We believe this will make it difficult for the (R)NS-RNN to store lexical information on the stack, but it might succeed in using $\Gamma$ as a small set of syntactic categories. To this end, we run exploratory experiments with the NS-RNN, RNS-RNN, and other language models on the Penn Treebank (PTB) \citep{marcus-etal-1994-penn} as preprocessed by \citet{mikolov-etal-2011-empirical}.

\subsection{Models and Training}
\label{sec:limited-ns-rnn-ptb-models-and-training}

We compare four types of model: LSTM, superposition (``Sup.'') with a maximum stack depth of 10, and memory-limited NS-RNNs (``NS'') and RNS-RNNs (``RNS'') with $D = 35$.

The hyperparameters for our baseline LSTM, initialization, and training schedule are based on the unregularized LSTM experiments of \citet{semeniuta-etal-2016-recurrent}. We train all models using simple stochastic gradient descent (SGD) as recommended by prior language modeling work \citep{merity-etal-2018-regularizing} and truncated BPTT with a sequence length of 35. For all models, we use a minibatch size of 32. We randomly initialize all parameters uniformly from the interval $[-0.05, 0.05]$. We divide the learning rate by 1.5 whenever the validation perplexity does not improve, and we stop training after 2 epochs of no improvement in validation perplexity.

We test two variants of Sup., pushing either the hidden state or a learned vector. Unless otherwise noted, the LSTM controller has 256 units, one layer, and no dropout. We include an LSTM with 258 units, which has more parameters than the largest Sup.\ model, and an LSTM with 267 units, which has more parameters than the largest (R)NS-RNN.

We use a word embedding layer of the same size as the hidden state. For each model, we randomly search \citep{bergstra-bengio-2012-random} for a good initial learning rate and gradient clipping threshold. Like \citet{yogatama-etal-2018-memory}, we report results for the model with the best validation perplexity out of 10 randomly searched models. The learning rate, which is divided by batch size and sequence length, is drawn from a log-uniform distribution over $[1, 100]$, and the gradient clipping threshold, which is multiplied by batch size and sequence length, is drawn from a log-uniform distribution over $[1 \times 10^{-5}, 1 \times 10^{-3}]$. We scale the learning rate and gradient clipping threshold this way because, under our implementation, sequence length and batch size can vary when the dataset is not evenly divisible by the prescribed values. Other language modeling work follows a different scaling convention for these two hyperparameters, typically scaling the learning rate by sequence length but not by batch size, and not rescaling the gradient clipping threshold. Under this convention the learning rate would be drawn from $[0.03125, 3.125]$ and the gradient clipping threshold from $[0.0112, 1.12]$.

\subsection{Evaluation}

In addition to perplexity, we also report the recently proposed Syntactic Generalization (SG) score metric \citep{hu-etal-2020-systematic,gauthier-etal-2020-syntaxgym}. This score, which ranges from 0 to 1, puts a language model through a battery of psycholinguistically-motivated tests that test how well a model generalizes to non-linear, nested syntactic patterns. It does this by measuring the \term{surprisal}, or log-probability, that the model assigns to certain words at carefully-selected points in a sentence. \Citet{hu-etal-2020-systematic} noted that perplexity does not, in general, agree with SG score, so we hypothesized the SG score would provide crucial insight into the stack's effectiveness.

\subsection{Results}

We show the results of our experiments on the Penn Treebank in \cref{tab:rns-ptb-results}.

\begin{longtable}{llccc}
    \caption{Language modeling results on the Penn Treebank \label{tab:rns-ptb-results}\/}\\
    \toprule
    Model & \# Params & Val & Test & SG Score \\
    \midrule
    \endfirsthead
    \caption[]{\textit{Continued}}\\
    \midrule
    Model & \# Params & Val & Test & SG Score \\
    \midrule
    \endhead
    \endfoot
    \bottomrule
    \endlastfoot
    LSTM, 256 units & 5,656,336 & 125.78 & 120.95 & 0.433 \\
    LSTM, 258 units & 5,704,576 & 122.08 & 118.20 & 0.420 \\
    LSTM, 267 units & 5,922,448 & 125.20 & 120.22 & 0.437 \\
    Sup.\ (push hidden state), 247 units & 5,684,828 & \textbf{121.24} & \textbf{115.35} & 0.387 \\
    Sup.\ (push learned), $|\pushedstackvectort{t}| = 22$ & 5,685,289 & 122.87 & 117.93 & 0.431 \\
    NS, $|Q| = 1$, $|\Gamma| = 2$ & 5,660,954 & 126.10 & 122.62 & 0.414 \\
    NS, $|Q| = 1$, $|\Gamma| = 3$ & 5,664,805 & 123.41 & 119.25 & 0.430 \\
    NS, $|Q| = 1$, $|\Gamma| = 4$ & 5,669,684 & 121.66 & 117.91 & 0.432 \\
    NS, $|Q| = 1$, $|\Gamma| = 5$ & 5,675,591 & 123.01 & 119.54 & 0.452 \\
    NS, $|Q| = 1$, $|\Gamma| = 6$ & 5,682,526 & 129.94 & 125.45 & 0.432 \\
    NS, $|Q| = 1$, $|\Gamma| = 7$ & 5,690,489 & 126.11 & 121.94 & 0.443 \\
    NS, $|Q| = 1$, $|\Gamma| = 11$ & 5,732,621 & 129.11 & 124.98 & 0.431 \\
    NS, $|Q| = 2$, $|\Gamma| = 2$ & 5,668,664 & 128.16 & 123.52 & 0.412 \\
    NS, $|Q| = 2$, $|\Gamma| = 3$ & 5,680,996 & 129.51 & 126.00 & \textbf{0.471} \\
    NS, $|Q| = 2$, $|\Gamma| = 4$ & 5,697,440 & 124.28 & 120.18 & 0.433 \\
    NS, $|Q| = 2$, $|\Gamma| = 5$ & 5,717,996 & 124.24 & 119.34 & 0.429 \\
    NS, $|Q| = 3$, $|\Gamma| = 2$ & 5,681,514 & 125.32 & 120.62 & 0.470 \\
    NS, $|Q| = 3$, $|\Gamma| = 3$ & 5,707,981 & 122.96 & 118.89 & 0.420 \\
    NS, $|Q| = 3$, $|\Gamma| = 4$ & 5,743,700 & 126.71 & 122.53 & 0.447 \\
    RNS, $|Q| = 1$, $|\Gamma| = 2$ & 5,660,954 & 122.64 & 117.56 & 0.435 \\
    RNS, $|Q| = 1$, $|\Gamma| = 3$ & 5,664,805 & 121.83 & 116.46 & 0.430 \\
    RNS, $|Q| = 1$, $|\Gamma| = 4$ & 5,669,684 & 127.99 & 123.06 & 0.437 \\
    RNS, $|Q| = 1$, $|\Gamma| = 5$ & 5,675,591 & 126.41 & 122.25 & 0.441 \\
    RNS, $|Q| = 1$, $|\Gamma| = 6$ & 5,682,526 & 122.57 & 117.79 & 0.416 \\
    RNS, $|Q| = 1$, $|\Gamma| = 7$ & 5,690,489 & 123.51 & 120.48 & 0.430 \\
    RNS, $|Q| = 1$, $|\Gamma| = 11$ & 5,732,621 & 127.21 & 121.84 & 0.386 \\
    RNS, $|Q| = 2$, $|\Gamma| = 2$ & 5,670,712 & 122.11 & 117.22 & 0.399 \\
    RNS, $|Q| = 2$, $|\Gamma| = 3$ & 5,684,068 & 131.46 & 127.57 & 0.463 \\
    RNS, $|Q| = 2$, $|\Gamma| = 4$ & 5,701,536 & 124.96 & 121.61 & 0.431 \\
    RNS, $|Q| = 2$, $|\Gamma| = 5$ & 5,723,116 & 122.92 & 117.87 & 0.423 \\
    RNS, $|Q| = 3$, $|\Gamma| = 2$ & 5,685,610 & 129.48 & 124.66 & 0.433 \\
    RNS, $|Q| = 3$, $|\Gamma| = 3$ & 5,714,125 & 127.57 & 123.00 & 0.434 \\
    RNS, $|Q| = 3$, $|\Gamma| = 4$ & 5,751,892 & 122.67 & 118.09 & 0.408 \\
\end{longtable}

We reproduce the finding of \citet{yogatama-etal-2018-memory} that Sup.\ can achieve lower perplexity than an LSTM with a comparable number of parameters, but this does not translate into a better SG score. We also show results for various sizes of NS and RNS. The setting $|Q| = 1$, $|\Gamma| = 2$ represents minimal capacity in the (R)NS-RNN models and is meant to serve as a baseline for the other settings. The other two settings are meant to test the upper limits of model capacity before computational cost becomes too great. The setting $|Q| = 1$, $|\Gamma| = 11$ represents the greatest number of stack symbol types we can afford to use, using only one PDA state. We selected the setting  $|Q| = 3$, $|\Gamma| = 4$ by increasing the number of PDA states, and then the number of stack symbol types, until computational cost became too great (recall that the time complexity is $\bigo{{|Q|}^4 {|\Gamma|}^3}$, so adding states is more expensive than adding stack symbol types).

The results for NS and RNS do not show a clear trend in perplexity or SG score as the number of states or stack symbols increases, or as the modifications in RNS are applied, even when breaking down SG score by type of syntactic test. We hypothesize that this is due to the information bottleneck caused by using a small discrete set of symbols $\Gamma$ in both models, a limitation we will address in \cref{chap:vrns-rnn}. This limitation in the (R)NS-RNN may explain why Sup.\ achieves lower perplexity; instead of discrete symbols, Sup.\ uses a stack of vectors, which can be used as word embeddings to encode lexical information. Another advantage that Sup.\ has is that only the depth of its stack is limited, whereas the (R)NS-RNN is limited in the number of timesteps for which it can retain information on the stack. Sup.\ can retain information on the stack indefinitely (likely with some decay). The effect of this advantage could be measured by training (R)NS and Sup.\ on a task with sequences no longer than $D$.

In \cref{tab:rns-ptb-circuits}, we show the same experiments with SG score broken down by syntactic ``circuit'' as defined by \citet{hu-etal-2020-systematic}, offering a more fine-grained look at the classes of errors the models make (Agr.\ = Agreement, Lic.\ = Licensing, GPE = Garden-Path Effects, GSE = Gross Syntactic Expectation, CE = Center Embedding, and LDD = Long-Distance Dependencies). Note that these tests are of quite different character; Garden-Path Effects and Center Embedding are arguably the most syntax-oriented. Note also that SG score tests how well a model's probability distribution corresponds to \emph{human} sentence processing difficulty, not just grammaticality. For example, in the case of Garden-Path Effects, a model is rewarded if it has \emph{more} difficulty on sentences with reduced relative clauses than those with explicit relative clauses, so improved nondeterministic processing may not contribute to a higher score.

\begin{longtable}{lcccccc}
    \caption{SG scores on the Penn Treebank broken down by circuit \label{tab:rns-ptb-circuits}\/}\\
    \toprule
    Model & Agr. & Lic. & GPE & GSE & CE & LDD \\
    \midrule
    \endfirsthead
    \caption[]{\textit{Continued}}\\
    \midrule
    Model & Agr. & Lic. & GPE & GSE & CE & LDD \\
    \midrule
    \endhead
    \endfoot
    \bottomrule
    \endlastfoot
    LSTM, 256 units & 0.667 & 0.446 & 0.330 & 0.397 & 0.482 & 0.414 \\
    LSTM, 258 units & 0.658 & 0.447 & 0.335 & 0.375 & 0.518 & 0.357 \\
    LSTM, 267 units & 0.667 & 0.497 & 0.343 & 0.446 & 0.411 & 0.350 \\
    Sup.\ (push hidden state) & 0.640 & 0.408 & 0.296 & 0.310 & 0.464 & 0.352 \\
    Sup.\ (push learned) & 0.684 & 0.439 & 0.340 & 0.408 & 0.482 & 0.395 \\
    NS, $|Q| = 1$, $|\Gamma| = 2$ & 0.588 & 0.452 & 0.298 & 0.391 & 0.339 & 0.418 \\
    NS, $|Q| = 1$, $|\Gamma| = 3$ & 0.623 & 0.467 & 0.400 & 0.413 & 0.393 & 0.354 \\
    NS, $|Q| = 1$, $|\Gamma| = 4$ & 0.640 & 0.497 & 0.331 & 0.375 & 0.571 & 0.340 \\
    NS, $|Q| = 1$, $|\Gamma| = 5$ & 0.605 & 0.514 & 0.394 & 0.413 & 0.589 & 0.344 \\
    NS, $|Q| = 1$, $|\Gamma| = 6$ & 0.632 & 0.424 & 0.408 & 0.391 & 0.464 & 0.399 \\
    NS, $|Q| = 1$, $|\Gamma| = 7$ & 0.719 & 0.470 & 0.351 & 0.473 & 0.500 & 0.344 \\
    NS, $|Q| = 1$, $|\Gamma| = 11$ & 0.640 & 0.432 & 0.329 & 0.424 & 0.500 & 0.413 \\
    NS, $|Q| = 2$, $|\Gamma| = 2$ & 0.702 & 0.388 & 0.329 & 0.446 & 0.446 & 0.371 \\
    NS, $|Q| = 2$, $|\Gamma| = 3$ & 0.658 & 0.527 & 0.367 & 0.446 & 0.518 & 0.411 \\
    NS, $|Q| = 2$, $|\Gamma| = 4$ & 0.632 & 0.464 & 0.345 & 0.386 & 0.518 & 0.387 \\
    NS, $|Q| = 2$, $|\Gamma| = 5$ & 0.711 & 0.464 & 0.307 & 0.413 & 0.518 & 0.355 \\
    NS, $|Q| = 3$, $|\Gamma| = 2$ & 0.711 & 0.528 & 0.349 & 0.435 & 0.518 & 0.406 \\
    NS, $|Q| = 3$, $|\Gamma| = 3$ & 0.746 & 0.439 & 0.316 & 0.375 & 0.411 & 0.376 \\
    NS, $|Q| = 3$, $|\Gamma| = 4$ & 0.702 & 0.450 & 0.364 & 0.484 & 0.536 & 0.369 \\
    RNS, $|Q| = 1$, $|\Gamma| = 2$ & 0.702 & 0.460 & 0.280 & 0.451 & 0.464 & 0.404 \\
    RNS, $|Q| = 1$, $|\Gamma| = 3$ & 0.649 & 0.427 & 0.438 & 0.418 & 0.446 & 0.347 \\
    RNS, $|Q| = 1$, $|\Gamma| = 4$ & 0.658 & 0.412 & 0.342 & 0.565 & 0.339 & 0.418 \\
    RNS, $|Q| = 1$, $|\Gamma| = 5$ & 0.728 & 0.449 & 0.370 & 0.429 & 0.482 & 0.371 \\
    RNS, $|Q| = 1$, $|\Gamma| = 6$ & 0.614 & 0.422 & 0.314 & 0.435 & 0.518 & 0.377 \\
    RNS, $|Q| = 1$, $|\Gamma| = 7$ & 0.649 & 0.460 & 0.374 & 0.337 & 0.411 & 0.404 \\
    RNS, $|Q| = 1$, $|\Gamma| = 11$ & 0.614 & 0.447 & 0.291 & 0.266 & 0.446 & 0.338 \\
    RNS, $|Q| = 2$, $|\Gamma| = 2$ & 0.649 & 0.417 & 0.365 & 0.375 & 0.339 & 0.334 \\
    RNS, $|Q| = 2$, $|\Gamma| = 3$ & 0.640 & 0.474 & 0.411 & 0.446 & 0.554 & 0.408 \\
    RNS, $|Q| = 2$, $|\Gamma| = 4$ & 0.658 & 0.469 & 0.336 & 0.326 & 0.500 & 0.403 \\
    RNS, $|Q| = 2$, $|\Gamma| = 5$ & 0.693 & 0.420 & 0.339 & 0.370 & 0.607 & 0.376 \\
    RNS, $|Q| = 3$, $|\Gamma| = 2$ & 0.579 & 0.435 & 0.295 & 0.440 & 0.554 & 0.445 \\
    RNS, $|Q| = 3$, $|\Gamma| = 3$ & 0.632 & 0.444 & 0.356 & 0.418 & 0.482 & 0.403 \\
    RNS, $|Q| = 3$, $|\Gamma| = 4$ & 0.588 & 0.427 & 0.342 & 0.353 & 0.482 & 0.373 \\
\end{longtable}

No single model performs best on all circuits. For all models, we find that SG scores are highly variable and uncorrelated to perplexity, corroborating findings by \citet{hu-etal-2020-systematic}. In fact, when we inspect all randomly searched LSTMs, we find that they are sometimes able to attain SG scores higher than any of those shown in \cref{tab:rns-ptb-results}. We show this in \cref{fig:sg-score-vs-perplexity}, where we plot SG score vs.\ test perplexity for all 10 random restarts of an LSTM and an RNS-RNN. We see that many of the models that were not selected actually have a much higher SG score (even above 0.48), suggesting that the standard validation perplexity criterion is a poor choice for syntactic generalization. The fact that SG score measures surprisal only at a specific part of a sentence, rather than the whole sentence, may explain some of the lack of correlation between perplexity and SG score. The fact that the training data in the Penn Treebank has a high degree of token masking likely exacerbates this issue; it only includes 10,000 token types.

\begin{figure*}
    \centering
    \pgfplotsset{
        y tick label style={
            /pgf/number format/.cd,
            precision=2,
            fixed,
            fixed zerofill
        }
    }
    \input{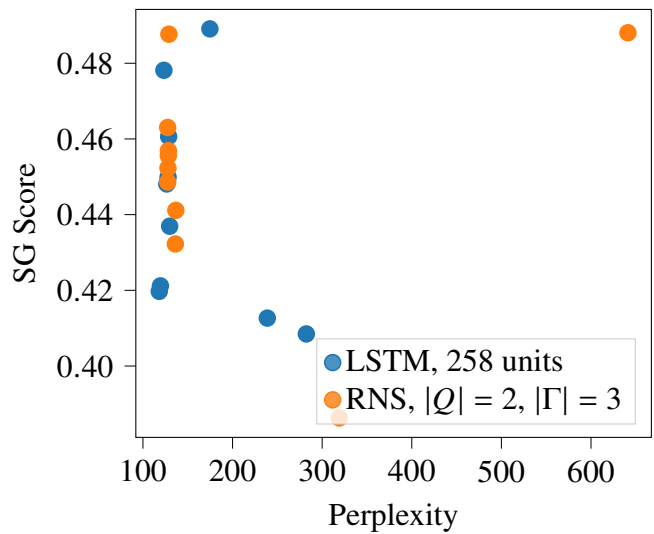}
    \caption[SG score vs.\ test perplexity, shown on all 10 random restarts for an LSTM and an RNS-RNN.]{SG score vs.\ test perplexity, shown on all 10 random restarts for an LSTM and an RNS-RNN. SG score is uncorrelated with perplexity, and models that narrowly miss out on having the best perplexity often have much higher SG scores.}
    \label{fig:sg-score-vs-perplexity}
\end{figure*}

From this we conclude that improving syntactic generalization on natural language remains elusive for all stack RNNs we tested, and that we may need to look beyond cross-entropy/perplexity as a training criterion. Indeed, \citet{hao-etal-2020-probabilistic} and \citet{kuribayashi-etal-2021-lower} noted that perplexity does not always correspond to human reading times.

\section{Conclusion}

Our memory-limited version of the RNS-RNN is a crucial modification towards practical use on natural language. We tested this model on the Penn Treebank, although we did not see performance improvements with the model sizes we were able to test, and in fact no stack RNNs excel in terms of syntactic generalization. We see that there is a discrepancy between the perplexity metric used as the validation criterion during training and syntactic generalization. Later, in \cref{chap:vrns-rnn}, we will demonstrate perplexity improvements with larger RNS-RNNs.

%% file: chapters/08-learning-non-cfls.tex
\chapter{Learning Non-Context-Free Languages}
\label{chap:learning-non-cfls}

Previously, we showed that the NS-RNN and RNS-RNN are highly effective at learning CFLs. In this chapter, we show that the controller and differentiable WPDA in the RNS-RNN interact to produce a surprising effect: the ability to learn many non-context-free languages (non-CFLs), such as \MarkedCopy{}. We explain how this is possible theoretically, and we show compelling results on several non-CFL tasks.

This work will appear in a paper published at ICLR 2023 \citep{dusell-chiang-2023-surprising}. The code used for the experiments in this chapter is publicly available.\footnote{\url{https://github.com/bdusell/nondeterministic-stack-rnn}}

\section{Motivation}

In \cref{chap:learning-cfls,chap:rns-rnn}, the RNS-RNN proved more effective than other stack RNNs at learning both nondeterministic and deterministic CFLs. We argued that, in practical terms, giving RNNs the ability to recognize context-free patterns may be beneficial for modeling natural language, as syntax exhibits hierarchical structure. Nondeterminism in particular is necessary for handling the very common phenomenon of syntactic ambiguity.

However, the RNS-RNN's reliance on a PDA may still render it inadequate for practical use, as not all phenomena in human language are context-free. For example, some languages, including Dutch \citep{bresnan-etal-1982-cross}, Swiss German \citep{shieber-1985-evidence}, and Bambara \citep{culy-1985-complexity} contain syntactic constructions of the form $w \cdots w$, known in linguistics as \term{cross-serial dependencies}. Formal languages like $\{ w \sym{\#} w \mid w \in \{\sym{0}, \sym{1}\}^\ast \}$ and $\{ w w \mid w \in \{\sym{0}, \sym{1}\}^\ast \}$ are classic examples of languages that cannot be recognized by any PDA \citep{hopcroft-ullman-1979-introduction,sipser-2013-introduction}, so a nondeterministic PDA should be of no use in learning these constructions.

We show, rather surprisingly, that this is not the case. Whereas an ordinary WPDA must use the same transition weights for all timesteps, the RNS-RNN can update them based on the status of ongoing nondeterministic runs of the WPDA. This means it can coordinate multiple runs in a way a PDA cannot\dash{}namely, to recognize intersections of CFLs, or to simulate multiple stacks.

In this chapter, we prove that the RNS-RNN can recognize all CFLs and intersections of CFLs. We show empirically that the RNS-RNN can model some non-CFLs; in fact it is the only stack RNN able to learn $\{ w \sym{\#} w \mid w \in \{ \sym{0}, \sym{1} \}^\ast \}$, whereas a deterministic multi-stack RNN cannot.

\section{Updates to the RNS-RNN}
\label{sec:rns-rnn}

First, we present an updated version of the RNS-RNN, which we use in all subsequent proofs and experiments. The changes we make are not necessary for achieving the results in this chapter, but they bring the implementation of the differentiable WPDA to parity with \cref{def:restricted-pda}, and they improve its time complexity by a factor of $|Q|$, allowing us to train larger models.

\subsection{Definition}
\label{sec:updated-rns-rnn}

The only changes we make are to the tensor of inner weights $\nsinnerweightletter$ and the tensor of forward weights $\nsforwardweightletter$. The tensor $\nsinnerweightletter$ now has size $n \times n \times |Q| \times |\Gamma| \times |Q| \times |\Gamma|$. For $1 \leq t \leq n-1$ and $-1 \leq i \leq t-1$,
\begin{align}
    \nsinnerweight{-1}{q}{x}{0}{r}{y} &= \indicator{q = q_0 \wedge x = \bot \wedge r = q_0 \wedge y = \bot} \label{eq:new-rns-rnn-gamma-init} \\
    \nsinnerweight{i}{q}{x}{t}{r}{y} &= 
    \begin{aligned}[t]
    &\indicator{i=t\!-\!1} \; \nspushweight{q}{t}{x}{r}{y} && \text{push} \\
                      & +\! \sum_{s,z} \nsinnerweight{i}{q}{x}{t\!-\!1}{s}{z} \; \nsreplweight{s}{t}{z}{r}{y} && \text{repl.} \\
                      & +\! \sum_{\mathclap{k=i+1}}^{t-2} \sum_{u} \nsinnerweight{i}{q}{x}{k}{u}{y} \; \nsinnerweightaux{k}{u}{y}{t}{r} && \! \text{pop}
    \end{aligned}
    \label{eq:new-rns-rnn-gamma} \\
    \nsinnerweightaux{k}{u}{y}{t}{r} &= \sum_{s, z} \nsinnerweight{k}{u}{y}{t\!-\!1}{s}{z} \; \nspopweight{s}{t}{z}{r} \; (0 \leq k \leq t\!-\!2). \label{eq:new-rns-rnn-gamma-prime}
\end{align}
The tensor $\nsforwardweightletter$ now has size $(n+1) \times |Q| \times |\Gamma|$.
\begin{align}
    \nsforwardweight{-1}{r}{y} &= \indicator{r = q_0 \wedge y = \bot} \label{eq:new-rns-rnn-alpha-init} \\
    \nsforwardweight{t}{r}{y} &=
    \sum_{i=-1}^{t-1} \sum_{q, x} \nsforwardweight{i}{q}{x} \, \nsinnerweight{i}{q}{x}{t}{r}{y} \quad (0 \leq t \leq n-1) \label{eq:new-rns-rnn-alpha-recurrence}
\end{align}

Next, we describe the two changes we have made to \cref{eq:ns-rnn-gamma,eq:ns-rnn-alpha-init,eq:ns-rnn-alpha-recurrence} to arrive at \cref{eq:new-rns-rnn-gamma-init,eq:new-rns-rnn-gamma-prime,eq:new-rns-rnn-gamma,eq:new-rns-rnn-alpha-init,eq:new-rns-rnn-alpha-recurrence}.

\subsection{Asymptotic Speedup}
\label{sec:langs-algorithm-speedup}

As mentioned in \cref{sec:langs-algorithm}, the time complexity of Lang's algorithm, and therefore the differentiable WPDA and RNS-RNN too, is $\bigo{{|Q|}^4 {|\Gamma|}^3 n^3}$, and its space complexity is $\bigo{{|Q|}^2 {|\Gamma|}^2 n^2}$. However, as noted by \citet{butoi-etal-2022-algorithms}, it is possible to improve its asymptotic time complexity with respect to $|Q|$ with a simple change: precomputing the product $\nsinnerweight{k}{u}{y}{t-1}{s}{z} \; \nspopweight{s}{t}{z}{r}$ used in the pop rule in \cref{eq:ns-rnn-gamma}, which we store in a tensor $\nsinnerweightauxletter$ in \cref{eq:new-rns-rnn-gamma-prime}. This reduces the time complexity of the RNS-RNN to $\bigo{{|Q|}^3 {|\Gamma|}^3 n^2 + {|Q|}^3 {|\Gamma|}^2 n^3}$.

\subsection{Bottom Symbol Fix}
\label{sec:rns-rnn-bottom-symbol-fix}

As mentioned in \cref{sec:ns-rnn-langs-algorithm}, in the previous implementation of the differentiable WPDA, the initial $\bot$ can never be replaced, and once a symbol is pushed on top of it after the first timestep, the initial $\bot$ can never be exposed on the top of the stack again. We rectify this by allowing the symbol above the initial $\bot$ to be popped, and allowing the initial $\bot$ symbol to be replaced with a different symbol type at any time, in conformance with \cref{def:restricted-pda}. We do this by simulating an extra push action at $t = -1$ in \cref{eq:new-rns-rnn-gamma-init}. Although it would be possible to prevent the initial $\bot$ from being replaced by masking out replace weights, this would be slightly more costly to implement.

\section{Recognition Power of RNS-RNNs}
\label{sec:proofs}

In this section, we investigate the power of RNS-RNNs as language recognition devices, proving that they can recognize all CFLs and all intersections of CFLs. These results hold true even when the RNS-RNN is run in real time (one timestep per input, with one extra timestep to read $\eos$). Although \citet{siegelmann-sontag-1992-computational} showed that even simple RNNs are as powerful as Turing machines, this result relies on assumptions of infinite precision and unlimited extra timesteps, which generally do not hold true in practice. The same limitation applies to the neural Turing machine \citep{graves-etal-2014-neural}, which, when implemented with finite precision, is no more powerful than a finite automaton, as its tape does not grow with input length. Previously, \citet{stogin-etal-2020-provably} showed that a variant of the superposition stack is at least as powerful as real-time DPDAs. Here we show that RNS-RNNs recognize a much larger superset of languages.

For this section only, we allow parameters to have values of $\pm\infty$, to enable the controller to emit probabilities of exactly zero. Because we use RNS-RNNs here for accepting or rejecting strings (whereas in the rest of the chapter, we only use them as language models for predicting the next symbol), we start by providing a formal definition of language recognition for RNNs \citep[cf.][]{chen-etal-2018-recurrent}.
\begin{definition}
Let $N$ be an RNN controller, possibly augmented with a differentiable stack. Let $\rnnhiddent{t} \in \realset^d$ be the hidden state of $N$ after reading $t$ symbols, and let $\logisticletter$ be the logistic sigmoid function. We say that $N$ \defterm{recognizes} language $L$ if there is a multi-layer perceptron (MLP) $\rnnrecognizeroutput = \logistic{\affine{2}{\,\logistic{\affine{1}{\rnnhiddent{|w|+1}}}}}$ such that, after reading $w \concatop \eos$, we have $\rnnrecognizeroutput > \tfrac{1}{2}$ iff $w \in L$.
\end{definition}

\subsection{Context-Free Languages}

Now, we show that RNS-RNNs recognize all CFLs under this definition.

\begin{proposition} \label{thm:cfl}
For every context-free language $L$, there exists an RNS-RNN that recognizes $L$.
\end{proposition}
\begin{proof}[Proof sketch]
We know from \cref{thm:restricted-pdas-recognize-all-cfls} that there is a restricted PDA $P$ that recognizes $L$. We construct an RNS-RNN that emits, at every timestep, weight 1 for transitions of $P$ and 0 for all others. Then $\rnnrecognizeroutput > \frac{1}{2}$ iff the PDA ends in an accept configuration.
\end{proof}

Since we already proved that there is a restricted PDA for every CFL (\cref{thm:restricted-pdas-recognize-all-cfls}), all we need to do is show that a restricted PDA can be converted to an equivalent RNS-RNN.

\begin{lemma} \label{thm:constantpda}
For any restricted PDA $P$ that recognizes language $L$, there is an RNS-RNN that recognizes $L$.
\end{lemma}
\begin{proof}
Let $P = (Q, \Sigma, \Gamma, \delta, q_0, F)$.

We write $\pdanumruns{\inputstring}$ for the total number of runs of $P$ that read $\inputstring$ and end with $\bot$ on top of the stack. We assume that $\pdanumruns{\inputstring} > 0$; this can always be ensured by adding to $P$ an extra non-accepting state $q_{\mathrm{trap}}$ and transitions $\pdatrans{q_0}{a}{\bot}{q_{\mathrm{trap}}}{\bot}$ and $\pdatrans{q_{\mathrm{trap}}}{a}{\bot}{q_{\mathrm{trap}}}{\bot}$ for all $a \in \Sigma$.

For any state set $X \subseteq Q$, we write $\pdanumrunsstateset{\inputstring}{X}$ for the total number of runs of $P$ that read $\inputstring$ and end in a state in $X$ with $\bot$ on top of the stack. Then for any string $\inputstring$, if $\inputstring \in L$, then $\pdanumrunsstateset{\inputstring}{F} \geq 1$; otherwise, $\pdanumrunsstateset{\inputstring}{F} = 0$.

We use the following definition for the LSTM controller of the RNS-RNN, where $\lstminputgatet{t}$, $\lstmforgetgatet{t}$, and $\lstmoutputgatet{t}$ are the input, forget, and output gates, respectively, and $\lstmcandidatet{t}$ is the candidate memory cell.
\begin{align*}
    \lstminputgatet{t} &= \logistic{\affine{i}{
        \begin{bmatrix}
            \rnninputt{t} \\
            \rnnhiddent{t-1}
        \end{bmatrix}
    }} \\
    \lstmforgetgatet{t} &= \logistic{\affine{f}{
        \begin{bmatrix}
            \rnninputt{t} \\
            \rnnhiddent{t-1}
        \end{bmatrix}
    }} \\
    \lstmcandidatet{t} &= \tanh(\affine{g}{
        \begin{bmatrix}
            \rnninputt{t} \\
            \rnnhiddent{t-1}
        \end{bmatrix}
    }) \\
    \lstmoutputgatet{t} &= \logistic{\affine{o}{
        \begin{bmatrix}
            \rnninputt{t} \\
            \rnnhiddent{t-1}
        \end{bmatrix}
    }} \\
    \lstmmemorycellt{t} &= \lstmforgetgatet{t} \elementwisemulop \lstmmemorycellt{t-1} + \lstminputgatet{t} \elementwisemulop \lstmcandidatet{t} \\
    \rnnhiddent{t} &= \lstmoutputgatet{t} \elementwisemulop \tanh(\lstmmemorycellt{t})
\end{align*}

Construct an RNS-RNN as follows. At each timestep $t$, upon reading input embedding $\rnninputt{t}$, make the controller emit a weight of 1 for all transitions of $P$ that scan input symbol $\inputsymbolt{t}$, and a weight of 0 for all other transitions (by setting the corresponding weights of $\nstranst{t}$ to $-\infty$).

Let $n = |\inputstring|$. After reading $\inputstring$, the stack reading $\stackreadingt{n}$ is the probability distribution of states and top stack symbols of $P$ after reading $\inputstring$, which the controller can use to compute $\rnnhiddent{n+1}$ and the MLP output $\rnnrecognizeroutput$ as follows. Let
\begin{equation*}
    p = \sum_{f \in F} \nsstackreading{n}{f}{\bot} = \frac{\pdanumrunsstateset{\inputstring}{F}}{\pdanumruns{\inputstring}}.
\end{equation*}
Note that $p$ is positive if $\inputstring \in L$, and zero otherwise. An affine layer connects $\rnnhiddent{n}$, $\rnninputt{n+1}$, and $\stackreadingt{n}$ to the candidate memory cell $\lstmcandidatet{n+1}$ (recall that the input at $n+1$ is $\eos$). Designate a unit $\lstmcandidateunit{n+1}{accept}$ in $\lstmcandidatet{n+1}$. For each $f \in F$, set the incoming weight from $\nsstackreading{n}{f}{\bot}$ to 1 and all other incoming weights to 0, so that $\lstmcandidateunit{n+1}{accept} = \tanh(p)$. Set this unit's input gate to 1 and forget gate to 0, so that memory cell $\lstmmemorycellunit{n+1}{accept} = \tanh(p)$. Set its output gate to 1, so that hidden unit $\rnnhiddenunit{n+1}{accept} = \tanh(\tanh(p))$, which is positive if $\inputstring \in L$ and zero otherwise.

Finally, to compute $\rnnrecognizeroutput$, use one hidden unit $\rnnrecognizermlpunit{1}$ in the output MLP layer, and set the incoming weight from $\rnnhiddenunit{n+1}{accept}$ to 1, and all other weights to 0. So $\rnnrecognizermlpunit{1} = \logistic{\tanh(\tanh(p))}$, which is greater than $\frac{1}{2}$ if $\inputstring \in L$ and equal to $\frac{1}{2}$ otherwise. Set the weight connecting $\rnnrecognizermlpunit{1}$ to $\rnnrecognizeroutput$ to 1, and set the bias term to $-\frac{1}{2}$, so that $\rnnrecognizeroutput = \logistic{\logistic{\tanh(\tanh(p)}} - \frac{1}{2})$, which is greater than $\frac{1}{2}$ if $\inputstring \in L$, and equal to $\frac{1}{2}$ otherwise.
\end{proof}

\subsection{Intersections of Context-Free Languages}

Next, we show that RNS-RNNs recognize a large class of non-CFLs: intersections of CFLs. Since the class of languages formed by the intersection of $k$ CFLs is a proper superset of the class formed by the intersection of $(k-1)$ CFLs \citep{liu-weiner-1973-infinite}, this means that RNS-RNNs are considerably more powerful than nondeterministic PDAs.

\begin{proposition} \label{thm:intersect}
For every finite set of context-free languages $L_1, \ldots, L_k$ over the same alphabet $\Sigma$, there exists an RNS-RNN that recognizes $L_1 \cap \cdots \cap L_k$.
\end{proposition}
\begin{proof}[Proof sketch]
Without loss of generality, assume $k=2$. Let $P_1$ and $P_2$ be PDAs recognizing $L_1$ and $L_2$, respectively. We construct a PDA $P$ that uses nondeterminism to simulate $P_1$ \emph{or} $P_2$, but the controller can query $P_1$ and $P_2$'s configurations, and it can set $\rnnrecognizeroutput > \frac{1}{2}$ iff $P_1$ \emph{and} $P_2$ both end in accept configurations.
\end{proof}

\begin{proof}
Without loss of generality, assume $k=2$. Let $P_1 = (Q_1, \Sigma, \Gamma_1, \delta_1, s_1, F_1)$ and $P_2 = (Q_2, \Sigma, \Gamma_2, \delta_2, s_2, F_2)$ be restricted PDAs recognizing $L_1$ and $L_2$, respectively. We can construct both so that $Q_1 \cap Q_2 = \emptyset$, and $s_1$ and $s_2$ each have no incoming transitions.
Construct a new PDA 
\begin{align*}
P &= (Q, \Sigma, \Gamma_1 \cup \Gamma_2, \delta, s, F) \\
Q &= (Q_1 \setminus \{s_1\}) \cup (Q_2 \setminus \{s_2\}) \cup \{s\} \\
\delta(q, x, a) &= \begin{cases}
\delta_1(q, x, a) & q \in Q_1 \setminus \{s_1\} \\
\delta_2(q, x, a) & q \in Q_2 \setminus \{s_2\} \\
\delta_1(s_1, x, a) \cup \delta_2(s_2, x, a) & q = s
\end{cases} \\
F &= \begin{cases}
(F_1 \setminus \{s_1\}) \cup (F_2 \setminus \{s_2\}) \cup \{s\} & s_1 \in F_1 \wedge s_2 \in F_2 \\
(F_1 \setminus \{s_1\}) \cup (F_2 \setminus \{s_2\}) & \text{otherwise.}
\end{cases}
\end{align*}
So far, this is just the standard union construction for PDAs.

Construct an RNS-RNN that sets $\nstranst{t}$ according to $\delta$, and assume $\pdanumruns{\inputstring} > 0$, as in \cref{thm:constantpda}. Let
\begin{align*}
    p_1 &= \sum_{f \in F_1} \nsstackreading{n}{f}{\bot} = \frac{\pdanumrunsstateset{\inputstring}{F_1}}{\pdanumruns{\inputstring}} \\
    p_2 &= \sum_{f \in F_2} \nsstackreading{n}{f}{\bot} = \frac{\pdanumrunsstateset{\inputstring}{F_2}}{\pdanumruns{\inputstring}}.
\end{align*}
Let $n = |\inputstring|$. For each $p_i$, designate a unit $\lstmcandidateuniti{n+1}{accept}{i}$ in $\lstmcandidatet{n+1}$ that will be positive if $p_1 > 0$ and negative if $p_1 = 0$. In order to ensure that $\lstmcandidateuniti{n+1}{accept}{i} \not= 0$, we subtract a small value from $p_i$ that is smaller than the smallest possible non-zero value of $p_i$. The RNS-RNN computes this value as follows. Let $b = |Q| (2 |\Gamma| + 1)$, which is the maximum number of choices $P$ can make from any given configuration. Designate one memory cell $\lstmcandidateunit{t}{offset}$ in $\lstmcandidatet{t}$. At the first timestep, $\lstmcandidateunit{t}{offset}$ is initialized to $\frac{1}{b}$, and at each subsequent timestep, the forget gate is used to multiply this cell by $\frac{1}{b}$ (the input gate is set to 0). So after reading $t$ symbols, $\lstmcandidateunit{t}{offset} = \frac{1}{b^t}$. Set the output gate for that cell to 1, so there is a hidden unit $\rnnhiddenunit{n}{offset}$ in $\rnnhiddent{n}$ that contains the value $\tanh(\frac{1}{b^n})$. The smallest non-zero value of $p_i$ is $\frac{1}{\pdanumruns{\inputstring}}$, and since $\pdanumruns{\inputstring} \leq b^n$ and $\tanh(x) < x$ when $x > 0$, $\rnnhiddenunit{n}{offset}$ is smaller than it.

As for the incoming weights to $\lstmcandidateuniti{n+1}{accept}{i}$, for each $f \in F_i$, set the weight for $\nsstackreading{n}{f}{\bot}$ to 1, and set the weight for $\rnnhiddenunit{n}{offset}$ to $-1$; set all other weights to 0. So $\lstmcandidateuniti{n+1}{accept}{i} = \tanh(p_i - \rnnhiddenunit{n}{offset}$), which is positive if $p_i > 0$ and negative otherwise. Set this unit's input gate to 1 and forget gate to 0, so that memory cell $\lstmmemorycelluniti{n+1}{accept}{i} = \lstmcandidateuniti{n+1}{accept}{i}$. Set its output gate to 1, so that hidden unit $\rnnhiddenuniti{n+1}{accept}{i} = \tanh(\lstmcandidateuniti{n+1}{accept}{i})$, which is positive if $p_i > 0$ and negative otherwise.

In the output MLP's hidden layer, for each $\rnnhiddenuniti{n+1}{accept}{i}$, include a unit $\rnnrecognizermlpunit{i}$, and set its incoming weight from $\rnnhiddenuniti{n+1}{accept}{i}$ to $\infty$, and all other weights to 0. So $\rnnrecognizermlpunit{i} = \logistic{\infty \cdot \rnnhiddenuniti{n+1}{accept}{i}}$, which is 1 if $p_i > 0$ and 0 otherwise. Finally, to compute $\rnnrecognizeroutput$, set the incoming weights from each $\rnnrecognizermlpunit{i}$ to 1, and set the bias term to $-\frac{3}{2}$. So $\rnnrecognizeroutput = \logistic{\rnnrecognizermlpunit{1} + \rnnrecognizermlpunit{2} - \frac{3}{2}}$, which is greater than $\frac{1}{2}$ if both $p_1 > 0$ and $p_2 > 0$, meaning $\inputstring \in L_1 \cap L_2$, and less than $\frac{1}{2}$ otherwise. This proof can be extended to $k$ CFLs by simulating $k$ PDAs in parallel, designating units for $1 \leq i \leq k$, and setting the bias term for $\rnnrecognizeroutput$ to $-k + \frac{1}{2}$.
\end{proof}

\section{Experiments}
\label{sec:non-cfl-experiments}

We now explore the ability of stack RNNs to recognize non-context-free phenomena with a language modeling task on several non-CFLs.

\subsection{Non-Context-Free Language Tasks}
\label{sec:non-cfl-tasks}

Each non-CFL, which we describe below, can be recognized by a real-time three-stack automaton.

\begin{description}
\item[$\bm{\CountThree{}}$] The language $\{ \sym{a}^n \sym{b}^n \sym{c}^n \mid n \geq 0 \}$, a classic example of a non-CFL \citep{sipser-2013-introduction}. A two-stack automaton can recognize this language as follows. While reading the $\sym{a}$'s, push them to stack 1. While reading the $\sym{b}$'s, match them with $\sym{a}$'s popped from stack 1 while pushing $\sym{b}$'s to stack 2. While reading the $\sym{c}$'s, match them with $\sym{b}$'s popped from stack 2. As the stacks are only needed to remember the \emph{count} of each symbol type, this language is also an example of a counting language; \Citet{weiss-etal-2018-practical} showed that LSTMs can learn this language by using their memory cells as counters.
\item[$\bm{\MarkedReverseAndCopy{}}$] The language $\{ w \sym{\#} \reverse{w} \sym{\#} w \mid w \in \{\sym{0}, \sym{1}\}^\ast \}$. A two-stack automaton can recognize it as follows. While reading the first $w$, push it to stack 1. While reading the middle $\reverse{w}$, match it with symbols popped from stack 1 while pushing $\reverse{w}$ to stack 2. While reading the last $w$, match it with symbols popped from stack 2. The explicit $\sym{\#}$ symbols are meant to make it easier for a model to learn when to transition between these three phases.
\item[$\bm{\CountAndCopy{}}$] The language $\{w \sym{\#}^n w \mid \text{$w \in \{\sym{0}, \sym{1}\}^\ast$, $n \geq 0$, and $|w| = n$}\}$. A two-stack automaton can recognize it as follows. While reading the first $w$, push it to stack 1. While reading $\sym{\#}^n$, move the symbols from stack 1 to stack 2 in reverse. While reading the final $w$, match it with symbols popped from stack 2. Whereas the middle $\reverse{w}$ in \MarkedReverseAndCopy{} rewards a model for learning to push the first $w$, in \CountAndCopy{}, the middle $\sym{\#}^n$ section does not; a model must coordinate both phases of pushing and popping described above to receive a reward.
\item[$\bm{\MarkedCopy{}}$] The language $\{ w \sym{\#} w \mid w \in \{\sym{0}, \sym{1}\}^\ast \}$. A three-stack automaton can recognize this language as follows. Let $w = uv$ where $|u| = |v|$ (for simplicity assume $|w|$ is even). While reading the first $u$, push it to stack 1. While reading the first $v$, push it to stack 2, and move the symbols from stack 1 to stack 3 in reverse. While reading the second $u$, match it with symbols popped from stack 3, and move the symbols from stack 2 to stack 1 in reverse. While reading the second $v$, match it with symbols popped from stack 1. The explicit $\sym{\#}$ symbol is meant to make learning this task easier.
\item[$\bm{\UnmarkedCopyDifferentAlphabets{}}$] The language $\{ w w' \mid \text{$w \in \{\sym{0}, \sym{1}\}^\ast$ and $w' = \phi(w)$} \}$, where $\phi$ is the homomorphism $\phi(\sym{0}) = \sym{2}$, $\phi(\sym{1}) = \sym{3}$. A three-stack automaton can recognize this language using a similar strategy to \MarkedCopy{}. In this case, a switch to a different alphabet, rather than a $\sym{\#}$ symbol, marks the second half of the string.
\item[$\bm{\UnmarkedReverseAndCopy{}}$] The language $\{ w \reverse{w} w \mid w \in \{\sym{0}, \sym{1}\}^\ast \}$. A two-stack automaton can recognize it using a similar strategy to \MarkedReverseAndCopy{}, but it must nondeterministically guess $|w|$.
\item[$\bm{\UnmarkedCopy{}}$] The language $\{ ww \mid w \in \{ \sym{0}, \sym{1} \}^\ast \}$, another classic example of a non-CFL \citep{sipser-2013-introduction}. A three-stack automaton can recognize it using a similar strategy to \MarkedCopy{}, except it must nondeterministically guess $|w|$.
\end{description}

The above languages all include patterns like $w \cdots w$, resembling cross-serial dependencies found in some human languages. In Swiss German \citep{shieber-1985-evidence}, the two $w$'s are distinguished by part-of-speech (a sequence of nouns and verbs, respectively), analogous to \UnmarkedCopyDifferentAlphabets{}. In Bambara \citep{culy-1985-complexity}, the two $w$'s are the same, but separated by a morpheme \textit{o}, analogous to \MarkedCopy{}.

\subsection{Data Sampling and Evaluation}

As in \cref{sec:learning-cfls-data-sampling}, before each training run, we sample a training set of 10,000 examples and a validation set of 1,000 examples from $\sampleprobname{L}$. To sample a string $w \in L$, we first uniformly sample a length $\ell$ from $[\ell_{\mathrm{min}}, \ell_{\mathrm{max}}]$ (as before, we use $\ell_{\mathrm{min}} = 40$ and $\ell_{\mathrm{max}} = 80$), then sample uniformly from $\stringsoflength{L}{\ell}$ (we avoid sampling lengths for which $\stringsoflength{L}{\ell}$ is empty). So, the distribution from which $w$ is sampled (cf. \cref{eq:cfg-task-distribution}) is
\begin{equation*}
    \sampleprob{L}{w} = \frac{1}{\big|\{ \ell \in [\ell_{\mathrm{min}}, \ell_{\mathrm{max}}] \mid \stringsoflength{L}{\ell} \not= \emptyset \}\big|} \frac{1}{\left|\stringsoflength{L}{|w|}\right|}.
\end{equation*}
For each language, we sample a single test set that is reused across all training runs. Examples in the test set vary in length from 40 to 100, with 100 examples sampled uniformly from $\stringsoflength{L}{\ell}$ for each length $\ell$.

We evaluate models using cross-entropy difference as in \cref{sec:rns-rnn-cfl-evaluation}, using $\sampleprob{L}{w}$ to compute the lower bound. For each non-CFL in \cref{sec:non-cfl-tasks}, $|\stringsoflength{L}{|w|}|$ can be computed directly from $|w|$, so computing $\sampleprob{L}{w}$ is straightforward.

\subsection{Models}
\label{sec:learning-non-cfls-models}

We compare five architectures, each of which consists of an LSTM controller connected to a different type of data structure. We include a bare LSTM baseline (``LSTM''). We also include a model that pushes learned vectors of size 10 to a superposition stack (``Sup.\ 10''), and another that pushes the controller's hidden state (``Sup.\ h''). Since each of these languages can be recognized by a three-stack automaton, we also test a model that is connected to three instances of the superposition stack, each of which has vectors of size 3 (``Sup.\ 3-3-3''). In this model, the controller computes a separate set of actions for each stack, and the stack reading returned to it is the concatenation of the stack readings of all three stacks. This allows Sup.\ 3-3-3 to transfer data among stacks. Finally, we include an RNS-RNN with $|Q| = 3$ and $|\Gamma| = 3$ (``RNS 3-3''). The RNS-RNN can effectively simulate multiple PDAs in parallel by partitioning $Q$ or $\Gamma$, and $|Q| = 3$ is sufficient to simulate three stacks (more precisely, three single-state PDAs). In all cases, the LSTM controller has one layer and 20 hidden units. We encode all input symbols as one-hot vectors.

\subsection{Training}
\label{sec:learning-non-cfls-training}

For each language and architecture, we train 10 models and report results for the model with the lowest cross-entropy difference on the validation set. We train each model by minimizing its cross-entropy (summed over the timestep dimension of each batch) on the training set, and we use per-symbol cross-entropy on the validation set as the early stopping criterion. For each training run, we randomly sample the initial learning rate from a log-uniform distribution over $[5\times{10}^{-4}, 1\times{10}^{-2}]$, and we use a gradient clipping threshold of 5. All other training details are the same as \cref{sec:rns-rnn-cfl-models-and-training}.

\subsection{Results}

We show cross-entropy difference on the validation and test sets in \cref{fig:non-cfl-results-1,fig:non-cfl-results-2}.

{
\newcommand{\myrow}[1]{
    \scalebox{0.8}{\input{figures/04-non-cfls/train/#1}}
    &\scalebox{0.8}{\input{figures/04-non-cfls/test/#1}} \\}
\newcommand{\myfigure}[3]{
    \begin{figure*}
        \pgfplotsset{
            every axis/.style={
                width=3.55in,
                height=#1
            },
            title style={yshift=-4.5ex},
            y tick label style={
                /pgf/number format/.cd,
                fixed,
                fixed zerofill,
                precision=1,
                /tikz/.cd
            }
        }
        \centering
        \begin{tabular}{@{}l@{\hspace{0in}}l@{}}    
            \multicolumn{2}{c}{\input{figures/04-non-cfls/legend}} \\
            #2
        \end{tabular}
        #3
    \end{figure*}}

    \begin{figure*}
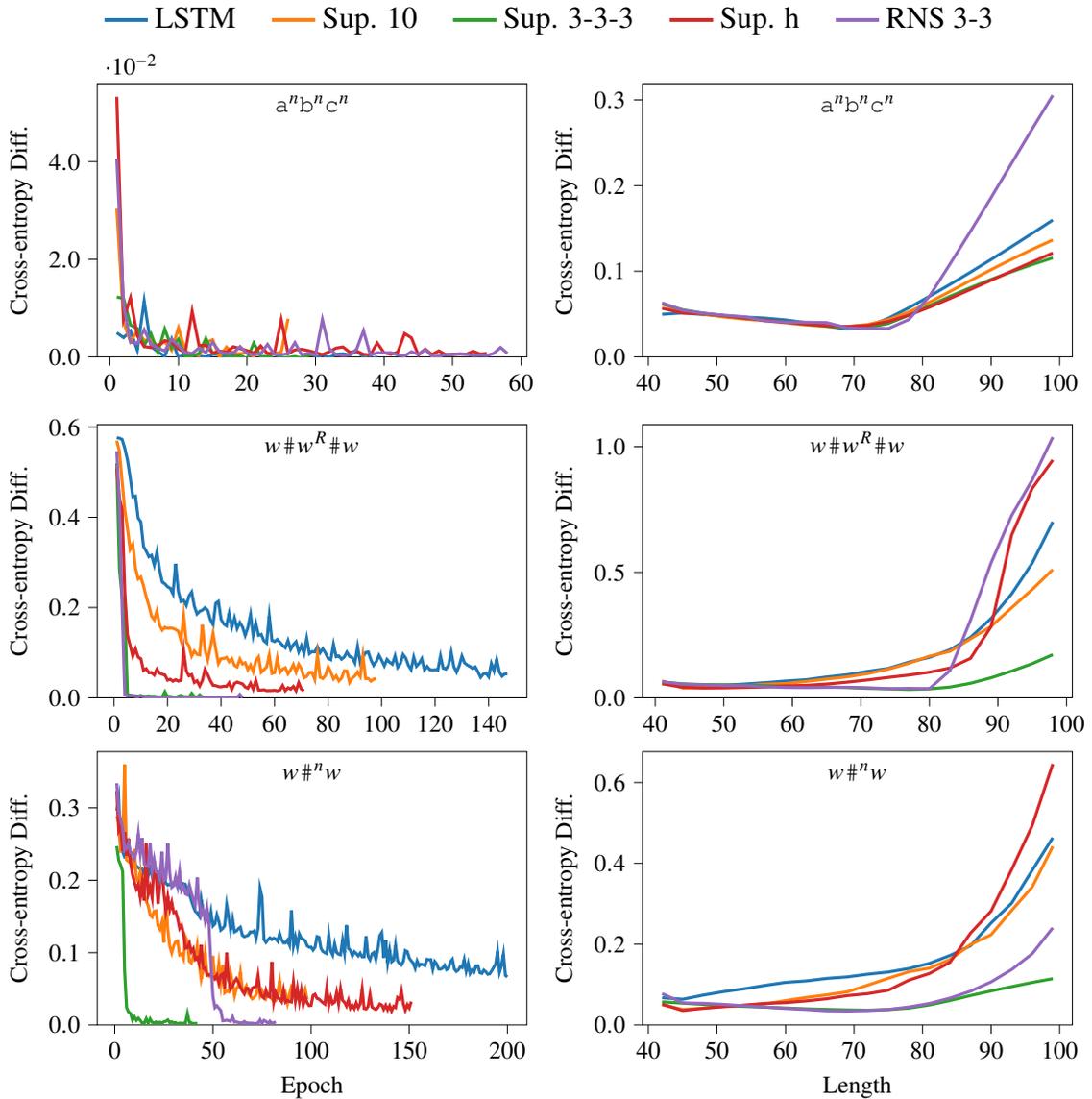

        \pgfplotsset{
            every axis/.style={
                width=3.55in,
                height=2.485in
            },
            title style={yshift=-4.5ex},
            y tick label style={
                /pgf/number format/.cd,
                fixed,
                fixed zerofill,
                precision=1,
                /tikz/.cd
            }
        }
        \centering
        \begin{tabular}{@{}l@{\hspace{0in}}l@{}}    
            \multicolumn{2}{c}{\input{figures/04-non-cfls/legend}} \\

    \scalebox{0.8}{\input{figures/04-non-cfls/train/count-3}}
    &\scalebox{0.8}{\input{figures/04-non-cfls/test/count-3}} \\
    
    \scalebox{0.8}{\input{figures/04-non-cfls/train/marked-reverse-and-copy}}
    &\scalebox{0.8}{\input{figures/04-non-cfls/test/marked-reverse-and-copy}} \\
    
    \scalebox{0.8}{\input{figures/04-non-cfls/train/count-and-copy}}
    &\scalebox{0.8}{\input{figures/04-non-cfls/test/count-and-copy}} \\

        \end{tabular}
        
    \caption[Results on three non-CFLs.]{Performance on three non-CFLs. Left: Cross-entropy difference in nats, on the validation set by epoch. Right: Cross-entropy difference on the test set, binned by string length. Each line is the best of 10 runs, selected by validation performance. All models easily solve \CountThree{}. Only multi-stack models (Sup.\ 3-3-3 and RNS 3-3) achieve optimal cross-entropy on \MarkedReverseAndCopy{} and \CountAndCopy{}.}
    \label{fig:non-cfl-results-1}

    \end{figure*}

    \begin{figure*}
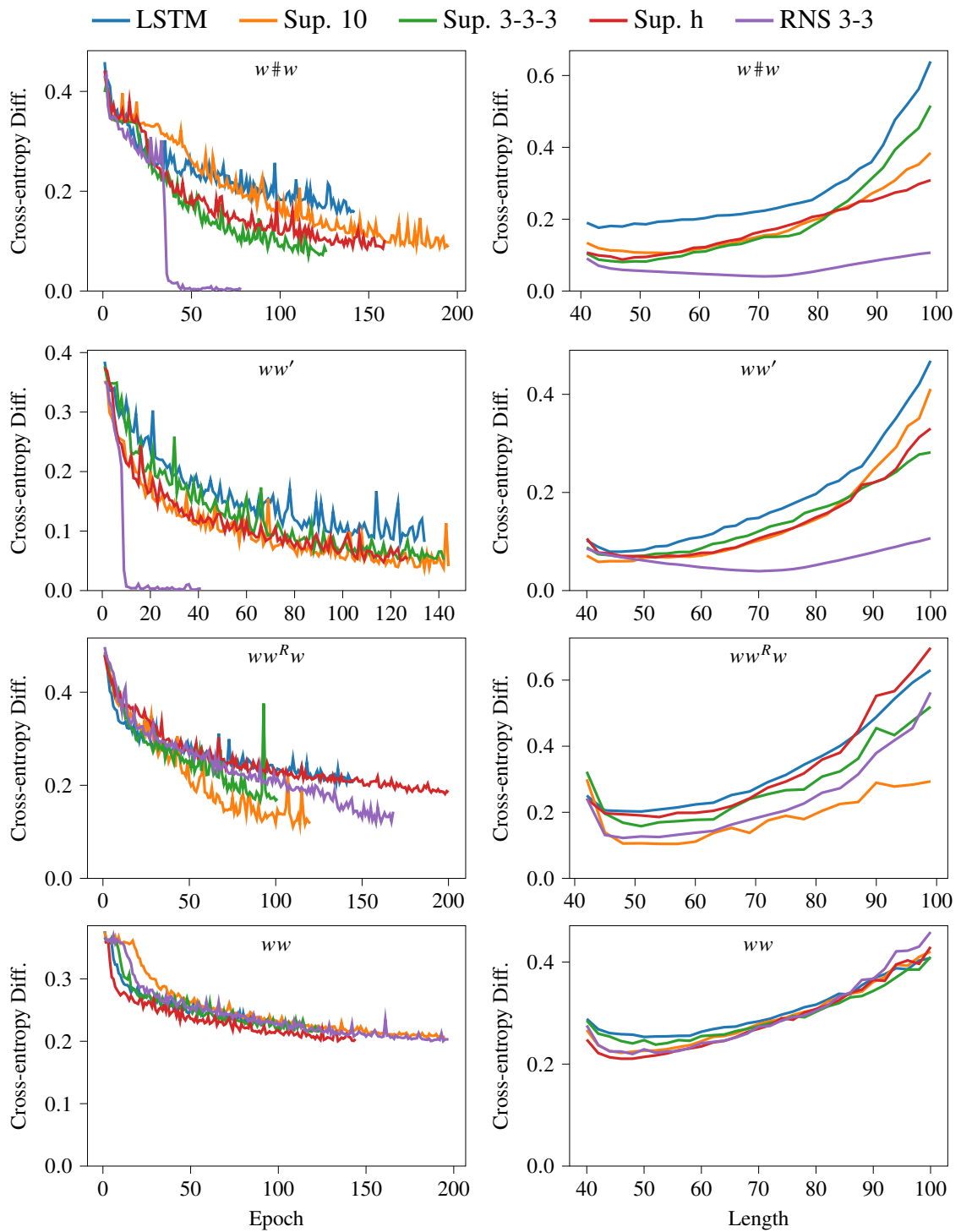

        \pgfplotsset{
            every axis/.style={
                width=3.55in,
                height=2.485in
            },
            title style={yshift=-4.5ex},
            y tick label style={
                /pgf/number format/.cd,
                fixed,
                fixed zerofill,
                precision=1,
                /tikz/.cd
            }
        }
        \centering
        \begin{tabular}{@{}l@{\hspace{0in}}l@{}}    
            \multicolumn{2}{c}{\input{figures/04-non-cfls/legend}} \\

    \scalebox{0.8}{\input{figures/04-non-cfls/train/marked-copy}}
    &\scalebox{0.8}{\input{figures/04-non-cfls/test/marked-copy}} \\
    
    \scalebox{0.8}{\input{figures/04-non-cfls/train/unmarked-copy-different-alphabets}}
    &\scalebox{0.8}{\input{figures/04-non-cfls/test/unmarked-copy-different-alphabets}} \\
    
    \scalebox{0.8}{\input{figures/04-non-cfls/train/unmarked-reverse-and-copy}}
    &\scalebox{0.8}{\input{figures/04-non-cfls/test/unmarked-reverse-and-copy}} \\
    
    \scalebox{0.8}{\input{figures/04-non-cfls/train/unmarked-copy}}
    &\scalebox{0.8}{\input{figures/04-non-cfls/test/unmarked-copy}} \\

        \end{tabular}
        
    \caption[Results on four more non-CFLs.]{Performance on four non-CFLs. Continued from \cref{fig:non-cfl-results-1}. On \MarkedCopy{} and \UnmarkedCopyDifferentAlphabets{}, which have no extra timesteps, only RNS 3-3 achieves optimal cross-entropy. No models solve \UnmarkedReverseAndCopy{} or \UnmarkedCopy{} optimally.}
    \label{fig:non-cfl-results-2}

    \end{figure*}
}

All models easily learn \CountThree{}, likely because the LSTM controller by itself can solve it using a counting mechanism \citep{weiss-etal-2018-practical}. The input and forget gates of the LSTM controller are not tied, so the values of its memory cells are not bounded to $(0, 1)$ and can be incremented and decremented. Note also that because there is at most one string in \CountThree{} for a given length, all strings in the test set between lengths 40 and 80 are in the training set.

Strings in \MarkedReverseAndCopy{} and \CountAndCopy{} contain two hints to facilitate learning: explicit boundaries for $w$, and extra timesteps in the middle, which simplify the task of transferring symbols between stacks (for details, compare the solutions for \MarkedReverseAndCopy{} and \MarkedCopy{} described in \cref{sec:non-cfl-tasks}). Only the models capable of simulating multiple stacks, Sup.\ 3-3-3 and RNS 3-3, achieve optimal cross-entropy on these two tasks. However, we note that RNS 3-3 does not generalize as well on lengths not seen during training.

Only RNS 3-3 learns \MarkedCopy{} and \UnmarkedCopyDifferentAlphabets{}. This suggests that although multiple deterministic stacks (Sup.\ 3-3-3) are enough to learn to copy a string given enough hints in the input (explicit boundaries, extra timesteps for computation), only the nondeterministic stack succeeds when the number of hints is reduced (no extra timesteps). A deterministic three-stack automaton could recognize \MarkedCopy{} if it were allowed to execute $|w|$ non-scanning transitions after reading $\sym{\#}$ to transfer symbols from one stack to another, but Sup.\ 3-3-3 is real-time, which likely explains why it can learn \CountAndCopy{} but not \MarkedCopy{}. Note that the three-stack construction for recognizing \MarkedCopy{} described in \cref{sec:non-cfl-tasks} requires nondeterministically guessing where the middle of $w$ is.  As noted in \cref{sec:cfls}, nondeterminism overcomes the restrictions imposed by real-time execution in PDAs. This likely explains why RNS 3-3 can solve \MarkedCopy{} when Sup.\ 3-3-3 cannot.

No stack models learn to copy a string without explicit boundaries (\UnmarkedReverseAndCopy{} and \UnmarkedCopy{}), although we note that Sup.\ 10 achieves the best performance on the test set for \UnmarkedReverseAndCopy{}. Although RNS 3-3 can use nondeterminism to guess the midpoint of $w$ in \MarkedCopy{}, a nondeterministic PDA is evidently not enough to guess both the midpoint of $w$ and $ww$ in the \UnmarkedCopy{} task. Solving \UnmarkedReverseAndCopy{} and \UnmarkedCopy{} would likely require a nondeterministic multi-stack automaton such as a nondeterministic embedded PDA \citep{vijayashanker-1987-study}, as opposed to a model that uses nondeterminism to simulate multiple stacks.

\section{Conclusion}

We showed that the RNS-RNN can recognize all CFLs and a large class of non-CFLs, allaying concerns that its reliance on a restricted PDA is insufficient for representing certain kinds of syntax. We tested stack RNNs on a variety of non-CFLs, and we showed that only the RNS-RNN can learn cross-serial dependencies (provided the boundary is explicitly marked), unlike a deterministic multi-stack architecture.

%% file: chapters/09-vrns-rnn.tex
\chapter{The Vector Nondeterministic Stack RNN}
\label{chap:vrns-rnn}

So far, the computational complexity of the RNS-RNN has limited us to relatively small sizes for $Q$ and $\Gamma$; the largest stack alphabet used in \cref{chap:incremental-execution} contained only 11 symbol types. However, natural languages have very large vocabularies, and a stack alphabet of size 11 would appear to be grossly insufficient for storing detailed lexical information. In this chapter, we show that the RNS-RNN is actually surprisingly effective at encoding large numbers of token types with a small stack alphabet. For example, on \MarkedReversal{}, we show that an RNS-RNN with only 3 stack symbol types can learn to simulate a stack of no fewer than 200 symbol types. We show that it does this by encoding them in the \emph{distribution} over stacks in the stack WFA, representing symbol types as clusters of points in the space of stack readings.

We also propose a more deliberate solution to increasing the information capacity of the nondeterministic stack. We present a new version of the RNS-RNN that simulates a stack of discrete symbols tagged with \emph{vectors}, combining the benefits of nondeterminism and neural networks' ability to embed information in compact vector representations. We call this new model the \term{Vector RNS-RNN (VRNS-RNN)}, and we show that it is a generalization of both the RNS-RNN and the superposition stack RNN. We demonstrate perplexity improvements with this new model on the Penn Treebank language modeling benchmark.

This work will appear in a paper published at ICLR 2023 \citep{dusell-chiang-2023-surprising}. The code used for the experiments in this chapter is publicly available.\footnote{\url{https://github.com/bdusell/nondeterministic-stack-rnn}}

\section{Vector RNS-RNN}
\label{sec:vrns-rnn-definition}

The Vector RNS-RNN (VRNS-RNN) is an extension of the RNS-RNN that uses a stack whose elements are symbols drawn from $\Gamma$ and augmented with vectors of size $\stackvectorsizeletter$. We call a WPDA with such a stack a \term{vector WPDA}, and we call the corresponding differentiable stack module the \term{differentiable vector WPDA}. Whereas its time complexity is cubic in $|\Gamma|$, its time and space complexity scale only linearly with $\stackvectorsizeletter$. Each run of the vector WPDA involves a stack of elements in $\Gamma \times \realset^{\stackvectorsizeletter}$. We assume the initial stack consists of the element $(\bot, \vrnsbottomvector)$. In the VRNS-RNN, we set $\vrnsbottomvector = \logistic{\vectorparam{w}{v}}$, where $\vectorparam{w}{v}$ is a learned parameter. The stack operations in a vector WPDA have the following semantics:
\begin{description}
\item[Push $q, x \rightarrow r, y$] If $q$ is the current state and $(x, \vecvar{u})$ is on top of the stack, go to state $r$ and push $(y, \pushedstackvectort{t})$ with weight $\nspushweight{q}{t}{x}{r}{y}$, where $\pushedstackvectort{t} = \logistic{\affine{v}{\rnnhiddent{t}}}$.
\item[Replace $q, x \rightarrow r, y$] If $q$ is the current state and $(x, \vecvar{u})$ is on top of the stack, go to state $r$ and replace $(x, \vecvar{u})$ with $(y, \vecvar{u})$ with weight $\nsreplweight{q}{t}{x}{r}{y}$. Note that we do \emph{not} replace $\vecvar{u}$ with $\pushedstackvectort{t}$; we replace the discrete symbol only and keep the vector the same. When $x = y$, this is a no-op.
\item[Pop $q, x \rightarrow r$] If $q$ is the current state and $(x, \vecvar{u})$ is on top of the stack, go to state $r$ and remove $(x, \vecvar{u})$ with weight $\nspopweight{q}{t}{x}{r}$, uncovering the stack element beneath.
\end{description}
Let $\topvectorofrun{\pdarunletter}$ denote the top stack vector at the end of run $\pdarunletter$. The stack reading $\stackreadingt{t} \in \realset^{|Q| \cdot |\Gamma| \cdot \stackvectorsizeletter}$ now includes, for each $(r, y) \in Q \times \Gamma$, an interpolation of $\topvectorofrun{\pdarunletter}$ for every run $\pdarunendsin{\pdarunletter}{t}{r}{y}$, normalized by the weight of all runs.
\begin{equation}
    \vrnsstackreadingvec{t}{r}{y} = \frac{
        \sum_{\pdarunendsin{\pdarunletter}{t}{r}{y}} \wpdarunweight{\pdarunletter} \; \topvectorofrun{\pdarunletter}
        }{
            \sum_{r' \in Q} \sum_{y' \in \Gamma} \sum_{\pdarunendsin{\pdarunletter}{t}{r'}{y'}} \wpdarunweight{\pdarunletter}
        }
    \label{eq:vrns-reading-inefficient}
\end{equation}
We compute the denominator using $\nsinnerweightletter$, $\nsinnerweightauxletter$, and $\nsforwardweightletter$ as in \cref{sec:updated-rns-rnn}. To compute the numerator, we compute a new tensor $\vrnsinnervectorletter$ which stores $\wpdarunweight{\pdarunletter} \; \topvectorofrun{\pdarunletter}$. For $1 \leq t \leq n-1$ and $-1 \leq i \leq t-1$,
\begin{align}
    \vrnsinnervector{-1}{q}{x}{0}{r}{y} &=
        \indicator{q = q_0 \wedge x = \bot \wedge r = q_0 \wedge y = \bot} \; \vrnsbottomvector && \text{init.\footnotemark} \notag \\
    \vrnsinnervector{i}{q}{x}{t}{r}{y} &=
        \indicator{i=t-1} \; \nspushweight{q}{t}{x}{r}{y} \; \pushedstackvectort{t} && \text{push} \notag \\
        & +\! \sum_{s,z} \vrnsinnervector{i}{q}{x}{t-1}{s}{z} \; \nsreplweight{s}{t}{z}{r}{y} && \text{repl.} \notag \\
        & +\! \sum_{\mathclap{k=i+1}}^{t-2} \sum_{u} \vrnsinnervector{i}{q}{x}{k}{u}{y} \; \nsinnerweightaux{k}{u}{y}{t}{r}. && \text{pop}
\end{align}
\footnotetext{Our code and experiments for this chapter implement $\bm{\zeta}[-1 \rightarrow 0][q, x \rightarrow r, y] = \mathbf{v}_0$ instead due to a mistake found late in the publication process. Consequently, in \cref{eq:vrns-reading}, runs can start with \emph{any} $(r, y) \in Q \times \Gamma$ in the numerator, but only $(q_0, \bot)$ in the denominator. Empirically, the VRNS-RNN still appears to work as expected.}
We compute the normalized stack reading $\stackreadingt{t}$ as follows.
\begin{align}
    \vrnsstackreadingvec{t}{r}{y} &= \frac{ \vrnsstackreadingnumer{t}{r}{y} }{ \sum_{r', y'} \nsforwardweight{t}{r'}{y'} } \\
    \vrnsstackreadingnumer{t}{r}{y} &= \sum_{i=-1}^{t-1} \sum_{q, x} \nsforwardweight{i}{q}{x} \, \vrnsinnervector{i}{q}{x}{t}{r}{y}
    \label{eq:vrns-reading}
\end{align}

\section{Relationship of VRNS-RNN to Other Differentiable Stacks}
\label{sec:relationship-of-vrns-rnn}

The VRNS-RNN is a generalization of both the RNS-RNN and the superposition stack of \citet{joulin-mikolov-2015-inferring}. Clearly, the RNS-RNN is a special case of VRNS-RNN where $m = 1$ and $\pushedstackvectort{t} = 1$.

On the other hand, the superposition stack is a special case of the differentiable vector WPDA with $|Q| = 1$, $|\Gamma| = 1$, normalized transition weights, and $\pushedstackvectort{0} = \veczero$. To see why, let us unroll the superposition stack's formula for $\stackreadingt{t}$ for the first few timesteps.
\begin{align*}
\mathbf{r}_3 &= \suppusht{3} \pushedstackvectort{3} \\
    &+ \supnoopt{3} \suppusht{2} \pushedstackvectort{2} \\
    &+ \supnoopt{3} \supnoopt{2} \suppusht{1} \pushedstackvectort{1} \\
    &+ \suppopt{3} \suppusht{2} \suppusht{1} \pushedstackvectort{0}
\end{align*}
This is a summation whose terms enumerate all sequences of actions. In the general case, $\stackreadingt{t}$ can be expressed as
\begin{equation*}
    \mathbf{r}_t = \sum_{\pi \leadsto t} \psi(\pi) \; \mathbf{v}(\pi)
\end{equation*}
where $\pi$ is a run of a restricted vector WPDA where $|Q| = |\Gamma| = 1$, and $\pi \leadsto t$ means that $\pi$ ends at timestep $t$. So, like the VRNS-RNN, the stack reading of the superposition stack is a weighted sum of an exponential number of runs, but only for a PDA with one state and one stack symbol type. In this way, it is capable of a kind of structural nondeterminism, which may explain how it was able to outperform the stratification stack and sometimes the NS-RNN in \cref{sec:ns-rnn-cfl-results}.

Another way of viewing the superposition stack, when taking the finite precision of the vectors on the stack into account, is that it is like a restricted WPDA with an extremely large stack alphabet, but none of the transitions can be directly conditioned on the top stack symbol like they are in a restricted WPDA with $|\Gamma| > 1$. This means that all transition weights are synchronized to have the same weight across all runs at the same timestep.

It is interesting that when transitions are not conditioned on stack symbols, the time complexity of the algorithm can be lowered from $\bigo{n^3}$ to $\bigo{n^2}$. Remember that Lang's algorithm simulates a pop transition by looking back in time to figure out the type of the uncovered stack symbol, by rewinding its corresponding push computation. This alone is responsible for a factor of $n$. Intuitively, we can save a factor of $n$ if all operations are synchronized because we no longer need to look back in time.

\section{Capacity Experiments}
\label{sec:capacity}

In this section, we examine how much information each model can transmit through its stack. Consider the language $\{w\sym{\#}\reverse{w} \mid w \in \Sigma^\ast \}$. It can be recognized by a real-time PDA with $|\Gamma| = 3$ when $|\Sigma| = 2$, but not when $|\Sigma| > 2$, as there is always a sufficiently long $w$ such that there are more possible $w$'s than possible PDA configurations upon reading the $\sym{\#}$. Similarly, because all the neural stack models we consider here are also real-time, we expect that they will be unable to model context-free languages with sufficiently large alphabets. This is especially relevant to natural languages, which have very large vocabulary sizes. 

Neural networks can encode large sets of distinct types in compact vector representations. However, since the RNS-RNN simulates a \emph{discrete} stack, it might struggle on tasks that require it to store strings over alphabets with sizes greater than $|\Gamma|$. On the other hand, a model that uses a stack of \emph{vectors}, like the superposition stack, might model such languages more easily by representing each symbol type as a different cluster of points in a vector space. Here, we make the surprising finding that the RNS-RNN can vastly outperform the superposition stack even for large alphabets, though not always. In addition, to see if we can combine the benefits of nondeterminism with the benefits of vector representations, we test the VRNS-RNN on the same tasks.

\subsection{Tasks, Data Sampling, and Evaluation}

We test the information capacity of each model with a language modeling task on three CFLs, varying their alphabet size $k$ from very small to very large.
\begin{description}
    \item[$\bm{\MarkedReversal{}}$] The language $\{ w \sym{\#} \reverse{w} \mid w \in \{ \sym{0}, \sym{1}, \cdots, k-1 \}^\ast \}$. This is a simple deterministic CFL.
    \item[Dyck] The language of strings $D_k$ over the alphabet $\{ \sym{(}_1, \sym{)}_1, \sym{(}_2, \sym{)}_2, \cdots, \sym{(}_k, \sym{)}_k \}$ where all brackets are properly balanced and nested in pairs of $\sym{(}_i$ and $\sym{)}_i$. This is a more complicated but still deterministic CFL.
    \item[$\bm{\UnmarkedReversal{}}$] The language $\{ w \reverse{w} \mid w \in \{ \sym{0}, \sym{1}, \cdots, k-1 \}^\ast \}$. This is a nondeterministic CFL which requires a model to guess $|w|$.
\end{description}
We express each language $L$ as a PCFG, using the same PCFG definitions as in \cref{sec:task-pcfgs}, but modified to include~$k$ symbol types rather than~2. We sample training, validation, and test sets in the same way as \cref{sec:learning-cfls-data-sampling}, and we evaluate models using cross-entropy difference as in \cref{sec:rns-rnn-cfl-evaluation}.

\subsection{Models and Training}

We compare four types of architecture, each of which uses an LSTM controller. We include an LSTM baseline (``LSTM''). We also include a superposition stack that pushes learned vectors of size 3 (``Sup.\ 3''). We use the notation ``RNS $|Q|$-$|\Gamma|$'' for RNS-RNNs, and ``VRNS $|Q|$-$|\Gamma|$-$m$'' for VRNS-RNNs. All details of the controller and training procedure are the same as in \cref{sec:learning-non-cfls-models,sec:learning-non-cfls-training}. We vary the alphabet size $k$ from 2 to 200 in increments of 40. For each task, architecture, and alphabet size, we run 10 random restarts.

\subsection{Results}

In \cref{fig:capacity-results-mean}, for each task, we show the mean cross-entropy difference on the validation set (over all random restarts) as a function of alphabet size. We include standard deviations in \cref{fig:capacity-results-mean-std}, and we plot the best performance of all random restarts in \cref{fig:capacity-results-best}.

\begin{figure*}
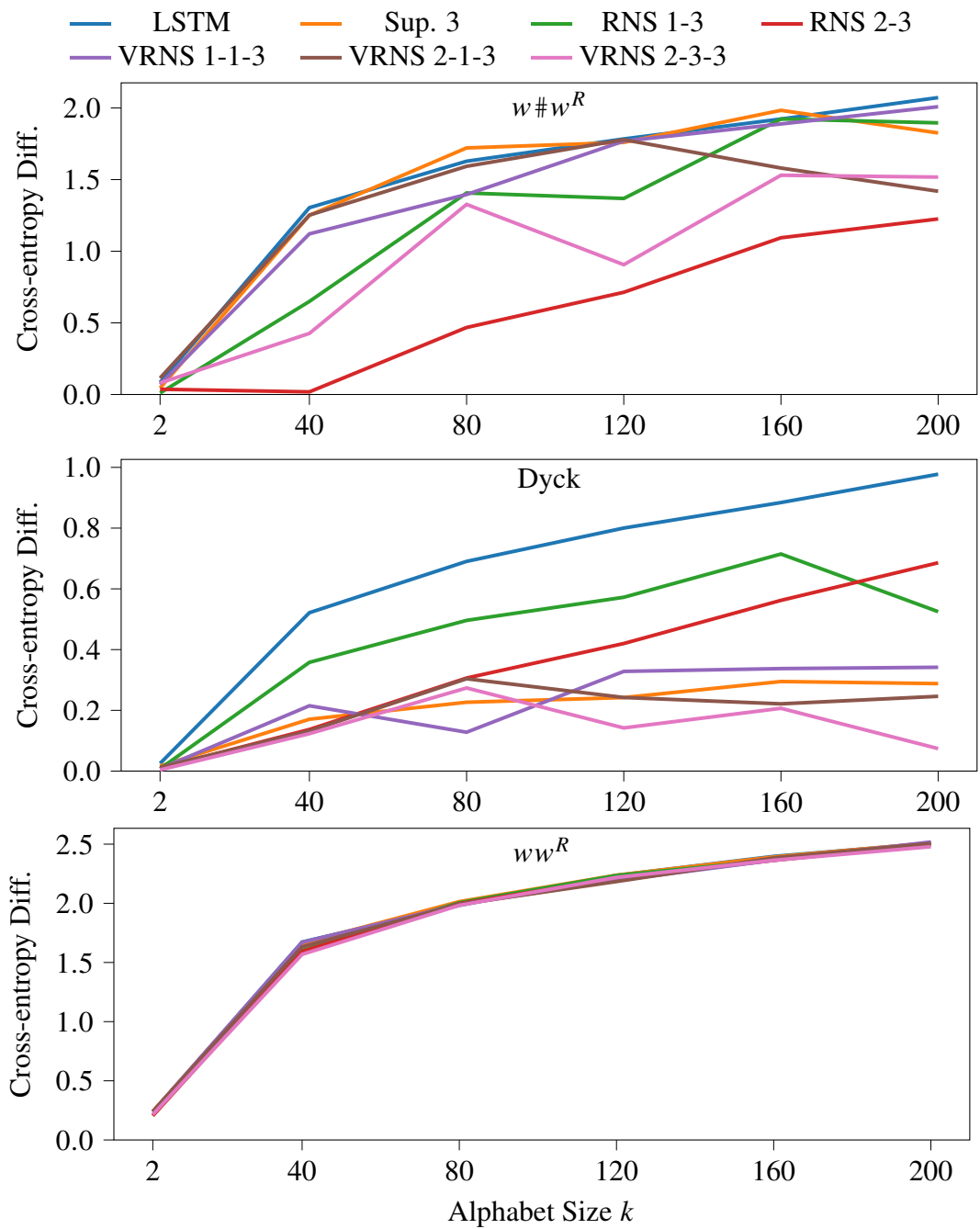

    \pgfplotsset{
        every axis/.style={
            height=2.4in,
            width=5.5in
        },
        title style={yshift=-4.5ex},
        y tick label style={
            /pgf/number format/.cd,
            fixed,
            fixed zerofill,
            precision=1,
            /tikz/.cd
        },
    }
    \centering
    \input{figures/05-capacity/legend}
    \input{figures/05-capacity/plots/marked-reversal-mean}
    \input{figures/05-capacity/plots/dyck-mean}
    \input{figures/05-capacity/plots/unmarked-reversal-mean}
    \caption[Results of capacity experiments (means).]{Mean cross-entropy difference on the validation set vs.~input alphabet size. Contrary to expectation, RNS 2-3, which models a discrete stack of only 3 symbol types, learns to solve \MarkedReversal{} with 200 symbol types more reliably than models with stacks of vectors. On the more complicated Dyck language, vector stacks perform best, with our newly proposed VRNS-RNN performing best. On \UnmarkedReversal{}, no models perform substantially better than the LSTM baseline.}
    \label{fig:capacity-results-mean}
\end{figure*}

\begin{figure*}
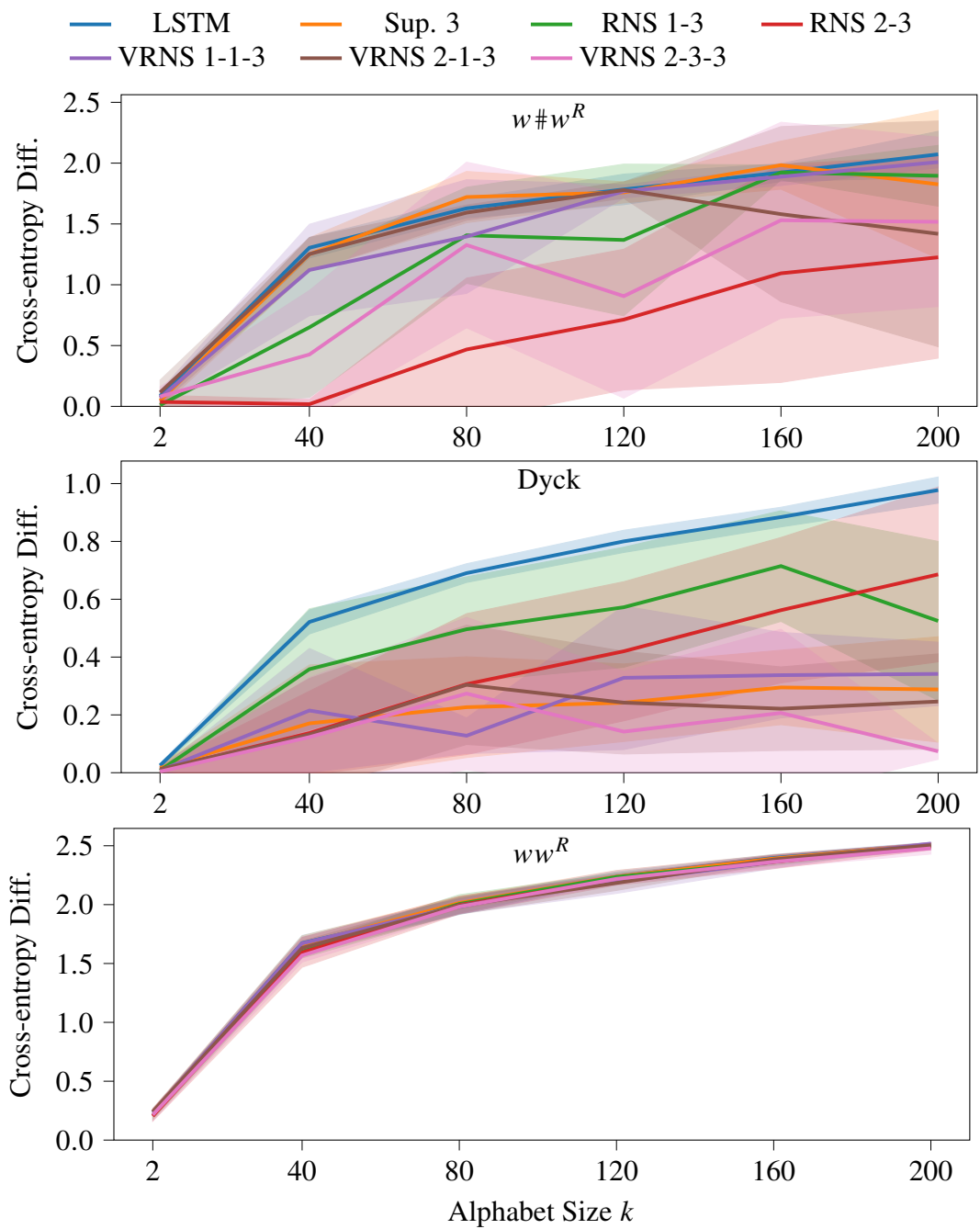

    \pgfplotsset{
        every axis/.style={
            height=2.4in,
            width=5.5in
        },
        title style={yshift=-4.5ex},
        y tick label style={
            /pgf/number format/.cd,
            fixed,
            fixed zerofill,
            precision=1,
            /tikz/.cd
        },
    }
    \centering
    \input{figures/05-capacity/legend}
    \input{figures/05-capacity/plots/marked-reversal-mean-std}
    \input{figures/05-capacity/plots/dyck-mean-std}
    \input{figures/05-capacity/plots/unmarked-reversal-mean-std}
    \caption[Results of capacity experiments (means and standard deviations).]{The same results shown in \cref{fig:capacity-results-mean}, but with standard deviations shown. Shaded regions indicate one standard deviation.}
    \label{fig:capacity-results-mean-std}
\end{figure*}

\begin{figure*}
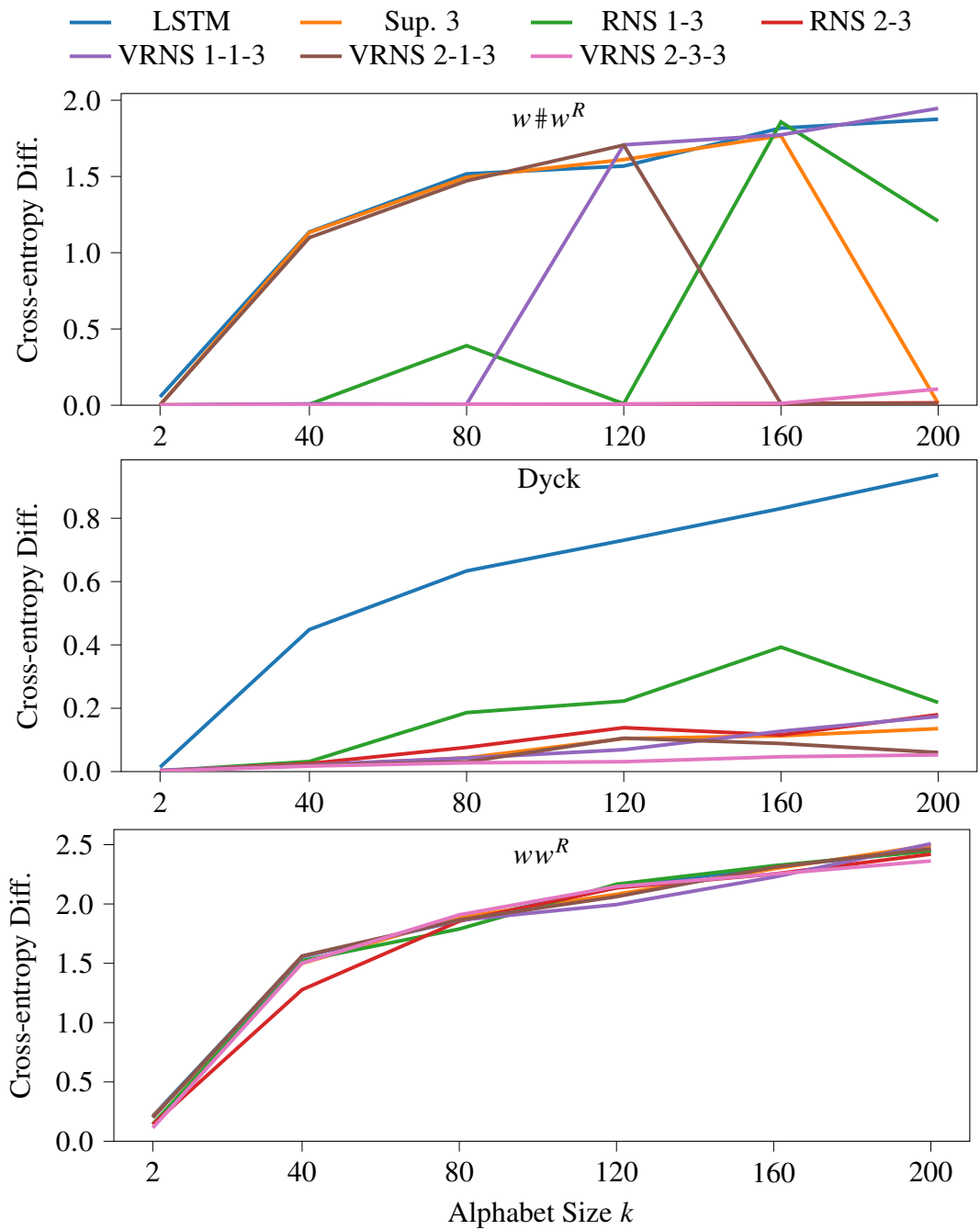

    \pgfplotsset{
        every axis/.style={
            height=2.4in,
            width=5.5in
        },
        title style={yshift=-4.5ex},
        y tick label style={
            /pgf/number format/.cd,
            fixed,
            fixed zerofill,
            precision=1,
            /tikz/.cd
        },
    }
    \centering
    \input{figures/05-capacity/legend}
    \input{figures/05-capacity/plots/marked-reversal-best}
    \input{figures/05-capacity/plots/dyck-best}
    \input{figures/05-capacity/plots/unmarked-reversal-best}
    \caption[Results of capacity experiments (best).]{Best cross-entropy difference on the validation set vs.\ input alphabet size. On \MarkedReversal{}, surprisingly, only RNS 2-3 achieves optimal cross-entropy for all alphabet sizes. On the more complicated Dyck language, our new VRNS-RNN (VRNS 2-1-3, VRNS 2-3-3) achieves the best performance for large alphabet sizes. No models perform much better than the LSTM baseline on \UnmarkedReversal{}, although RNS 2-3 performs well for $k = 40$.}
    \label{fig:capacity-results-best}
\end{figure*}

On \MarkedReversal{}, the single-state RNS 1-3, and even VRNS 1-1-3 and Sup.\ 3, struggles for large $k$. Only the multi-state models RNS 2-3, VRNS 2-1-3, and VRNS 2-3-3 show a clear advantage over the LSTM. Surprisingly, RNS 2-3, which models a discrete stack alphabet of only size 3, attains the best performance on large alphabets; in \cref{fig:capacity-results-best}, it is the only model capable of achieving optimal cross-entropy on all alphabet sizes. On the Dyck language, a more complicated DCFL, the model rankings are as expected: vector stacks (Sup.\ 3 and VRNS) perform best, with the largest VRNS model performing best. RNS-RNNs still show a clear advantage over the LSTM, but not as much as vector stack RNNs.

None of the models perform substantially better than the LSTM on the nondeterministic \UnmarkedReversal{} language, suggesting that they cannot simultaneously combine the tricks of encoding symbol types as points in high-dimensional space and nondeterministically guessing $|w|$. Although the VRNS-RNN does combine both nondeterminism and vectors, the nondeterminism only applies to the discrete symbols of $\Gamma$, of which there are no more than 3, far fewer than $k$ when $k \geq 40$. The fact that all models learn a weak baseline that does not outperform the LSTM likely explains the low variance in \cref{fig:capacity-results-mean-std}.

If RNS 2-3 has only 3 symbol types at its disposal, how can it succeed on \MarkedReversal{} for large $k$? Recall that $\stackreadingt{t}$ is a vector that represents a probability distribution over $Q \times \Gamma$. Perhaps the RNS-RNN, via $\stackreadingt{t}$, represents symbol types as different clusters of points in $\realset^{|Q| \cdot |\Gamma|}$. To test this hypothesis, we select the RNS 2-3 model with the best validation performance on \MarkedReversal{} for $k = 40$ and evaluate it on 100 samples drawn from $\sampleprobname{L}$. For each symbol between $\sym{\#}$ and $\eos$, we extract the stack reading vector computed just prior to predicting that symbol. Aggregating over all 100 samples, we reduce the stack readings to 2 dimensions using principal component analysis. We plot them in \cref{fig:capacity-reading-pca}, labeling each point according to the symbol type to be predicted just after the corresponding stack reading. Indeed, we see that stack readings corresponding to the same symbol cluster together, suggesting that the model is orchestrating the weights of different runs in a way that causes the stack reading to encode different symbol types as points in the 5-dimensional simplex.

\begin{figure*}
    \pgfplotsset{
        every axis/.style={
            height=3.5in,
            width=3.5in
        }
    }
    \centering
    \input{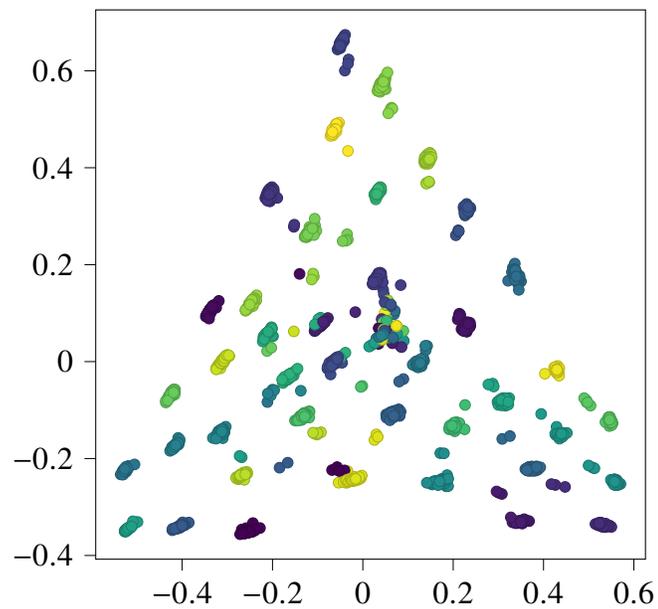}
    \caption[PCA plot of stack readings when the RNS-RNN is run on \MarkedReversal{} with 40 symbol types.]{When RNS 2-3 is run on \MarkedReversal{} with 40 symbol types, the stack readings are as visualized above. The readings are 6-dimensional vectors, projected down to 2 dimensions using PCA. The color of each point represents the top stack symbol. Points corresponding to the same symbol type cluster together, indicating the RNS-RNN has learned to encode symbols as points in the 5-dimensional simplex. The disorganized points in the middle are from the first and last timesteps of the second half of the string, which appear to be irrelevant for prediction.}
    \label{fig:capacity-reading-pca}
\end{figure*}

In \cref{fig:capacity-reading-heatmap}, we show heatmaps of actual stack reading vectors across time on an example string in \MarkedReversal{} when $k = 40$.

\begin{figure*}
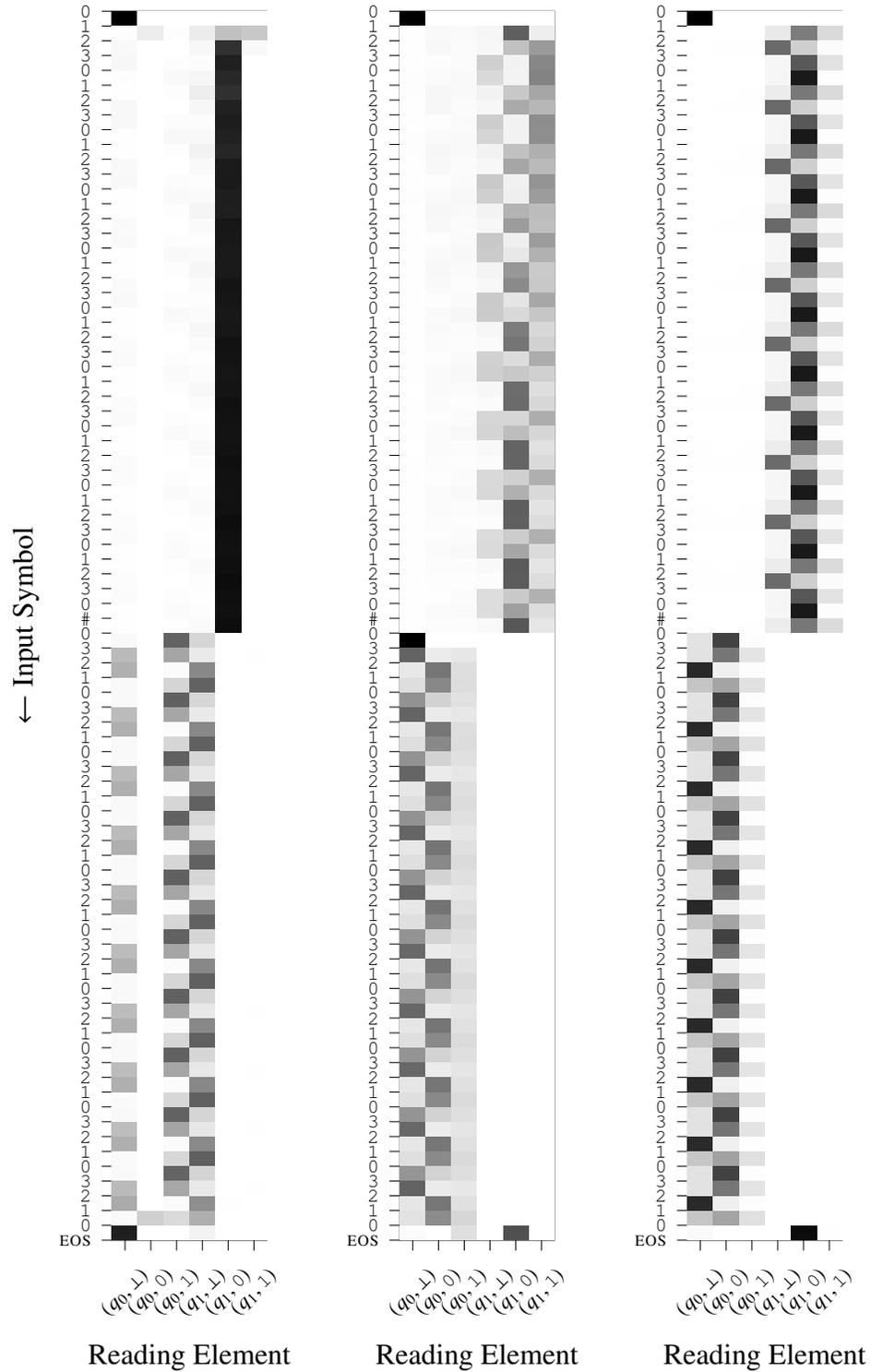

    \pgfplotsset{
      every axis/.style={
        height=7.6in,
        width=1.5in
      },
      yticklabel style={font=\scriptsize},
      xticklabel style={font=\scriptsize}
    }
    \centering
    \begin{tabular}{@{}lll@{}}    
    \input{figures/05-capacity/reading-heatmap-best} &
    \input{figures/05-capacity/reading-heatmap-1} &
    \input{figures/05-capacity/reading-heatmap-9}
    \end{tabular}
    \caption[Heatmaps of stack readings when the RNS-RNN is run on a string in \MarkedReversal{} with 40 symbol types.]{Heatmaps of $\mathbf{r}_t$ over time on a string from \MarkedReversal{} when $k = 40$, generated from the best RNS 2-3 model (left) and two other random restarts (middle, right). The $w$ string repeats the pattern $\sym{0}\sym{1}\sym{2}\sym{3}$, which is clearly seen in the reading vectors. Black~=~1, white~=~0.}
    \label{fig:capacity-reading-heatmap}
\end{figure*}

\section{Natural Language Modeling}
\label{sec:vrns-rnn-ptb}

Having examined these stack RNNs on formal languages, we now examine how they fare on natural language modeling, as the combination of nondeterminism and vector representations in the VRNS-RNN may prove beneficial. Following \cref{sec:limited-ns-rnn-ptb-experiments}, we measure perplexity (lower is better) on the Penn Treebank.

\subsection{Models and Training}

We use the same LSTM and superposition stack baselines, and various sizes of RNS-RNN and VRNS-RNN (we are able to train larger RNS-RNNs thanks to the speedup in \cref{sec:langs-algorithm-speedup}). In order to preserve context across batches, we train all models using truncated BPTT. We apply the same memory-limiting technique used in \cref{sec:limited-ns-rnn} to the VRNS-RNN, limiting non-zero entries in $\nsinnerweightletter$ and $\vrnsinnervectorletter$ to those where $t - i \leq D$, where $D = 35$.

We train each model by minimizing its cross-entropy (averaged over the timestep dimension of each batch) on the training set, using per-symbol perplexity on the validation set as the early stopping criterion. For each training run, we randomly sample the initial learning rate from a log-uniform distribution over $[1, 100]$, and the gradient clipping threshold from a log-uniform distribution over $[0.0112, 1.12]$. We initialize all parameters uniformly from $[-0.05, 0.05]$.

All other details of the models and training procedure are the same as in \cref{sec:limited-ns-rnn-ptb-models-and-training}.

\subsection{Results}

We show results in \cref{tab:vrns-rnn-ptb-results}. Most stack RNNs achieve better test perplexity than the LSTM baseline. The best models are those that simulate more nondeterminism (VRNS when $|Q| = 3$ and $|\Gamma| = 3$, and RNS when $|Q| = 4$ and $|\Gamma| = 5$). Although the superposition stack RNNs outperform the LSTM baseline, it is the combination of both nondeterminism and vector embeddings (VRNS 3-3-5) that achieves the best performance, combining the ability to process syntax nondeterministically with the ability to pack lexical information into a vector space on the stack.

\begin{table}
    \caption[Language modeling results on the Penn Treebank, with the VRNS-RNN included]{Language modeling results on the Penn Treebank}
    \label{tab:vrns-rnn-ptb-results}
    \begin{center}
        \input{figures/06-vrns-ptb/table}
    \end{center}
    \generalnote{Validation and test perplexity on the Penn Treebank of the best of 10 random restarts for each architecture. The model with the best test perplexity is our new VRNS-RNN when it combines a modest amount of nondeterminism (3 states and 3 stack symbols) with vectors of size~5.}
\end{table}

\section{Conclusion}

We showed that the RNS-RNN can far exceed the amount of information it seemingly should be able to encode in its stack given its finite stack alphabet. Our newly proposed VRNS-RNN combines the benefits of nondeterminism and vector embeddings, and we showed that it has better performance than other stack RNNs on the Dyck language and a natural language modeling benchmark.

%% file: figures/06-vrns-ptb/table.tex
\begin{tabular}{@{}lcc@{}}
\toprule
Model & Val. & Test \\
\midrule
LSTM, 256 units & 129.99 & 125.90 \\
Sup. (push hidden), 247 units & 124.99 & 121.05 \\
Sup. (push learned), $|\pushedstackvectort{t}| = 22$ & 125.68 & 120.74 \\
RNS 1-29 & 131.17 & 128.11 \\
RNS 2-13 & 128.97 & 122.76 \\
RNS 4-5 & 126.06 & 120.19 \\
VRNS 1-1-256 & 130.60 & 126.70 \\
VRNS 1-1-32 & \textbf{124.49} & 120.45 \\
VRNS 1-5-20 & 128.35 & 124.63 \\
VRNS 2-3-10 & 129.30 & 124.03 \\
VRNS 3-3-5 & 124.71 & \textbf{120.12} \\
\bottomrule
\end{tabular}

%% file: chapters/10-stack-attention.tex
\chapter{Stack Attention}
\label{chap:stack-attention}

The transformer \citep{vaswani-etal-2017-attention} has become the dominant neural network architecture in NLP, achieving state-of-the-art performance across a wide range of tasks. In this chapter, we shift focus from RNNs and show how to incorporate differentiable stacks, including our differentiable vector WPDA, into the transformer architecture by using them as a stack-structured attention mechanism. We show that transformer language models with nondeterministic stack attention learn CFLs very effectively, outperforming baseline transformers in every case, and outperforming even the RNS-RNN on the hardest CFL.

\section{Background}

Attention has become a cornerstone of modern neural network architectures \citep{cho-etal-2015-describing,bahdanau-etal-2015-neural,vaswani-etal-2017-attention}. Loosely defined, attention refers to an operation that, given a sequence of input vectors $\sublayerinputt{1}, \ldots, \sublayerinputt{n} \in \realset^d$, produces an output $\sublayeroutputletter$ that represents a linear interpolation, or ``soft-selection,'' over $\sublayerinputt{1}, \ldots, \sublayerinputt{n}$.

The transformer uses an attention operator called \term{scaled dot-product attention}. Given a \term{query vector} $\vecvar{q} \in \realset^d$, scaled dot-product attention computes the dot-product of $\vecvar{q}$ with each $\sublayerinputt{i}$, divides it by $\sqrt{d}$, and normalizes the results with a softmax. The output $\sublayeroutputletter$ is the weighted sum of $\sublayerinputt{1}, \ldots, \sublayerinputt{n}$ according to the softmax weights.
\begin{equation}
    \begin{aligned}
        \vecvar{a}'[i] &= \frac{ \vecvar{q} \cdot \sublayerinputt{i} }{ \sqrt{d} } \hspace{4em} (1 \leq i \leq n) \\
        \vecvar{a} &= \softmax(\vecvar{a}') \\
        \sublayeroutputletter &= \sum_{i=1}^n \vecvar{a}[i] \; \sublayerinputt{i}
    \end{aligned}
    \label{eq:scaled-dot-product-attention}
\end{equation}
In a transformer sublayer, this is done multiple times for $\sublayeroutputt{1}, \ldots, \sublayeroutputt{n}$, producing a new sequence of the same size. Standard practice is to give each sublayer multiple attention \term{heads}, so the layer performs a fixed number of soft-selections (typically 8) over the input at the same time.

As we saw in \cref{sec:vrns-rnn-definition}, the VRNS-RNN, too, produces a weighted sum of input vectors. Given a sequence of action tensors $\nstranst{1}, \ldots, \nstranst{t}$, and a sequence of pushed vectors $\pushedstackvectort{0}, \pushedstackvectort{1}, \ldots, \pushedstackvectort{t}$, the stack reading $\stackreadingt{t}$ is the weighted sum of all $\pushedstackvectort{i}$ that end up on top of the stack of a vector WPDA run at timestep $t$, where the run weights are dictated by the WPDA transition weights in $\nstranst{1}, \ldots, \nstranst{t}$. Likewise, the superposition stack is a simpler case of this, where the vector WPDA has only one state and one stack symbol type (\cref{sec:relationship-of-vrns-rnn}). Based on this insight, we adapt these two differentiable stacks into a new type of attention sublayer that replaces scaled dot-product attention, where the action weights are the ``queries,'' the pushed vectors are the sublayer inputs $\sublayerinputt{1}, \ldots, \sublayerinputt{n}$, and the stack readings $\stackreadingt{1}, \ldots, \stackreadingt{n}$ are the sublayer outputs. Although multi-head attention can express multiple interpretations of the input at once, it represents only a fixed number, whereas stack attention sums over an exponential number of WPDA runs.

\citet{kim-etal-2017-structured} proposed a framework for attention mechanisms that sum over an exponential number of latent structures, similar to ours. Specifically, they developed structured attention operators based on the forward-backward algorithm for linear-chain CRFs, and the inside-outside algorithm for projective dependency trees. We, on the other hand, develop an attention mechanism based on Lang's algorithm for WPDAs. A key difference in their work is that, at each timestep, their attention mechanism marginalizes over all latent structures backward \emph{and forward} in time. This precludes it from being used as a language model or decoder. In our stack attention, we marginalize backward in time only.

\section{Stack Attention Sublayer}

Transformers consist of multiple layers, each of which contains multiple \term{sublayers}. Every sublayer is supplemented with layer normalization, dropout, and residual connections. In a transformer encoder, each layer consists of a self-attention sublayer followed by a feedforward sublayer. Similarly, in a transformer decoder, each layer consists of a self-attention sublayer, a cross-attention sublayer that attends to the output of an encoder, and a feedforward sublayer, in that order.

Our method consists of replacing scaled dot-product attention in the self-attention sublayer with a differentiable stack module. Whereas scaled dot-product attention answers the question, ``which input is most important?'' stack attention answers the question, ``which input vector is most likely on top of the stack, given the stack actions up to timestep $t$?'' In this way, stack attention gives transformers the ability to simulate WPDAs with unlimited stack memory and recognize arbitrary CFLs.

We start by giving the general formula that is common to all sublayer types. Let $\layerinputt{t} \in \realset^{\dmodel}$ and $\layeroutputt{t} \in \realset^{\dmodel}$ be the input and output, respectively, of the sublayer at timestep $t$.
\begin{equation}
    \layeroutputt{t} = \layerinputt{t} + \funcname{Dropout}(\funcname{Sublayer}_t(\funcname{LayerNorm}(\layerinputt{t})))
\end{equation}
Here, $\funcname{Sublayer}_t$ is called the \term{sublayer function} that customizes the behavior of the sublayer. This equation uses pre-norm instead of post-norm, i.e.\ layer normalization occurs before the sublayer function rather than after \citep{nguyen-salazar-2019-transformers}. Layer normalization is also applied to the output of the last layer of the transformer.

For a single head of scaled dot-product self-attention, $\funcname{Sublayer}_t$ follows \cref{eq:scaled-dot-product-attention}, where $\vecvar{q}$ is a linear projection of $\layerinputt{t}$, and $\layerinputt{t}$ and $\layeroutputt{t}$ are also passed through linear projections. When $i$ in \cref{eq:scaled-dot-product-attention} is limited to at most $t$ instead of $n$, we say that the attention layer is \term{causally masked}, since it cannot attend to future timesteps. This is necessary when using transformers as language models, as we do in this chapter.

We define a stack attention sublayer as follows. Let $\stackreadingt{t}$ and $\stackobjectfunc{\stackobjectt{t-1}}{\stackactionst{t}}$ have the same definitions as in \cref{sec:controller-stack-interface}. The actions $\stackactionst{t}$ for both the superposition stack and the differentiable vector WPDA have the form $(\stackactionvect{t}, \pushedstackvectort{t})$, where $\stackactionvect{t}$ is a vector of stack action weights (for the differentiable vector WPDA, $\stackactionvect{t}$ is $\nstranst{t}$). We map $\sublayerinputt{t}$ to $\stackactionvect{t}$ with a learned linear projection, similarly to how the scaled dot-product attention sublayer generates queries. For the superposition stack, we use a softmax to ensure $\stackactionvect{t}$ sums to one as in \cref{sec:superposition-stack}. We set $\pushedstackvectort{t} = \sublayerinputt{t}$, passing $\pushedstackvectort{t}$ through a learned linear projection if $\dmodel \neq \stackvectorsizeletter$. Likewise, we set $\sublayeroutputt{t} = \stackreadingt{t}$, passing $\stackreadingt{t}$ through a learned linear projection if the size of $\stackreadingt{t}$ differs from $\dmodel$. Note that unlike scaled dot-product attention, the sublayer function contains a recurrence on past timesteps. We illustrate this architecture in \cref{fig:stack-attention-diagram}.

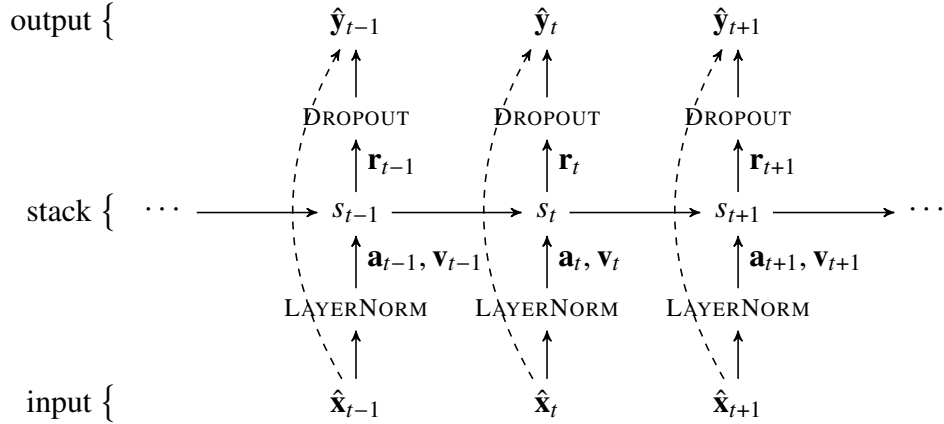
\begin{figure*}
    \newcommand{\dropout}{\scalebox{0.8}{$\funcname{Dropout}$}}
    \newcommand{\layernorm}{\scalebox{0.8}{$\funcname{LayerNorm}$}}
    \def\xsep{1in}
    \newcommand{\mycolumn}[3]{
        \node (h#1) at (#2*\xsep, -1in) {$\layerinputt{#3}$};
        \node (ln#1) at (#2*\xsep, -0.5in) {\layernorm};
        \draw[->] (h#1) edge (ln#1);
        \node (s#1) at (#2*\xsep, 0) {$\stackobjectt{#3}$};
        \draw[->] (ln#1) edge node[right] {$\stackactionvect{#3}, \pushedstackvectort{#3}$} (s#1);
        \node (dr#1) at (#2*\xsep, 0.5in) {\dropout};
        \draw[->] (s#1) edge node[right] {$\stackreadingt{#3}$} (dr#1);
        \node (o#1) at (#2*\xsep, 1in) {$\layeroutputt{#3}$};
        \draw[->] (dr#1) edge (o#1);
        \draw[->, bend left, dashed] (h#1) edge (o#1);}
    \centering
    \begin{tikzpicture}[x=2.5cm,y=2.3cm]
        \node (leftdots) at (-2*\xsep, 0) {$\cdots$};
        \mycolumn{1}{-1}{t-1}
        \mycolumn{2}{0}{t}
        \mycolumn{3}{1}{t+1}
        \node (rightdots) at (2*\xsep, 0) {$\cdots$};
        \draw[->] (leftdots) edge (s1);
        \draw[->] (s1) edge (s2);
        \draw[->] (s2) edge (s3);
        \draw[->] (s3) edge (rightdots);
        \node[anchor=east] at (-2.2in, -1in) {input \big\{};
        \node[anchor=east] at (-2.2in, 0) {stack \big\{};
        \node[anchor=east] at (-2.2in, 1in) {output \big\{};
    \end{tikzpicture}
    \caption[Conceptual diagram of a stack attention sublayer.]{Conceptual diagram of a stack attention sublayer, unrolled across a portion of time. Dashed arrows indicate residual connections.}
    \label{fig:stack-attention-diagram}
\end{figure*}

\section{Experiments}

In this section, we test the performance of transformers with stack attention on the same five CFL tasks from \cref{chap:learning-cfls}.

\subsection{Models and Training}

As baselines, we include the same LSTM, superposition stack RNN (``Sup.''), and RNS-RNN (``RNS'') architectures used in \cref{sec:rns-rnn-cfl-models-and-training}. We also test a baseline transformer (``Trans.''), a transformer with a superposition stack attention sublayer (``Trans.\ + Sup.''), and a transformer with a differentiable vector WPDA attention sublayer (``Trans.\ + VRNS''). We also refer to Trans.\ + VRNS as nondeterministic stack attention. For all transformer models, we set $\dmodel = 32$ and use a dropout rate of 0.1. Like \citet{vaswani-etal-2017-attention}, we map inputs to vectors of size $\dmodel$ with a scaled embedding layer and apply sinusoidal positional encodings. The first input is always a special $\bos$ (beginning of sequence) symbol. We map the outputs of the final layer to logits for predicting the next word via a learned affine transformation. Since the task is language modeling, all scaled dot-product attention layers are causally masked, so the transformer is trained not to attend to inputs at future timesteps. We use 4 attention heads in all scaled dot-product attention layers. We set the size of all feedforward sublayers to 64.

The baseline transformer has 5 layers. For Trans.\ + Sup.\ and Trans.\ + VRNS, in the third (middle) layer, we simply replace the scaled dot-product attention sublayer with the corresponding stack attention sublayer. For Trans.\ + Sup.\ we set $m = 32$. For Trans.\ + VRNS, we set $m = 5$, and we use the same sizes for $|Q|$ and $|\Gamma|$ as in \cref{sec:rns-rnn-cfl-models-and-training}.

As in \cref{sec:learning-non-cfls-training}, for each language and architecture, we train 10 models and report results for the model with the lowest cross-entropy difference on the validation set. For layer normalization layers, we initialize weights to 1, and biases to 0. All other details regarding initialization and training are the same as \cref{sec:learning-non-cfls-training}.

\subsection{Results}

We show cross-entropy difference on the validation and test sets in \cref{fig:stack-attention-cfls}.

\begin{figure*}
    \pgfplotsset{
        every axis/.style={
            width=3.45in,
            height=2in
        },
        title style={yshift=-4.7ex},
        y tick label style={
            /pgf/number format/.cd,
            fixed,
            fixed zerofill,
            precision=1,
            /tikz/.cd
        }
    }
    \centering
    \begin{tabular}{@{}l@{\hspace{0.05in}}l@{}}
        \multicolumn{2}{c}{\input{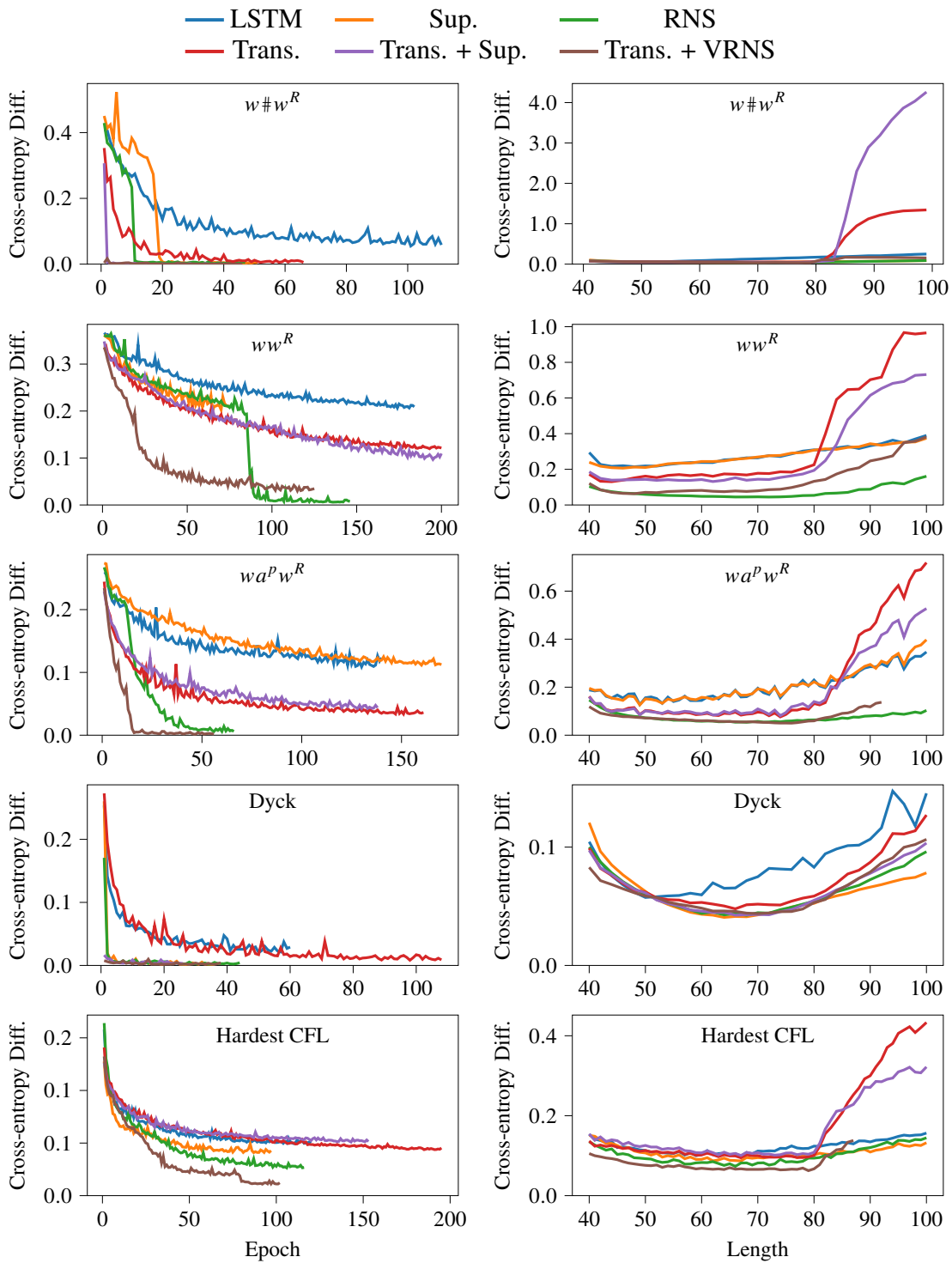}} \\
        \input{figures/07-stack-attention-cfls/table}
    \end{tabular}
    \caption[Results of using transformers and stack attention on CFLs.]{Performance of stack RNNs and transformers with stack attention on CFLs. Some test results for Trans.\ + VRNS on long strings are missing due to GPU memory limitations.}
    \label{fig:stack-attention-cfls}
\end{figure*}

Nondeterministic stack attention (Trans.\ + VRNS) achieves strong results across all tasks, even outperforming the RNS-RNN (RNS) on the hardest CFL. It also outperforms the baseline transformer (Trans.) and superposition stack attention (Trans.\ + Sup.) in all cases (there is some slight overlap with Trans.\ + Sup.\ in the test set of Dyck, where the performance of all models is near-optimal). On strings within the same length distribution as the training data (40 to 80), Trans.\ + VRNS achieves the best performance on all tasks except \UnmarkedReversal{}. Stack attention (Trans.\ + Sup.\ and Trans.\ + VRNS) converges extremely quickly (in terms of parameter updates) on the deterministic \MarkedReversal{} and Dyck tasks, even faster than stack RNNs.

How do transformers compare to their LSTM counterparts? In general, they often fit the training and validation data better, but they generalize poorly to strings in the test set that are longer than those seen in training. The baseline transformer generally outperforms the LSTM on strings within the training length distribution. However, on all tasks except Dyck, Trans.\ and Trans.\ + Sup.\ both fail catastrophically on strings longer than 80. Trans.\ + VRNS appears not to suffer from the same issue, although cross-entropy does generally start to increase more rapidly on strings longer than 80. Trans.\ + Sup.\ generally improves upon Sup., except on the hardest CFL, and Sup.\ still generalizes to longer strings better. Trans.\ + VRNS has comparable performance to RNS; it outperforms RNS on the hardest CFL, but underperforms on \UnmarkedReversal{}, and does not generalize to longer strings as well on \MarkedReversal{} and \PaddedReversal{}. This may have to do with interference from the scaled dot-product attention layers.

\section{Conclusion}

We showed that differentiable stacks, including the differentiable vector WPDA from \cref{chap:vrns-rnn}, can be incorporated into transformers as an attention mechanism. This demonstrates that nondeterministic stacks can be adapted to multiple neural network paradigms (RNNs and transformers). On CFL language modeling tasks, we showed that nondeterministic stack attention improves upon the standard transformer architecture by a wide margin, and even improves upon the RNS-RNN on the hardest CFL, a challenging ambiguous language.

%% file: chapters/11-conclusion.tex
\chapter{Conclusion}
\label{chap:conclusion}

Although neural networks have led to remarkable improvements in natural language processing systems, the predominant architectures still have problems learning compositional syntax in terms of expressivity, generalization, and training data efficiency. Prior work attempted to alleviate these issues by augmenting RNNs with differentiable stacks, inspired by the theoretical connection between syntax and pushdown automata. However, these architectures were deterministic in design, whereas natural language is full of syntactic ambiguity, which deterministic stacks cannot handle.

In this dissertation, we corrected this discrepancy by proposing a new type of differentiable stack that simulates a \emph{nondeterministic} pushdown automaton. We proved that it is expressive enough to model arbitrary CFLs, and we showed how to incorporate it into multiple types of neural network architecture: namely, RNNs and transformers. In both cases, we showed empirically that it is much more effective than prior differentiable stacks at learning context-free languages, which capture the essence of hierarchical syntax. We also showed that RNNs with nondeterministic stacks can achieve lower perplexity on natural language than prior stack RNNs.

\section{Summary of Contributions}

In \cref{chap:ns-rnn}, we proposed a new type of differentiable stack, the differentiable WPDA, which can be incorporated into a neural network to allow it to simulate all runs of a nondeterministic WPDA. We did this by developing a normal form for PDAs that is restricted enough that it can be simulated efficiently with tensor operations, yet is still powerful enough to recognize all context-free languages. We also showed that because the output of the differentiable WPDA is a weighted sum over PDA runs, training becomes more reliable, because during backpropagation, all runs receive a reward that is directly proportional to their usefulness.

In \cref{chap:learning-cfls}, we validated our approach by showing that the NS-RNN, which consists of an LSTM augmented with a differentiable WPDA, outperforms prior stack RNNs on a variety of CFLs. Notably, this includes the hardest CFL, a nondeterministic language with theoretically maximal parsing difficulty.

In \cref{chap:rns-rnn}, we presented the RNS-RNN, showing that two changes to the NS-RNN significantly improve its ability to learn CFLs. We removed local normalization on the WPDA's transition weights, and we showed that this allows the model to learn much faster. The success of using unnormalized weights together with nondeterminism points towards the effectiveness of representing a space of structures in a neural network as a weighted sum free from local normalization, since all possibilities receive gradient proportionally to their contribution to the objective, and the magnitude of the contribution is not limited. We also refrained from marginalizing PDA states out of the stack reading distribution, which proved essential to improving performance on nondeterministic CFLs.

In \cref{chap:incremental-execution}, we developed a memory-limiting technique that allows the NS-RNN and RNS-RNN to be trained using truncated BPTT on arbitrarily long sequences. We presented results on a natural language modeling task and a comprehensive syntactic generalization benchmark. However, the NS-RNN and RNS-RNN did not yet improve over the superposition stack for the sizes we tested. We also showed that there appears to be a lack of correlation between perplexity and syntactic generalization, corroborating others' findings.

In \cref{chap:learning-non-cfls}, we addressed a potential shortcoming of the RNS-RNN: the differentiable WPDA is designed to recognize CFLs, but not all phenomena in natural language, such as cross-serial dependencies, are context-free. We showed theoretically and empirically that, thanks to the interaction between the controller and the differentiable WPDA, the RNS-RNN can learn many non-CFLs, including cross-serial dependencies, provided the boundary is explicitly marked. On the other hand, merely attaching multiple deterministic stacks to an RNN does not allow it to achieve the same effect.

In \cref{chap:vrns-rnn}, we addressed another possible shortcoming: the RNS-RNN is limited to small stack alphabet sizes due to computational cost, which would seem to be inadequate for natural language. However, we showed that on two DCFLs, the RNS-RNN is surprisingly good at embedding information in the distribution over WPDA runs, allowing it to handle very large alphabets. We also presented the VRNS-RNN, an extension of the RNS-RNN that uses a stack of vectors. The VRNS-RNN is a generalization of both the RNS-RNN and the superposition stack designed to combine the benefits of nondeterminism and embedding vectors. We showed that it can achieve lower perplexity than both on a natural language modeling task.

In \cref{chap:stack-attention}, we demonstrated that nondeterministic stacks are not limited to RNNs, but can be incorporated into the transformer architecture as well, yielding strong results on CFLs.

\section{Directions for Future Work}

We conclude by discussing some ways this work can contribute to future research, and its consequences for the field.

\subsection{Natural Language}

Overall, the differentiable WPDA has strong results on formal languages, but improvements on natural language and syntactic generalization remain relatively small. One barrier to stronger results on natural language is the computational cost of the differentiable WPDA, as large WPDA sizes are prohibitively expensive in terms of time and memory. Although information capacity in the stack turns out not to be a serious problem, and the VRNS-RNN alleviates this issue, we are still restricted to small values of $|Q|$ and $|\Gamma|$, which are what allow the network to represent ambiguous structures. This likely explains why the VRNS-RNN did not improve performance on \UnmarkedReversal{} for large alphabet sizes. Decreasing the computational cost further, perhaps using approximations of the stack reading, will be an important consideration in future work.

The experiments in \cref{chap:incremental-execution,chap:vrns-rnn} also lack certain features used in state-of-the-art language models. Training on a task with a larger vocabulary size, with regularization techniques like dropout, and supplementing it with a large pretrained language model may lead to more conclusive results. It is also worth exploring to what degree syntactic ambiguity is actually a factor in the performance of downstream NLP tasks, and why optimizing neural networks using perplexity does not correspond to better syntactic generalization. To facilitate analysis, it should be possible to recover parses of a sentence from a trained RNS-RNN or stack attention layer by running the Viterbi algorithm on the differentiable WPDA. Another straightforward continuation of this work is to incorporate stack attention into the encoder and decoder of a machine translation system.

Successfully combining nondeterministic stacks with state-of-the-art NLP systems using the suggestions above would allow them to learn compositional syntactic rules more robustly.

\subsection{Generalization Issues}

On our formal language tasks, networks with differentiable WPDAs sometimes generalize poorly on strings that are longer than those seen in the training data. \citet{deletang-etal-2023-neural}, too, reimplemented our NS-RNN and showed that it does not generalize well on longer strings for some formal language tasks that it should solve perfectly. This suggests that regularization techniques may be necessary for better syntactic generalization. We also believe that certain changes to the RNS-RNN architecture, inspired by details of the proof of \cref{thm:constantpda}, would improve generalization. Namely, these changes include (1) adding more layers between the stack reading and the controller, (2) concatenating the stack reading to the hidden state when predicting $\rnnoutputt{t}$, fixing a ``lag'' between updating the WPDA and predicting the next output, and (3) removing normalization in the stack reading. We have already completed preliminary experiments showing that (1) and (2) solve generalization issues on some tasks, including \MarkedReverseAndCopy{}. However, we have seen that simply removing stack reading normalization destabilizes training.

On the linguistics side, several papers have studied the inductive bias of neural network architectures with regard to syntax \citep{frank-mathis-2007-transformational,mccoy-etal-2018-revisiting,mccoy-etal-2020-does,warstadt-bowman-2020-neural,warstadt-etal-2020-learning,mulligan-etal-2021-structure,petty-frank-2021-transformers,mueller-etal-2022-coloring}, using an experimental design based on Chomsky's argument from the poverty of the stimulus \citep{wilson-2006-learning}. In this framework, a model is trained on a sequence transduction task, such as turning a declarative sentence into a question. For example, ``The dog can see the cat that can meow'' becomes ``Can the dog see the cat that can meow?'' The linguistically correct rule for this transformation is to front the main auxiliary verb (which requires a syntactic parse). However, if the training data does not contain the right disambiguating examples, it may also be consistent with a linguistically incorrect rule, such as fronting the \emph{first} auxiliary verb. How then would a model generalize to a sentence like ``The dog that can bark can see the cat that can meow''? The answer depends on the architecture's inductive bias. If the model simply fronts the first verb, it will incorrectly produce ``Can the dog that bark can see the cat that can meow?'' To date, no one has yet demonstrated a neural architecture that consistently generalizes in human-like fashion without syntactic supervision or pretraining on massive amounts of data. We have already completed preliminary work on the RNS-RNN, with inconclusive results. These experiments, however, use greedy decoding, limiting the quality of the outputs, and beam search may provide clearer answers.

Developing an architecture with a hierarchical inductive bias would likely provide the field with models that can generalize more predictably to syntactic patterns not encountered in training, and that can master syntax from fewer training examples.